\pdfoutput=1
\documentclass[letterpaper,11pt]{scrartcl}

\usepackage[T1]{fontenc}
\usepackage[utf8]{inputenc}
\usepackage{lmodern}
\usepackage{amsfonts,amsmath,amssymb,amsthm,mathtools}
\usepackage[letterpaper,margin=1in]{geometry}
\usepackage[english]{babel}
\usepackage[dvipsnames]{xcolor}
\definecolor{revcolor}{HTML}{B3261E}
\newif\ifrevmarks \revmarkstrue

\usepackage{tikz}
\usetikzlibrary{arrows.meta}
\definecolor{coltreat}{HTML}{0066CC}
\definecolor{colout}{HTML}{D62828}
\definecolor{collatent}{HTML}{9A9A9A}
\usepackage{wrapfig}
\usepackage{mdframed}
\usepackage[authoryear,round]{natbib}
\usepackage{bm}
\usepackage{multirow}
\usepackage{array}
\usepackage{verbatim}
\usepackage[normalem]{ulem}
\usepackage{thmtools}
\usepackage{setspace}
\usepackage{adjustbox}
\usepackage{graphicx}
\usepackage{subfig}
\usepackage{placeins}
\usepackage[useregional]{datetime2}
\usepackage{enumitem}
\usepackage[most]{tcolorbox}
\tcbuselibrary{theorems}
\usepackage{longtable}
\usepackage{ltablex}
\usepackage{nccmath}
\usepackage{sansmath}
\usepackage{booktabs,tabularx}
\usepackage{makecell}
\usepackage[ruled,vlined,linesnumbered]{algorithm2e}
\usepackage[colorlinks=true,citecolor=NavyBlue,linkcolor=RoyalBlue,urlcolor=NavyBlue]{hyperref}

\newtheoremstyle{uprightboxed}{0pt}{0pt}{\normalfont}{}{\bfseries}{.}{.5em}{}
\theoremstyle{uprightboxed}
\newtheorem{definition}{Definition}
\newtheorem{theorem}{Theorem}
\newtheorem{assumption}{Assumption}

\newtheorem{lemma}{Lemma}
\newtheorem{proposition}{Proposition}
\newtheorem{corollary}{Corollary}

\tcolorboxenvironment{definition}{colback=gray!5!white,colframe=black,fonttitle=\bfseries,boxrule=0.5pt,breakable}
\tcolorboxenvironment{theorem}{colback=gray!5!white,colframe=black,fonttitle=\bfseries,boxrule=0.5pt,breakable}
\tcolorboxenvironment{corollary}{colback=gray!5!white,colframe=black,fonttitle=\bfseries,boxrule=0.5pt,breakable}
\tcolorboxenvironment{lemma}{colback=gray!5!white,colframe=black,fonttitle=\bfseries,boxrule=0.5pt,breakable}
\tcolorboxenvironment{proposition}{colback=gray!5!white,colframe=black,fonttitle=\bfseries,boxrule=0.5pt,breakable}
\tcolorboxenvironment{remark}{colback=gray!5!white,colframe=black,fonttitle=\bfseries,boxrule=0.5pt,breakable}
\tcolorboxenvironment{assumption}{colback=white,colframe=black,fonttitle=\bfseries,boxrule=0.5pt,breakable}

\graphicspath{{figures/}}

\newcommand{\draftdate}{\the\year-\ifnum\month<10 0\fi\the\month-\ifnum\day<10 0\fi\the\day}
\titlehead{\hfill \draftdate}
\title{Proximal Balancing for Causal Effect Estimation under Unmeasured Confounding}
\author{Yonghan Jung\thanks{Code, results, and scripts that reproduce every figure and table: \url{https://github.com/CausalDataScience/proximal-balancing}.}\\[0.3em]\normalsize School of Information Sciences\\\normalsize University of Illinois Urbana-Champaign}
\date{}

\begin{document}
\maketitle
\allowdisplaybreaks

\begin{abstract}
\emergencystretch=2em
Estimating causal effects from observational data is central to science and policy, but the effects are not identified when confounders are unmeasured. Proximal causal inference addresses this problem with proxies of the unmeasured confounders. However, existing proxy-based approaches either designate proxy roles and solve an inverse problem, which is ill-posed and hard to estimate with high-dimensional proxies, or use a latent-variable model, which assumes that the learned latent variable matches the hidden confounder and leaves bias when it does not. To address these challenges, we introduce \emph{proximal balancing}. It carries the classical idea of covariate balancing to confounders that are observed only through proxies: it learns a low-dimensional summary of the covariates and proxies that makes the treatment groups comparable, and then adjusts for this summary. It needs no designated proxy roles, inverse problem, or latent model. We give identification theory, finite-sample guarantees, and a practical algorithm, PROBE. We demonstrate the method on low-dimensional, high-dimensional, and image proxies and on real-world data.

\end{abstract}

\tableofcontents
\newpage

\begingroup
\allowdisplaybreaks
\sloppy
\section{Introduction}
\label{sec:main-introduction}

Estimating treatment effects from observational data, a central task in medicine and public policy, usually assumes that every confounder (every common cause of the treatment and the outcome) is measured. This assumption, known as \emph{no unmeasured confounding}, often fails. In a hypothetical study of an aggressive therapy $A$ and a recovery outcome $Y$, clinicians treat patients with severe disease more often, and severe disease lowers recovery, so the unrecorded severity $U$ confounds the comparison even after adjustment for recorded covariates $X$ such as age. The effect is then not identified, that is, not determined by the distribution of the observed data \citep{pearl2009causality}.

Although $U$ is unrecorded, modern data often capture it indirectly. For instance, a pretreatment medical image reflects severity through the size, texture, and spread of affected tissue. Variables that carry such indirect information about a hidden confounder are called \emph{proxies}, ranging from discrete codes to multimodal data such as images with thousands of pixels. Proxy-based causal inference has drawn growing attention because it identifies causal effects under hidden confounding by adding structural assumptions that are often plausible in practice \citep{kuroki2014measurement,miao2018proxy,tchetgen2024ProximalCausalInference}.

Existing proxy-based methods either designate proxy roles in advance, for example, that each proxy affects only one of the treatment and the outcome, and obtain the identification functional as the solution of an inverse problem, such as an integral equation for a bridge function \citep{miao2018proxy}, or model latent variables from the observations and treat them as the hidden confounder under strong assumptions \citep{louizos2017causal} (see Section~\ref{sec:related-work} for details). For high-dimensional proxies such as images, these approaches are difficult to apply. When multiple medical images, such as radiographs, CT scans, and MRI scans taken before treatment, serve as proxies of hidden severity, clinical knowledge seldom tells which images may influence the treatment decision, so roles cannot be designated with confidence. The inverse problem is ill-posed, so small errors in the observed regressions can produce large errors in its solution, and this solution is harder to estimate when it takes an image as input \citep{kallus2021causal}. A learned latent variable, in turn, need not match the hidden confounder \citep{rissanen2021critical}.

\begin{figure*}[t]
\centering
\includegraphics{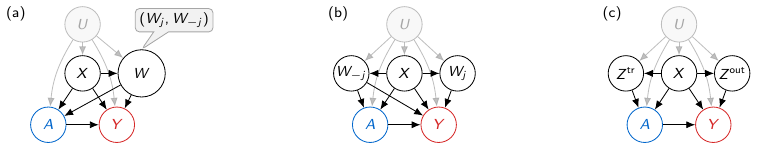}
\caption[Causal structures of proximal balancing and proximal causal inference.]{\textbf{(a)} The proxy $W$, split into $(W_j,W_{-j})$. \textbf{(b)} A higher-resolution view of (a), showing an allowed structure: the held-out block $W_j$ affects $Y$ but not $A$. \textbf{(c)} Proximal causal inference with designated proxies $Z^{\mathrm{tr}}$ and $Z^{\mathrm{out}}$.}
\label{fig:causal-structures}
\end{figure*}

To address these concerns, we introduce \emph{proximal balancing}, which uses a multi-dimensional proxy without designating roles in advance, solving an inverse problem, or modeling the confounder. It builds on balancing in causal inference \citep{rosenbaum1983central}: when high-dimensional covariates $X$ suffice for adjustment, a low-dimensional function $\phi(X)$ satisfying $X\perp\!\!\!\perp A\mid\phi(X)$ suffices as well. Proximal balancing carries this idea to hidden confounding. As Fig.~\ref{fig:causal-structures}(a) shows, it splits the proxy into a held-out block $W_j$ and the remaining blocks $W_{-j}$ and learns $\phi(X,W_{-j})$ satisfying $(X,W_j)\perp\!\!\!\perp A\mid\phi(X,W_{-j})$. We first develop this idea for a known held-out block and then extend it to the case where only the whole proxy $W=(W_1,\ldots,W_J)$ is given and the valid block is unknown.

We make three contributions.
\begin{enumerate}[leftmargin=*]
\item \textbf{Proximal balancing framework.} We develop proximal balancing, which identifies the average treatment effect from a high-dimensional proxy (Section~\ref{sec:probe}).
\item \textbf{Extensions.} We bound the bias under approximate balance, identify the effect by a plurality vote when the valid block is unknown, and give finite-sample guarantees for our algorithm, PROxy Blockwise Exclusion (PROBE) (Sections~\ref{sec:unknown-blocks}--\ref{sec:learning-search}).
\item \textbf{Empirical evidence.} Simulations with scalar, vector, and image-rendered proxies show that proximal balancing recovers the known effect (Section~\ref{sec:main-experiments}).
\end{enumerate}

\subsection{Related Work}
\label{sec:related-work}

\textbf{Causal inference under unmeasured confounding.} When the causal graph is known, a graphical criterion such as the front-door criterion can identify an effect despite unmeasured confounding \citep{pearl2009causality,fulcher2020robust,jung2026debiased}, and a complete algorithm decides whether an effect is identifiable \citep{tian2002general,shpitser2006identification}. Weighting, doubly robust, and debiased machine learning estimators have been developed for identifiable effects \citep{jung2020estimating,jung2020werm,jung2021estimating-dml,jung2021estimating-mec,bhattacharya2022semiparametric,jung2023estimating-obs-exp,guo2023targeted,jung2024complete,jung2024uca}. These methods, however, assume that the causal graph, or at least its equivalence class, is known; without such a graph, the observed data only bound the effect \citep{jung2026information}.

\begin{wraptable}{r}{0.45\textwidth}
\vspace{-\baselineskip}
\centering
\small
\setlength{\tabcolsep}{2.5pt}
\caption[Requirements of proxy-based methods.]{Requirements of proxy-based methods; a check means not needed. Roles: designated proxy roles. Inverse: an inverse problem. Latent: a latent model of $U$.}
\label{tab:related-work-delta}
\begin{tabular}{@{}lccc@{}}
\toprule
Method & Roles & Inverse & Latent \\
\midrule
Proximal inference & $\times$ & $\times$ & $\checkmark$ \\
Single-proxy control & $\times$ & $\times$ & $\checkmark$ \\
CEVAE & $\checkmark$ & $\checkmark$ & $\times$ \\
\citeauthor{saha2026causal} & $\checkmark$ & $\checkmark$ & $\times$ \\
Proximal balancing & $\checkmark$ & $\checkmark$ & $\checkmark$ \\
\bottomrule
\end{tabular}
\vspace{-\baselineskip}
\end{wraptable}

\textbf{Proxy-based identification.} For discrete variables, matrix adjustment \citep{kuroki2014measurement,lee2021matrix} identifies the effect by inverting the matrix of proxy probabilities given the confounder. Proximal causal inference, which grew out of negative-control methods \citep{lipsitch2010negative} and matrix adjustment, uses a treatment-side and an outcome-side proxy, which may affect only the treatment and only the outcome, respectively, as in Fig.~\ref{fig:causal-structures}(c) \citep{miao2018proxy,tchetgen2024ProximalCausalInference,cui2024semiparametric}. It computes the effect by solving an inverse problem, an integral equation among observed distributions, and completeness conditions make this solution identify the effect. Kernel and deep-learning estimators solve this inverse problem with statistical guarantees \citep{mastouri2021proximal,kallus2021causal}, including for image proxies and for proxies extracted from two separate pretreatment texts \citep{xu2021deep,chen2024proximal}. Single-proxy control \citep{tchetgen2014control,park2024single} uses one proxy of $Y(0)$, the outcome that a unit would have without treatment, and identifies the effect on the treated; a kernel version handles deterministic outcomes \citep{xu2025kernel}. All of these methods fix proxy roles before the analysis and solve an inverse problem (Table~\ref{tab:related-work-delta}). Fig.~\ref{fig:causal-structures} contrasts proximal balancing, panels (a) and (b), with the standard proximal structure, panel (c). Proximal balancing restricts only the valid held-out block $W_j$, which must not affect the treatment; the remaining blocks $W_{-j}$ may affect both, and the valid block need not be known in advance. It replaces the inverse problem with an observable balance condition followed by ordinary adjustment.

\textbf{Latent-variable modeling approaches.} CEVAE \citep{louizos2017causal} fits a variational autoencoder \citep{kingma2014auto} to the proxies, treatment, and outcome, and adjusts for the inferred latent vector. Its justification rests on the premise that the encoding and decoding steps recover a latent variable that matches $U$, a premise that need not hold, as \citet{rissanen2021critical} demonstrate empirically. Proximal balancing models neither the proxy nor the confounder.

\textbf{Unknown valid candidates.} Majority and plurality rules, implemented by median and mode estimators, identify an effect when most, or the largest group, of the candidate instruments are valid \citep{kang2016instrumental,bowden2016consistent,hartwig2017robust,guo2018confidence}. In proximal inference, \citet{rakshit2025adaptive} take a median over candidate proxies when more than half are valid, \citet{yu2025fortified} assume that at least a known number of treatment-side proxies are valid, and data-driven searches select valid negative controls in linear models \citep{kummerfeld2024data,xie2024automating}. \citet{saha2026causal} identify the effect without designating proxy roles when three proxies are mutually independent given a categorical confounder. Proximal balancing applies the plurality rule to held-out blocks that pass a balance audit, in a nonparametric model with a learned representation. Its conditions also allow the blocks to depend on each other beyond the confounder.

\textbf{Balancing and representation learning.} Balancing scores and weights align measured covariates across treatment groups \citep{rosenbaum1983central,hainmueller2012entropy,imai2014covariate}, and representation learning balances after compressing high-dimensional covariates \citep{shalit2017estimating}. Both assume that the measured covariates suffice for adjustment, and a compressed representation can lose confounding information. Proximal balancing audits its representation with a held-out proxy block, and its estimation step uses separate samples for representation learning, nuisance fitting, and evaluation.

\section{Problem Setting and Notation}
\label{sec:problem-setting}

 We observe $n$ i.i.d. copies $O_1,\ldots,O_n$ of $O\triangleq(X,W,A,Y)$, where $A\in\{0,1\}$ is treatment, $Y$ is the observed outcome, and $X$ contains observed pretreatment covariates. For each $a\in\{0,1\}$, the \emph{potential outcome} $Y(a)$ is the outcome that would be observed if treatment were set to $a$. The target is the average treatment effect $\tau\triangleq\mathbb E\{Y(1)-Y(0)\}$.

We impose the following global outcome conditions.

\begin{assumption}[Consistency and integrability]
\label{ass:consistency}
For each $a\in\{0,1\}$, $Y=Y(A)$ and $\mathbb E|Y(a)|<\infty$.
\end{assumption}

The high-dimensional pretreatment proxy $W=(W_1,\ldots,W_J)$ is divided into prespecified blocks. An unobserved pretreatment variable $U$ may affect both treatment and outcome. The proxy $W$ carries information about $U$. Panels (a) and (b) of Fig.~\ref{fig:causal-structures} visualize this problem setting; they do not themselves guarantee identification. We analyze one block at a time: one block is held out, and the other blocks remain available.

\begin{definition}[Held-out and remaining proxy blocks]
\label{def:held-out-remaining-blocks}
Let $[J]\triangleq\{1,\ldots,J\}$. For $j\in[J]$, call $W_j$ the held-out block and write $W_{-j}\triangleq(W_k:k\in[J]\setminus\{j\})$ for the collection of remaining blocks. Define $V_j\triangleq(X,W_{-j})$ as the retained variables and $T_j\triangleq(X,W_j)$ as the variables used later to formulate the held-out balance condition.
\end{definition}

\section{Identification by Proximal Balancing}
\label{sec:identification}

We introduce \emph{proximal balancing}, an identification method that brings covariate-balancing ideas \citep{hainmueller2012entropy,imai2014covariate} to high-dimensional pretreatment proxies. It learns a representation from the retained proxy blocks and uses the held-out block with observed covariates to audit residual treatment information.

\subsection{Exact Balance}
\label{sec:probe}

We first consider a fixed held-out block within the high-dimensional proxy $W=(W_1,\ldots,W_J)$. For $j\in[J]$, Definition~\ref{def:held-out-remaining-blocks} calls $W_j$ the \emph{held-out block} and $W_{-j}$ the \emph{remaining blocks}.

\begin{assumption}[Latent exchangeability and held-out treatment independence]
\label{ass:joint-latent-exchangeability}\label{ass:held-out-separation}
	For each $a\in\{0,1\}$, $Y(a)\perp\!\!\!\perp A\mid(U,V_j)$ and $W_j\perp\!\!\!\perp A\mid(U,V_j)$.
\end{assumption}
Latent exchangeability says that, after $(U,V_j)$ is fixed, treatment carries no further information about $Y(a)$. Held-out treatment independence states that $W_j$ carries no additional treatment information given the latent and retained variables.

Let $\phi_j$ be a measurable encoder and set $Z_j\triangleq\phi_j(V_j)=\phi_j(X,W_{-j})$. Proximal balancing seeks a representation that makes the treatment arms comparable for the potential-outcome means, while $T_j=(X,W_j)$ supplies an observable check of residual treatment information. For a distribution $Q$ of $(V_j,U)$, let $\mathcal K_jQ$ be the distribution of $T_j$ obtained by drawing $(V_j,U)\sim Q$ and then $W_j\sim P(W_j\mid V_j,U)$. Let $\mu_{a,j,z}(v,u)\triangleq\mathbb E\{Y(a)\mid V_j=v,U=u,Z_j=z\}$ be the mean of $Y(a)$ at a latent state.

Different distributions $Q,Q'$ of $(V_j,U)$ can induce the same distribution of $T_j$. The next assumption states when this loss of latent information is harmless for the potential-outcome means needed for identification.

\begin{assumption}[Outcome-relevant completeness]
\label{ass:outcome-relevant-completeness}
For almost every $z$, every pair of distributions $Q,Q'$ of $(V_j,U)$ that are absolutely continuous with respect to $P(V_j,U\mid Z_j=z)$ and under which $\mu_{0,j,z}$ and $\mu_{1,j,z}$ are integrable, and each $a\in\{0,1\}$, $\mathcal K_jQ=\mathcal K_jQ'$ implies $\mathbb E_Q\mu_{a,j,z}=\mathbb E_{Q'}\mu_{a,j,z}$.
\end{assumption}

 In words, if two distributions of the latent state $Q,Q'$ give different means of $Y(a)$, they also give different distributions of $T_j=(X,W_j)$.

\begin{theorem}[Identification from a held-out audit]
\label{thm:exact-proxy-identification}
Fix $j\in[J]$ and a measurable encoder $\phi_j$, and set $Z_j\triangleq\phi_j(V_j)$. Suppose Assumptions~\ref{ass:consistency}--\ref{ass:outcome-relevant-completeness} hold. If
\begin{align}
	0 &< P(A=1\mid Z_j)<1\ \text{a.s.},\qquad (X,W_j)\perp\!\!\!\perp A\mid Z_j
\label{eq:exact-proxy-balance}
\end{align}
then
\begin{align}
\tau &=\mathbb E[\mathbb E(Y\mid A=1,Z_j)-\mathbb E(Y\mid A=0,Z_j)].
\label{eq:exact-identification}
\end{align}
\end{theorem}

\subsection{Unknown Valid Block}
\label{sec:unknown-blocks}

Section~\ref{sec:probe} treats the ideal case in which a valid held-out block is known and exact balance holds. We next allow the valid block to be unknown and then quantify the bias caused by imperfect balance.

Let $\mathfrak S\triangleq\{S\subset[J]:1\le |S|\le J-1\}$ be the collection of nonempty proper subsets of proxy-block indices. For $S\in\mathfrak S$, write $W_S=(W_k)_{k\in S}$ for the held-out blocks and $W_{-S}=(W_k)_{k\notin S}$ for the remaining blocks, set $V_S\triangleq(X,W_{-S})$, and let $Z_S\triangleq\phi_S(V_S)$. The singleton $S=\{j\}$ recovers the split of Section~\ref{sec:probe}. Each split yields the population candidate adjustment value
\begin{align}
\label{eq:block-candidate-adjustment-value}
	\theta_S &\triangleq\mathbb E[\mathbb E(Y\mid A=1,Z_S)-\mathbb E(Y\mid A=0,Z_S)].
\end{align}

The next definition records the observable screen and the target plurality condition.
\begin{definition}[Proximal plurality]
\label{def:population-plurality}
Let $\mathcal B$ be the collection of splits $S\in\mathfrak S$ for which overlap $0<P(A=1\mid Z_S)<1$ holds almost surely and exact balance $(X,W_S)\perp\!\!\!\perp A\mid Z_S$ holds. The target has a unique proximal plurality when
\begin{align}
|\{S\in\mathcal B:\theta_S=\tau\}| &>\max_{c\ne\tau:\{S\in\mathcal B:\theta_S=c\}\ne\varnothing}|\{S\in\mathcal B:\theta_S=c\}|,
\label{eq:unique-target-plurality}
\end{align}
where the maximum over an empty set is zero.
\end{definition}
In words, proximal plurality means that $\tau$ is the most frequent candidate value among the screened splits, even if it is not attained by a majority. For example, let $J=3$, so $\mathfrak S=\{\{1\},\{2\},\{3\},\{1,2\},\{1,3\},\{2,3\}\}$. If $\mathcal B=\{\{1\},\{2\},\{1,2\},\{1,3\}\}$, $\tau=1$, $\theta_{\{1\}}=\theta_{\{1,2\}}=1$, $\theta_{\{2\}}=0.5$, and $\theta_{\{1,3\}}=0.3$, then $\tau$ appears twice while each competing value appears once, so the target has a unique proximal plurality.

\begin{theorem}[Identification with an unknown valid block]
\label{thm:unknown-valid-block-identification}
If $\mathcal B\ne\varnothing$ and the target has a unique proximal plurality, then $\{\tau\}=\operatorname*{arg\,max}_{c\in\mathbb R}|\{S\in\mathcal B:\theta_S=c\}|$.
\end{theorem}

\subsection{Approximate Balance}
\label{sec:imperfect-balance}

Exact balance in Eq.~\eqref{eq:exact-proxy-balance} may be infeasible in practice. We derive the bound for a fixed held-out block $W_j$. The same argument extends to a grouped held-out block $W_S$ after replacing the single-block objects by their grouped counterparts. Define $Z_j^\phi \triangleq \phi(V_j)$, $e_{j,\phi}(z) \triangleq P(A=1\mid Z_j^\phi=z)$ and $q_{j,\phi}(t,z) \triangleq P(A=1\mid T_j=t,Z_j^\phi=z)$.

\begin{definition}[Residual treatment discrepancy]
\label{def:held-out-proxy-discrepancy}
The residual treatment discrepancy of $\phi$ is
\begin{align}\label{eq:residual-treatment-discrepancy}
	D_{\mathrm{res},j}(\phi) &\triangleq\left[\mathbb E\bigl[\{q_{j,\phi}(T_j,Z_j^\phi)-e_{j,\phi}(Z_j^\phi)\}^2\bigr]\right]^{1/2}.
\end{align}
\end{definition}
The discrepancy $D_{\mathrm{res},j}(\phi)$ is the $L_2$ magnitude of the additional treatment information in $T_j=(X,W_j)$ beyond $Z_j^\phi$. It characterizes exact conditional balance as follows.

\begin{proposition}[Zero discrepancy]
\label{prop:zero-discrepancy}
	$D_{\mathrm{res},j}(\phi)=0$ if and only if $A\perp\!\!\!\perp(X,W_j)\mid Z_j^\phi$.
\end{proposition}

To translate residual treatment discrepancy into bias of an adjusted effect, we impose the following stability condition.
\begin{assumption}[$L_2$ outcome-relevant stability]
\label{ass:uniform-stability}
For each encoder $\phi$, let $\mu^\phi_{a,j,z}$ be $\mu_{a,j,z}$ with $Z_j^\phi = z$, let $P_{\phi,z}\triangleq P(T_j\mid Z_j^\phi=z)$, and for a signed measure $\nu$ with a density with respect to $P_{\phi,z}$ write $\|\nu\|_{\phi,z}\triangleq\|d\nu/dP_{\phi,z}\|_{L_2(P_{\phi,z})}$. There is a finite $\Gamma_j(\phi)$ such that, for almost every $z$, each $a\in\{0,1\}$, and all $Q,Q'$ as in Assumption~\ref{ass:outcome-relevant-completeness} with $Z_j^\phi$ in place of $Z_j$,
\[|\mathbb E_Q\mu^\phi_{a,j,z}-\mathbb E_{Q'}\mu^\phi_{a,j,z}|\le\Gamma_j(\phi)\,\|\mathcal K_jQ-\mathcal K_jQ'\|_{\phi,z}.\]
\end{assumption}
 If $\mathcal K_jQ=\mathcal K_jQ'$, the right side of Assumption~\ref{ass:uniform-stability} is zero, so Assumption~\ref{ass:uniform-stability} implies Assumption~\ref{ass:outcome-relevant-completeness} for $Z_j^\phi$. It gives additional rate information: when the two distributions of $T_j$ are close, the two means of $Y(a)$ are close, within the unobserved factor $\Gamma_j(\phi)$.

Under Assumption~\ref{ass:uniform-stability}, imperfect balance yields a quantitative bound on causal bias.

\begin{theorem}[Approximate-balance bias]
\label{thm:approximate-proxy-bias-bound}
For any representation $Z$, define its ordinary adjustment value as $\tau_Z\triangleq\mathbb E[\mathbb E(Y\mid A=1,Z)-\mathbb E(Y\mid A=0,Z)]$.
Under Assumptions~\ref{ass:consistency}, \ref{ass:joint-latent-exchangeability}, and \ref{ass:uniform-stability}, and overlap $\eta\le P(A=1\mid Z_j^\phi)\le1-\eta$ almost surely for some $\eta\in(0,1/2]$,
\begin{align}
|\tau_{Z_j^\phi}-\tau| &\le\frac{\Gamma_j(\phi)}{\eta(1-\eta)}D_{\mathrm{res},j}(\phi).
\label{eq:uniform-approximate-proxy-bias-bound}
\end{align}
\end{theorem}

\section{Learning, Estimation, and Search}
\label{sec:learning-search}
\label{sec:estimation}

This section gives the finite-sample proximal-balancing method and algorithm, called PROxy Blockwise Exclusion (PROBE). Throughout this section, we use three independent samples: a discrepancy sample $\mathcal I_{\mathrm D}$ of size $n_{\mathrm D}$ for representation learning, a nuisance sample $\mathcal I_{\mathrm N}$ for outcome and propensity models, and an evaluation sample $\mathcal I_{\mathrm E}$ of size $n_{\mathrm E}$.

We restrict the candidate encoder class to representations satisfying a common overlap condition.

\begin{assumption}[Common representation overlap]
\label{ass:library-overlap}
There is one $\eta\in(0,1/2]$ such that, for every candidate encoder $\phi$, $\eta\le P(A=1\mid Z_j^\phi)\le1-\eta$ almost surely.
\end{assumption}

We estimate residual treatment discrepancy by comparing how well treatment $A$ can be predicted from $Z_j^\phi$ alone and from $(T_j,Z_j^\phi)$. Let $\mathcal F_{j,n_{\mathrm D}}$ be a prespecified encoder class with $Z_j^\phi=\phi(V_j)$. Let $\mathcal H_{0,n_{\mathrm D}}$ and $\mathcal H_{1,n_{\mathrm D}}$ be prespecified classes of $[0,1]$-valued treatment predictors based on $Z_j^\phi$ and $(T_j,Z_j^\phi)$, respectively. We assume that $\mathcal H_{1,n_{\mathrm D}}$ includes $\mathcal H_{0,n_{\mathrm D}}$ as the subclass.

\begin{definition}[Empirical Brier discrepancy]
\label{def:empirical-held-out-proxy-discrepancy}
	For $Z_{j,i}^\phi=\phi(V_{j,i})$, define
	\begin{align}
		\widehat R_{0,\phi}(h) &\triangleq \frac{1}{n_{\mathrm D}} \sum_{i\in\mathcal I_{\mathrm D}} \{A_i-h(Z_{j,i}^\phi)\}^2, \qquad \widehat R_{1,\phi}(h) \triangleq \frac{1}{n_{\mathrm D}} \sum_{i\in\mathcal I_{\mathrm D}} \{A_i-h(T_{j,i},Z_{j,i}^\phi)\}^2, \notag\\
		\widetilde D_{j,n_{\mathrm D}}^2(\phi) &\triangleq \left[ \widehat R_{0,\phi}(\widehat h_{0,\phi}) - \widehat R_{1,\phi}(\widehat h_{1,\phi}) \right]_+ \label{eq:empirical-brier-fitting}, \quad \widehat h_{b,\phi}\in\arg\min_{h\in\mathcal H_{b,n_{\mathrm D}}}\widehat R_{b,\phi}(h), \text{ for }b\in\{0,1\}.
	\end{align}
	
\end{definition}
 The Brier risks are minimized by $e_{j,\phi}$ and $q_{j,\phi}$ of Section~\ref{sec:imperfect-balance}, so $\widetilde D_{j,n_{\mathrm D}}^2(\phi)$ estimates $D_{\mathrm{res},j}^2(\phi)$.

\begin{theorem}[Learned-representation error]
\label{cor:representation-causal}
	Fix $\delta_{\mathrm D}\in(0,1)$, and select a measurable empirical minimizer $\widehat\phi_j \in \operatorname*{arg\,min}_{\phi\in\mathcal F_{j,n_{\mathrm D}}} \widetilde D_{j,n_{\mathrm D}}^2(\phi)$. Set $\widehat Z_j\triangleq\widehat\phi_j(V_j)$ and $\tau_{\widehat Z_j}\triangleq \mathbb E[ \mathbb E(Y\mid A=1,\widehat Z_j) - \mathbb E(Y\mid A=0,\widehat Z_j)]$. Let $\xi_{j,n_{\mathrm D}}(\delta_{\mathrm D})$ be a nonnegative deterministic radius satisfying
	\begin{align}\label{eq:uniform-discrepancy-radius}
		P\left\{
		\sup_{\phi\in\mathcal F_{j,n_{\mathrm D}}}
		\left|
		\widetilde D_{j,n_{\mathrm D}}^2(\phi)
		-
		D_{\mathrm{res},j}^2(\phi)
		\right|
		\le
		\xi_{j,n_{\mathrm D}}(\delta_{\mathrm D})
		\right\}
		&\ge
		1-\delta_{\mathrm D}.
	\end{align}
	Suppose Assumptions~\ref{ass:consistency} and \ref{ass:joint-latent-exchangeability} hold, Assumption~\ref{ass:uniform-stability} holds for every $\phi\in\mathcal F_{j,n_{\mathrm D}}$, and Assumption~\ref{ass:library-overlap} holds. Then, with probability at least $1-\delta_{\mathrm D}$,
	\begin{align}\label{eq:representation-oracle-causal}
		|\tau_{\widehat Z_j}-\tau| &\le
		\frac{\Gamma_j(\widehat\phi_j)}{\eta(1-\eta)}
		\left[
		\inf_{\phi\in\mathcal F_{j,n_{\mathrm D}}}
		D_{\mathrm{res},j}^2(\phi)
		+
		2\xi_{j,n_{\mathrm D}}(\delta_{\mathrm D})
		\right]^{1/2}.
	\end{align}
\end{theorem}
Thm.~\ref{cor:representation-causal} extends Thm.~\ref{thm:approximate-proxy-bias-bound} to a learned representation: the infimum is the best balance the encoder class can reach, and $2\xi_{j,n_{\mathrm D}}(\delta_{\mathrm D})$ is the price of learning it from finite data.

The preceding theorem controls representation-learning bias. To estimate $\tau_{\widehat Z_j}$ while limiting bias from nuisance estimation, we use an augmented inverse-probability-weighted (AIPW) score, whose remainder is a product of propensity and outcome-regression errors \citep{robins1994estimation,tsiatis2006semiparametric,chernozhukov2017double,kennedy2022semiparametric}. Let $m_{a,j}(z)=\mathbb E[Y\mid A=a,\widehat Z_j=z]$ and $e_j(z)=P(A=1\mid\widehat Z_j=z)$. Let $\widehat{m}_{a,j}(z), \widehat{e}_j(z)$ denote its estimator fitted using samples $\mathcal I_{\mathrm N}$, clip $\widehat e_j$ to $[\eta,1-\eta]$, and evaluate them on $\mathcal I_{\mathrm E}$. At $z=\widehat Z_{j,i}$, the uncentered influence-function score \citep{robins1994estimation,tsiatis2006semiparametric} is
\begin{align}\label{eq:honest-aipw-score}
	\mathtt{UIF}_j(O_i) &\triangleq \widehat m_{1,j}(z)-\widehat m_{0,j}(z)+\frac{A_i\{Y_i-\widehat m_{1,j}(z)\}}{\widehat e_j(z)}-\frac{(1-A_i)\{Y_i-\widehat m_{0,j}(z)\}}{1-\widehat e_j(z)}.
\end{align}
Then, we define the AIPW estimator as 
\begin{align}\label{eq:aipw-estimator}
	\widehat\tau_j^{\mathrm{AIPW}} \triangleq \frac{1}{n_{\mathrm E}}\sum_{i\in\mathcal I_{\mathrm E}}\mathtt{UIF}_j(O_i).
\end{align}

\begin{theorem}[Honest AIPW error]
\label{thm:conditional-aipw-error}
Fix $\delta_{\mathrm E}\in(0,1)$. Suppose Assumptions~\ref{ass:consistency} and~\ref{ass:joint-latent-exchangeability} hold, Assumption~\ref{ass:uniform-stability} holds for $\widehat\phi_j$, Assumption~\ref{ass:library-overlap} holds, $\widehat e_j$ is clipped to $[\eta,1-\eta]$, and the variance defined below is finite. Conditional on $\mathcal I_{\mathrm D}$ and $\mathcal I_{\mathrm N}$, let $O$ be a fresh observation and define $\|g\|_{2,j}\triangleq[\mathbb E\{g(\widehat Z_j)^2\mid\mathcal I_{\mathrm D},\mathcal I_{\mathrm N}\}]^{1/2}$ and $\sigma_j^2\triangleq\operatorname{Var}\{\mathtt{UIF}_j(O)\mid\mathcal I_{\mathrm D},\mathcal I_{\mathrm N}\}$. Define 
	\begin{align}
		\mathtt{err}_{j} \triangleq \frac{1}{\eta}\|\widehat e_j-e_j\|_{2,j}\sum_{a=0}^1\|\widehat m_{a,j}-m_{a,j}\|_{2,j},\qquad\mathtt{bias}_j\triangleq\frac{\Gamma_j(\widehat\phi_j)}{\eta(1-\eta)}D_{\mathrm{res},j}(\widehat\phi_j).
	\end{align}
	Then, with probability at least $1-\delta_{\mathrm E}$ over $\mathcal I_{\mathrm E}$,
\begin{align}
|\widehat\tau_j^{\mathrm{AIPW}}-\tau_{\widehat Z_j}| \le \frac{\sigma_j}{\sqrt{n_{\mathrm E}\delta_{\mathrm E}}}+\mathtt{err}_j, \qquad |\widehat\tau_j^{\mathrm{AIPW}}-\tau| \le \frac{\sigma_j}{\sqrt{n_{\mathrm E}\delta_{\mathrm E}}}+\mathtt{err}_j+\mathtt{bias}_j.
\label{eq:end-to-end-sensitivity}
\end{align}
\end{theorem}
Thm.~\ref{thm:conditional-aipw-error} adds sampling and nuisance errors to Thm.~\ref{cor:representation-causal}; the nuisance term $\mathtt{err}_j$ is doubly robust, vanishing if the propensity model or both outcome models are correct, but $\mathtt{bias}_j$ remains.

So far, we have treated $j$ as a fixed held-out block. When the valid block is unknown, PROxy Blockwise Exclusion (PROBE) in Algo.~\ref{alg:probe-search} turns the proximal plurality principle of Section~\ref{sec:unknown-blocks} into a finite-sample search. PROBE searches the $T=|\mathfrak S|=2^J-2$ splits and samples $m\in\{1,\ldots,T\}$ splits independently of data and learned objects, as in Algorithm~\ref{alg:probe-search}. Fig.~\ref{fig:probe-mechanism} summarizes the search.

\begin{figure}[t]
\centering
\includegraphics[width=\linewidth]{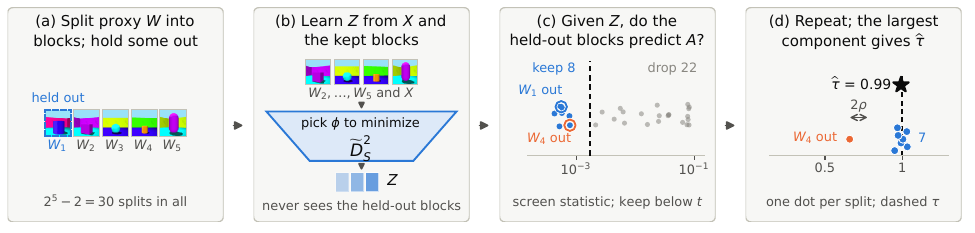}
\caption{PROBE learns, audits, and aggregates across held-out splits.}
\label{fig:probe-mechanism}
\end{figure}

\begin{algorithm}[!h]
\caption{PROBE split search}\label{alg:probe-search}
\KwIn{split budget $m$, screen threshold $t$, overlap $\eta$, linking radius $\rho$}
\textbf{Sample:} draw $m$ splits $S$ uniformly without replacement from $\mathfrak S$.\;
\textbf{Learn and screen:} for each sampled split $S$, learn $\widehat\phi_S$ on $\mathcal I_{\mathrm D}$ by minimizing the discrepancy in Definition~\ref{def:empirical-held-out-proxy-discrepancy}; fit an unclipped propensity $\widehat e_S^{\mathrm{raw}}$ on $\mathcal I_{\mathrm N}$; retain the sampled $S$ if $\widetilde D_{S,n_{\mathrm D}}^2(\widehat\phi_S)\le t$ and $\widehat e_S^{\mathrm{raw}}$ passes a prespecified empirical overlap check at level $\eta$. Define $\widehat{\mathcal B}$ as the \emph{screen set} of all splits $S$ satisfying these fixed checks.\;
\textbf{Estimate:} fit outcome regressions on $\mathcal I_{\mathrm N}$, set $\widehat e_S=\operatorname{clip}(\widehat e_S^{\mathrm{raw}},\eta,1-\eta)$, and compute $\widehat\theta_S=n_{\mathrm E}^{-1}\sum_{i\in\mathcal I_{\mathrm E}}\mathtt{UIF}_S(O_i)$ using Eq.~\eqref{eq:honest-aipw-score}.\;
\textbf{Aggregate:} link $S,S'$ when $|\widehat\theta_S-\widehat\theta_{S'}|\le2\rho$; return the median $\widehat\tau$ of the unique largest component, and declare failure if its size is tied.\;
\end{algorithm}

For $S\in\widehat{\mathcal B}$, the estimate $\widehat\theta_S$ targets $\theta_S\triangleq\tau_{\widehat\phi_S(V_S)}$. PROBE can fail in three ways: $\widehat\theta_S$ misses $\theta_S$ by sampling and nuisance error, $\theta_S$ misses $\tau$ because balance is only approximate, or splits whose held-out blocks violate Assumption~\ref{ass:joint-latent-exchangeability} or~\ref{ass:uniform-stability} agree on a wrong value and outnumber the rest. The next theorem bounds all three.

\begin{theorem}[Finite-sample PROBE cluster recovery]
\label{thm:probe-search-error}
\normalfont
Suppose Assumption~\ref{ass:library-overlap} holds and, for every $S\in\widehat{\mathcal B}$, $\widehat e_S$ is clipped to $[\eta,1-\eta]$ and $\sigma_S<\infty$, where $\sigma_S$, $\mathtt{err}_S$, and $\mathtt{bias}_S$ are the quantities of Thm.~\ref{thm:conditional-aipw-error} with $S$ in place of $j$. Let $C_0\subseteq\widehat{\mathcal B}$ be the splits whose held-out blocks satisfy Assumptions~\ref{ass:consistency}, \ref{ass:joint-latent-exchangeability}, and~\ref{ass:uniform-stability} with encoder $\widehat\phi_S$, let $N\triangleq|\widehat{\mathcal B}|$, and for $\delta\in(0,1)$ let
\[\varepsilon\triangleq\max_{S\in\widehat{\mathcal B}}\Bigl\{\sigma_S\sqrt{\frac{N}{n_{\mathrm E}\delta}}+\mathtt{err}_S\Bigr\},\qquad b\triangleq\max_{S\in C_0}\mathtt{bias}_S.\]
Suppose that
\begin{enumerate}[label=(\roman*),leftmargin=*,itemsep=0pt,topsep=2pt]
	\item the linking radius satisfies $\rho\ge b+\varepsilon$;
	\item the splits in $\widehat{\mathcal B} \setminus C_0$ fall into clusters $C_1,\ldots,C_L$, where $|\theta_{S'} - \theta_S| \leq 2(\rho-\varepsilon)$ whenever $S,S'$ are in the same cluster;
	\item values $\theta_S$ in different clusters among $C_0,\ldots,C_L$ are more than $2(\rho+\varepsilon)$ apart;
	\item $C_0$ is the largest cluster; $\Delta\triangleq(|C_0|-\max_{1\le\ell\le L}|C_\ell|)/N>0$ (or $\Delta=1$ if $L=0$, since then $C_0=\widehat{\mathcal B}$).
\end{enumerate}
Then, with probability at least $1-\delta-P(R=0)-L\,\mathbb E\{e^{-R\Delta^2/2}\}$, where $R\sim\operatorname{Hypergeometric}(T,N,m)$ is the number of the $m$ sampled splits that lie in $\widehat{\mathcal B}$, the output $\widehat\tau$ of Algorithm~\ref{alg:probe-search} satisfies
\begin{align}
|\widehat\tau-\tau|\le\max_{S\in\widehat{\mathcal B}}\Bigl\{\frac{\sigma_S\sqrt N}{\sqrt{n_{\mathrm E}\delta}}+\mathtt{err}_S\Bigr\}+\max_{S\in C_0}\mathtt{bias}_S.
\label{eq:probe-search-error-probability}
\end{align}
\end{theorem}

Eq.~\eqref{eq:probe-search-error-probability} is Eq.~\eqref{eq:end-to-end-sensitivity} applied to every retained split at $\delta_{\mathrm E}=\delta/N$: its first maximum collects the sampling and doubly robust nuisance terms over all retained splits, its second collects the bias term over $C_0$, and the search adds the failure probability $P(R=0)+L\,\mathbb E\{e^{-R\Delta^2/2}\}$. Thm.~\ref{thm:probe-search-error} quantifies the cost of examining only $m$ uniformly sampled splits rather than all of $\mathfrak S$~through two terms: $P(R=0)$, the chance that no sampled split is retained, and $L\,\mathbb E\{e^{-R\Delta^2/2}\}$, which bounds the chance that the sampled splits of $C_0$ fail to outnumber those of every other cluster.

\section{Experiments}
\label{sec:main-experiments}

This section provides empirical demonstration of the proposed theories. Appendix~\ref{sec:appendix-experiments} gives the simulation details. Fig.~\ref{fig:scenarios} illustrates the causal diagram of every scenario in which the true effect is known, and each analysis below reads its setting from that diagram. We estimate the average treatment effect $\tau$ through PROBE (Algo.~\ref{alg:probe-search}), and report the mean absolute error (MAE) over independently generated data sets.

We compare PROBE with various other methods. CEVAE \citep{louizos2017causal} learns a latent variable from all of $(X,W)$. Two further methods adjust for $(X,W)$ and for $X$ alone. For image proxies, the pixel CNN adjusts for a representation that a convolutional network \citep{lecun1998gradient} learns from all pixels. None of these methods holds out a block or checks balance. The remaining methods need the role of each proxy, which the diagram determines. Proximal two-stage least squares (2SLS) \citep{tchetgen2024ProximalCausalInference} needs a treatment-side proxy, which may point into $A$ but not into $Y$, and an outcome-side proxy, which may point into $Y$ but not into $A$ (Fig.~\ref{fig:causal-structures}(c)); single-proxy control \citep{park2024single} and its kernel version KSPC \citep{xu2025kernel} need one proxy of the untreated outcome $Y(0)$. We run each of these three methods once with the roles of the true diagram and once with wrong roles. An oracle that adjusts for $(X,U)$ blocks the path $A\leftarrow U\rightarrow Y$ directly and serves as a reference. In the chest X-ray diagrams (Fig.~\ref{fig:scenarios}(d), (e)), the proxy $W$ is one chest X-ray per unit and there are no covariates, so adjustment for $X$ is the difference in means; there PROBE is compared with adjustment for $X$, the pixel CNN, and the oracle.

\begin{figure}[t]
\centering
\includegraphics[width=\linewidth]{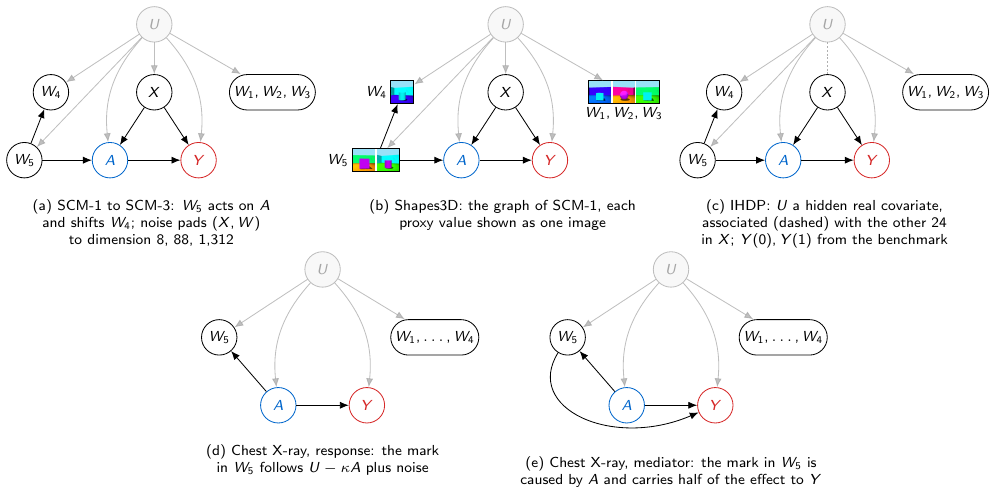}
\caption[Causal diagrams of the experiments.]{Causal diagrams of the experiments. The hidden confounder $U$ is gray, the treatment $A$ blue, and the outcome $Y$ red; the dashed line in (c) is an association between real covariates.}
\label{fig:scenarios}
\end{figure}

\paragraph{Synthetic Data Analysis.} In our example scenario (Fig.~\ref{fig:scenarios}(a)), a hidden severity $U$ affects both the treatment and the outcome, and the true effect is $\tau=1$. $U$ also affects the covariates $X$, which affect $A$ and $Y$. The proxy has five blocks, and $U$ points into each of them. Blocks $W_1$, $W_2$, and $W_3$ receive no other edge: they are noisy measurements of $U$ and are valid held-out blocks, each satisfying Assumptions~\ref{ass:joint-latent-exchangeability} and~\ref{ass:outcome-relevant-completeness}. Block $W_5$ also points into $A$, so it violates held-out treatment independence (Assumption~\ref{ass:joint-latent-exchangeability}), and no representation can balance it. Block $W_4$ points into neither $A$ nor $Y$ but is shifted by $W_5$; it can be balanced but violates outcome-relevant completeness (Assumption~\ref{ass:outcome-relevant-completeness}), so it points to a wrong effect. PROBE is not told which blocks are valid, whereas the role-requiring methods are given the correct roles in their correct-role runs. The proxy is a vector, high-dimensional in SCM-2 and SCM-3. The methods that need roles are told which coordinates form each block; PROBE is not, and finds the blocks from $W$ alone, before it sees $A$ or $Y$ (Appendix~\ref{sec:appendix-block-discovery}). We generate $100$ data sets of size $n=24{,}000$ for each setting.

\begin{figure}[t]
\centering
\includegraphics[width=\linewidth]{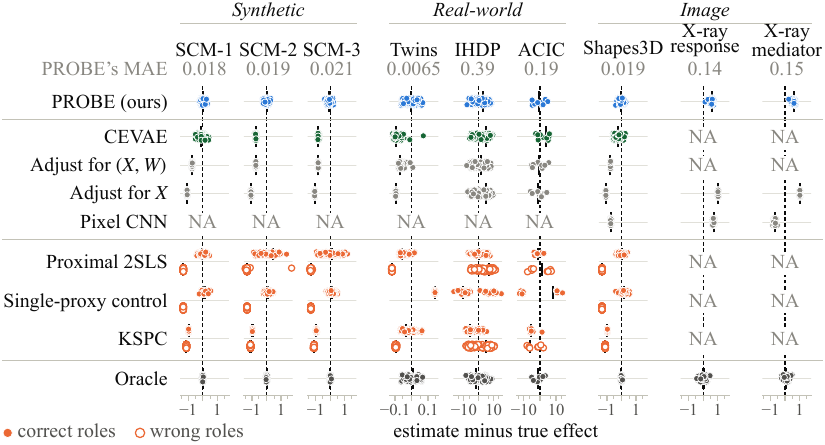}
\caption[Errors in the synthetic, real-world, and image data.]{Estimate minus the true effect on a symmetric log scale; each dot is one data set. Filled and hollow dots show methods that need proxy roles with correct and wrong roles; on the real-world data, single-proxy control is used as an ATE estimator. In the synthetic and real-world panels PROBE finds the blocks from $W$ (Appendix~\ref{sec:appendix-block-discovery}); in the image panels the blocks are given.}
\label{fig:experiments}\label{fig:synthetic}\label{fig:realdata}
\end{figure}
Fig.~\ref{fig:synthetic} (Synthetic) shows that PROBE, without being told the roles or the blocks, stays near zero error in all three settings, while CEVAE degrades as the dimension grows and adjustment for $(X,W)$ or $X$ stays biased.

\paragraph{Real-World Data Analysis.}\label{sec:real-world-analysis} We use four real-world data sets. Twins \citep{louizos2017causal}, IHDP \citep{hill2011bayesian}, and ACIC 2016 \citep{dorie2019automated} are semi-synthetic benchmarks built on real covariates, with $11{,}984$, $747$, and $4{,}802$ units and $42$, $24$, and $78$ covariates. In each, we hide one covariate as $U$ and generate the proxies and the treatment from it with one fixed recipe, so the effect is known exactly. Fig.~\ref{fig:scenarios}(c) shows the resulting graph: the proxies and the treatment follow the graph of the synthetic scenario, the hidden covariate $U$ is associated with the remaining covariates $X$ instead of being their cause, and the potential outcomes are those of the benchmark. We also analyze the right heart catheterization (RHC) study \citep{connors1996effectiveness}, an observational study with no ground truth; because Fig.~\ref{fig:experiments} compares methods only where the true effect is known, the RHC analysis is in Appendix~\ref{sec:appendix-rhc}.

Fig.~\ref{fig:realdata} (Real-world) shows that PROBE, without being told the roles or the blocks, is the most accurate method on Twins and among the three most accurate on IHDP and ACIC. Methods that need the roles are accurate only with the correct roles and degrade sharply with wrong ones.

\par\addvspace{3.25ex plus 1ex minus .2ex}
\begin{wrapfigure}{r}{2.78in}
\centering
\includegraphics[width=2.70in]{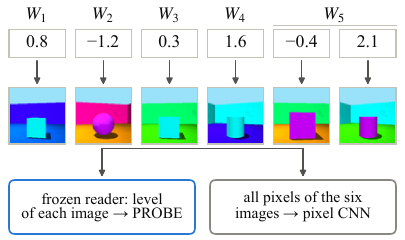}
\captionsetup{format=plain}
\caption[Shapes3D setting.]{Shapes3D setting: each proxy value of a unit sets the orientation of one Shapes3D image.}
\label{fig:shapes3d}
\end{wrapfigure}
\noindent{\usekomafont{disposition}\usekomafont{paragraph}Shapes3D Analysis.}\hspace{1em}In this scenario (Fig.~\ref{fig:scenarios}(b)), the true effect is $\tau=1$, $W_1$, $W_2$, and $W_3$ are valid held-out blocks, and $W_4$ and $W_5$ are not. Here every proxy block is made of images. Each unit is first generated as in SCM-1. The hidden confounder $U$ itself is never observed; what is generated from it are six proxy values, one in each of $W_1$ to $W_4$ and two in $W_5$. The values in $W_1$ to $W_3$ are noisy measurements of $U$. The two values in $W_5$ mix a measurement of $U$ with a second value that affects only the treatment, and the value in $W_4$ is a measurement of $U$ shifted by that second value, which is why $W_4$ and $W_5$ are not valid held-out blocks. Each proxy value is then shown as an image instead of a number. Shapes3D \citep{3dshapes18} is a collection of $64\times64$ images of an object in a room whose colors, shape, size, and orientation vary; for each proxy value we take one such image whose orientation, the angle at which the scene is turned, equals the value rounded to one of $15$ angles, and whose colors, shape, and size are random. A unit's proxy is therefore six images, one in each of $W_1$ to $W_4$ and two in $W_5$, and the analyst sees only their pixels (Fig.~\ref{fig:shapes3d}; SCM-4 in Appendix~\ref{sec:appendix-image-rendering}). To use these images, PROBE first turns each one back into its proxy value with an image reader, a convolutional network trained beforehand on other Shapes3D images to recognize the orientation, which reads every image of this experiment correctly. It then runs on the six recovered values as in SCM-1. The pixel CNN instead adjusts for a summary that it learns from all pixels of the six images, including the images of $W_4$ and $W_5$. We generate $100$ data sets of size $n=24{,}000$. Fig.~\ref{fig:experiments} (Image, Shapes3D) shows that PROBE, without being told the roles, has an MAE of $0.019$, as in SCM-1, while adjustment for $(X,W)$ has an MAE of $0.31$ and the pixel CNN, which reads every pixel but never checks balance, $0.29$.

\paragraph{Chest X-ray Analysis.}\label{sec:single-image} In this scenario, a hidden severity $U$ affects both the treatment and the outcome, the true effect is $\tau=1$, the proxy $W$ is one of $5{,}856$ pediatric chest X-rays \citep{kermany2018identifying,yang2023medmnist}, and $U$ is the score of a pneumonia classifier on that X-ray. The image is cut into five blocks, and $U$ points into each of them (Fig.~\ref{fig:scenarios}(d), (e)). The lung quadrants $W_1$ to $W_4$ receive no other edge and satisfy held-out treatment independence (Assumption~\ref{ass:joint-latent-exchangeability}), and the central strip $W_5$ violates it, because $A$ also points into it: the mark in $W_5$ is either a response that follows $U-\kappa A$ plus noise, so that $A$ and $U$ both point into $W_5$ (d), or a mediator that passes half of the effect on to $Y$, so that $A$ points into $W_5$ and $W_5$ into $Y$ (e). PROBE is not told which block carries the mark and screens one block at a time (Appendix~\ref{sec:appendix-single-image}); the pixel CNN adjusts for all five blocks. We generate $50$ data sets for each scenario.

\begin{figure}[tb]
\centering
\includegraphics[width=\linewidth]{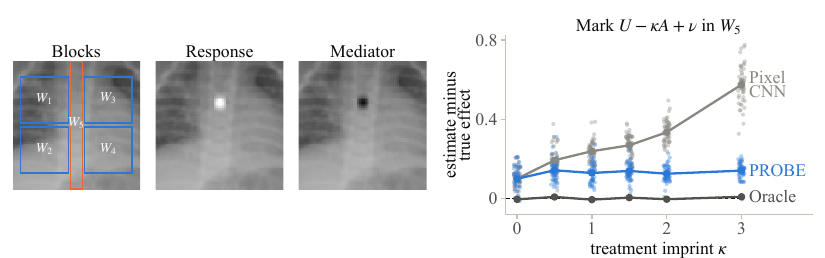}
\caption[Chest X-ray proxy.]{Left: the five blocks of one chest X-ray, and the same X-ray with a response mark and with a mediator mark in $W_5$. Right: estimate minus the true effect as the treatment's imprint $\kappa$ grows; $50$ data sets per value.}
\label{fig:single-image}
\end{figure}

Fig.~\ref{fig:experiments} (Image, X-ray) shows that PROBE hides $W_5$ in all $50$ data sets, with an MAE of $0.14$ under the response and $0.15$ under the mediator, against $0.27$ and $0.26$ for the pixel CNN. As the imprint $\kappa$ grows from $0$ to $3$ (Fig.~\ref{fig:single-image}, right), the MAE of the pixel CNN grows from $0.10$ to $0.57$, as Prop.~\ref{prop:single-image-decomposition}(b) predicts, while that of PROBE stays between $0.10$ and $0.14$; Appendix~\ref{sec:appendix-single-image} reports the other scenarios.

\section{Conclusion}
\label{sec:main-conclusion}

We introduced proximal balancing, which estimates the average treatment effect from a multi-dimensional proxy without designating proxy roles, solving an inverse problem, or modeling the hidden confounder. When a held-out block does not affect the treatment and preserves the outcome-relevant part of the confounder, adjusting for a representation that balances it identifies the effect (Thm.~\ref{thm:exact-proxy-identification}). When the valid block is unknown, a plurality vote over held-out blocks identifies the effect (Thm.~\ref{thm:unknown-valid-block-identification}); under approximate balance, the bias is bounded by a stability-weighted residual treatment discrepancy (Thm.~\ref{thm:approximate-proxy-bias-bound}); and our algorithm, PROBE (Algo.~\ref{alg:probe-search}), has finite-sample guarantees for representation learning, honest AIPW estimation, and search (Thms.~\ref{cor:representation-causal}--\ref{thm:probe-search-error}). In synthetic data with vector and image proxies, PROBE had the smallest error of all competing methods other than the oracle, and on real-world data it remained competitive without role information, while methods given wrong roles degraded sharply. When the whole proxy is one chest X-ray, cutting it into blocks and excluding the blocks that PROBE's held-out screen flags removed most of the error that a treatment response, a mediator, or a device together with a response in one or two regions added to a whole-image adjustment, left a clean image unchanged, and became more valuable as the treatment's imprint in the image grew; it cannot help when such content spans every region (Section~\ref{sec:single-image}).

\endgroup

\bibliographystyle{apalike}
\bibliography{reference}

\clearpage
\appendix

\begingroup
\allowdisplaybreaks
\sloppy
\section{Problem Setting}
\label{sec:setup}

We observe independent copies
\(O_i=(X_i,W_i,A_i,Y_i), \qquad i=1,\ldots,n,\)
where $X\in\mathcal X$ is a possibly high-dimensional vector of observed pretreatment covariates, $\mathcal X$ is its value space, $A\in\{0,1\}$ is treatment, and $Y$ is the observed outcome.

For each $a\in\{0,1\}$, the \emph{potential outcome} $Y(a)$ is the outcome that would be observed if treatment were set to $a$. Our target is the population average treatment effect
\(\tau\triangleq\mathbb E\{Y(1)-Y(0)\}.\)

Let $U$ denote an unobserved pretreatment variable that may affect both treatment $A$ and outcome $Y$. When $U$ affects both $A$ and $Y$, the observed-data distribution alone does not identify $\tau$ without additional structure. We therefore observe a proxy measurement $W$ that contains information about $U$. The proxy is divided into prespecified blocks,
\(W=(W_1,\ldots,W_J).\)
Each block may be scalar or vector-valued. A block can represent one repeated measurement, sensor channel, laboratory panel, or prespecified region of a larger measurement.

For each block index $j\in\{1,\ldots,J\}$, we call $W_j$ the \emph{held-out proxy block} and let $\mathcal W_j$ denote its value space. We call the combination of all other blocks the \emph{remaining blocks} (Definition~\ref{def:held-out-remaining-blocks}) and write
\(W_{-j} \triangleq (W_1,\ldots,W_{j-1},W_{j+1},\ldots,W_J)\)
for the resulting vector of remaining proxy blocks. Let $\mathcal V_j$ denote the value space of the pair $(X,W_{-j})$. Finally, define
\(V_j\triangleq(X,W_{-j})\in\mathcal V_j.\)
The variable $V_j$ is the combined vector of the observed covariates $X$ and the remaining proxy blocks $W_{-j}$. We also write $T_j\triangleq(X,W_j)$ for the observed covariates together with the held-out block.

\section{Identification}
\label{app:identification-details}

Let $d_{V_j}$ denote the coordinate dimension of $V_j$. Fix a user-specified output dimension $d_{Z_j}\le d_{V_j}$. For a measurable representation function, define
	\(\phi:\mathcal V_j\longrightarrow\mathbb R^{d_{Z_j}}, \qquad Z_j^\phi\triangleq\phi(V_j).\)
Let $\Phi_j$ denote a prespecified class of such representation functions.
When a representation function $\phi_j\in\Phi_j$ is fixed, we abbreviate $Z_j^{\phi_j}=\phi_j(V_j)$ as $Z_j$.

\subsection{Identification for a Prespecified Held-Out Block} \label{sec:fixed-held-out-block-identification}

\subsubsection{Assumptions} \label{sec:identification-assumptions}

\begin{assumption}[Consistency and integrability; restatement of Assumption~\ref{ass:consistency}]
	\label{appdetail:ass:consistency}
	For each $a\in\{0,1\}$,
	\(Y=Y(A), \qquad \mathbb E|Y(a)|<\infty.\)
\end{assumption}

\begin{assumption}[Latent exchangeability and held-out treatment independence; restatement of Assumption~\ref{ass:joint-latent-exchangeability}] \label{appdetail:ass:joint-latent-exchangeability} \label{appdetail:ass:held-out-separation}
	For each $a\in\{0,1\}$,
	\(Y(a)\perp\!\!\!\perp A\mid(U,V_j)\) and \(W_j\perp\!\!\!\perp A\mid(U,V_j).\)
\end{assumption}

The first clause says that $(U,V_j)$ is sufficient to remove hidden confounding between treatment and each potential outcome.
The second clause says that, after $(U,V_j)$ is fixed, treatment status provides no additional information about the distribution of $W_j$.

Fig.~\ref{fig:proximal-causal-structures} shows three examples compatible with proximal balancing and one representative standard proximal structure.

\begin{figure}[!t]
	\centering
	\includegraphics[width=0.86\textwidth]{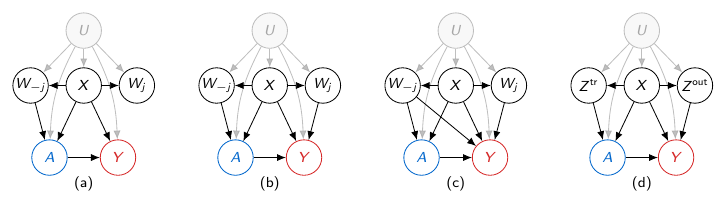}
	\caption{Causal structures compatible with proximal balancing in panels (a)--(c) and a standard two-proxy proximal structure in panel (d).}
	\label{fig:proximal-causal-structures}
\end{figure}

We then introduce the latent and proxy objects within each representation stratum.

\begin{definition}[Latent and proxy objects within a representation stratum]
\label{appdetail:def:latent-proxy-objects}
\normalfont
	Let $P_{Z_j}$ denote the distribution of $Z_j$. For $P_{Z_j}$-almost every $z$, fix versions of the following conditional distributions. Define
	\(Q_{a,j,z} \triangleq P(V_j,U\mid A=a,Z_j=z), \qquad a\in\{0,1\},\)
	and, by held-out treatment independence, define the treatment-invariant conditional distribution of the held-out block by
	\(K_j(\cdot\mid v,u) \triangleq P(W_j\mid V_j=v,U=u).\)
	For any probability distribution $Q$ on $(V_j,U)$, let $\mathcal K_jQ$ denote the induced distribution of $(X,W_j)$ obtained by first drawing $(V_j,U)=(X,W_{-j},U)\sim Q$, then drawing $W_j\sim K_j(\cdot\mid V_j,U)$, and finally retaining $(X,W_j)$.

	For each $a\in\{0,1\}$, define
	\(\mu_{a,j,z}(v,u) \triangleq \mathbb E\{Y(a)\mid V_j=v,U=u,Z_j=z\}.\)
	For $P_{Z_j}$-almost every $z$, let $\mathcal Q_{j,z}$ be the collection of probability distributions $Q$ on $(V_j,U)$ such that $Q\ll P(V_j,U\mid Z_j=z)$ and $\mathbb E_Q|\mu_{a,j,z}(V_j,U)|<\infty$ for both $a\in\{0,1\}$.
\end{definition}

\begin{assumption}[Outcome-relevant completeness; restatement of Assumption~\ref{ass:outcome-relevant-completeness}]
	\label{appdetail:ass:outcome-relevant-completeness}
	For $P_{Z_j}$-almost every $z$, every $Q,Q'\in\mathcal Q_{j,z}$, and each $a\in\{0,1\}$,
	\(\mathcal K_jQ=\mathcal K_jQ' \quad\Longrightarrow\quad \mathbb E_Q\{\mu_{a,j,z}(V_j,U)\} = \mathbb E_{Q'}\{\mu_{a,j,z}(V_j,U)\}.\)
\end{assumption}

Assumption~\ref{ass:outcome-relevant-completeness} is stated for all pairs $Q,Q'\in\mathcal Q_{j,z}$ (Definition~\ref{appdetail:def:latent-proxy-objects}) within a representation stratum. The next proposition gives concrete sufficient conditions. It treats the held-out block as a measurement of $(X,U)$ that is conditionally independent of the remaining blocks, and it asks the measurement channel to be complete.

\begin{proposition}[Sufficient conditions for outcome-relevant completeness]
\label{appdetail:prop:orc-sufficient}
\normalfont

Fix a block $j$ and a representation $\phi_j$. Suppose that
\begin{enumerate}[label=(\roman*),leftmargin=*]
\item $W_j\perp\!\!\!\perp W_{-j}\mid(X,U)$ and, for each $a\in\{0,1\}$, $Y(a)\perp\!\!\!\perp W_{-j}\mid(X,U,Z_j)$, so that $K_j(\cdot\mid v,u)$ and $\mu_{a,j,z}(v,u)$ depend on $v=(x,w_{-j})$ only through $x$;
\item for every $x$, the map $\pi\mapsto\int K_j(\cdot\mid x,u)\,\pi(du)$ from probability distributions of $U$ to distributions of $W_j$ is injective.
\end{enumerate}
Then Assumption~\ref{ass:outcome-relevant-completeness} holds. Condition (ii) holds in the following two specializations, which are not exhaustive.
\begin{enumerate}[label=(\alph*),leftmargin=*]
\item \textit{Injective decoder with additive noise.} $W_j=g(X,U)+\epsilon_j$, where $\epsilon_j$ is independent of $(X,W_{-j},U)$ and has a characteristic function without zeros, as Gaussian noise does, and $u\mapsto g(x,u)$ is injective for every $x$.
\item \textit{Count channel with an injective rate.} $W_j$ is a vector of counts whose coordinates are conditionally independent given $(X,U)$, with $W_{j,\ell}\mid(X=x,U=u)\sim\operatorname{Poisson}\{\lambda_\ell(x,u)\}$ for each coordinate $\ell$, and $u\mapsto\lambda(x,u)\triangleq(\lambda_1(x,u),\ldots,\lambda_{d}(x,u))$ is injective for every $x$.
\end{enumerate}

\end{proposition}

Condition (i) lets the remaining blocks influence the outcome only through $(X,U)$ and through components retained in $Z_j$; Thm.~\ref{thm:exact-proxy-identification}, restated below, uses the more general Assumption~\ref{ass:outcome-relevant-completeness}. The first specialization covers latent-variable generative models in which the held-out block is a decoded image or embedding of $(X,U)$ plus independent noise, and the second covers count-valued proxies such as event, word, or sequencing-read counts. When the latent state takes finitely many values and both proxy blocks are pure measurements satisfying the rank conditions of \citet{kuroki2014measurement}, their matrix restoration identifies $\tau$ directly.

\subsubsection{Exact Identification for a Fixed Block}
\label{sec:exact-proxy-identification}

The next theorem gives the central identification result. It states when adjustment for $Z_j$ identifies the population average treatment effect.

\begin{theorem}[Identification from a held-out audit; restatement of Thm.~\ref{thm:exact-proxy-identification}] \label{appdetail:thm:exact-proxy-identification}
\normalfont
Fix a prespecified block $j\in\{1,\ldots,J\}$. Suppose Assumptions~\ref{ass:consistency}--\ref{ass:outcome-relevant-completeness} hold for this block and its representation. If $Z_j=\phi_j(V_j)$ is constructed such that
	\(0<P(A=1\mid Z_j)<1 \quad\text{almost surely},\)
and
	\((X,W_j)\perp\!\!\!\perp A\mid Z_j,\)
	then
	\(\tau = \mathbb E\!\left[ \mathbb E(Y\mid A=1,Z_j) - \mathbb E(Y\mid A=0,Z_j) \right].\)
\end{theorem}

The theorem is a conditional mean identification result. It establishes that the conditional mean of $Y(a)$ given $(A,Z_j)$ does not depend on $A$, and consistency converts this equality into the displayed adjustment formula. It does not assert the full conditional independence $Y(a)\perp\!\!\!\perp A\mid Z_j$.

\subsubsection{Comparison with the Canonical Two-Proxy Proximal Model}

Graphically and structurally, proximal balancing differs from the standard proximal causal inference framework \citep{miao2018proxy,tchetgen2024ProximalCausalInference}. The standard framework prespecifies a treatment-inducing proxy $Z^{\mathrm{tr}}$ and an outcome-inducing proxy $Z^{\mathrm{out}}$. By contrast, proximal balancing requires only the held-out block $W_j$ to satisfy $W_j\perp\!\!\!\perp A\mid(U,V_j)$. The remaining block is not subject to this restriction and may affect $A$, $Y$, both, or neither, as illustrated by panels (a)--(c).

The underlying assumptions and identification strategies are also different. A representative canonical completeness condition is, for every $(a,x)$,
\(\mathbb E\{g(U)\mid Z^{\mathrm{tr}},A=a,X=x\}=0\text{ almost surely} \quad\Longrightarrow\quad g(U)=0 \quad\text{almost surely}\)
for functions $g$ in the specified class. Identification then proceeds by seeking an outcome bridge $h$ satisfying an observed conditional-moment equation such as
	\(\mathbb E(Y\mid Z^{\mathrm{tr}},A,X) = \mathbb E\{h(Z^{\mathrm{out}},A,X)\mid Z^{\mathrm{tr}},A,X\}.\)

Proximal balancing instead combines observable balance in Thm.~\ref{thm:exact-proxy-identification} with outcome-relevant completeness in Assumption~\ref{ass:outcome-relevant-completeness}. For comparison, the canonical proximal conditions refer to proxy validity, existence of an integrable outcome bridge, and treatment-proxy completeness for a designated proxy pair. The proximal-balancing conditions refer to Assumptions~\ref{ass:consistency}--\ref{ass:outcome-relevant-completeness}, representation overlap, and exact balance for a designated proxy split. A natural question is whether either framework's sufficient conditions contain those of the other.

\begin{proposition}[Non-nesting of canonical proximal and proximal-balancing sufficient conditions]
	\label{prop:non-nesting-proximal-probe}
	\normalfont
	There exists a causal data-generating process for which the canonical proximal conditions hold for a designated proxy pair $(Z^{\mathrm{tr}},Z^{\mathrm{out}})$ but the proximal-balancing conditions fail for the split with $W_{-j}=Z^{\mathrm{tr}}$ and $W_j=Z^{\mathrm{out}}$ under every representation $\phi_j$. Conversely, there exists a causal data-generating process for which the proximal-balancing conditions hold for a designated proxy split but the canonical proximal conditions fail for the pair $(Z^{\mathrm{tr}},Z^{\mathrm{out}})=(W_{-j},W_j)$. Moreover, neither canonical treatment-proxy completeness nor Assumption~\ref{ass:outcome-relevant-completeness} implies the other.
\end{proposition}

\subsection{Identification When the Valid Held-Out Block Is Unknown}
\label{sec:unknown-valid-held-out-block}

Section~\ref{sec:fixed-held-out-block-identification} fixes the held-out block in advance. We now allow its composition to be unknown. Consider every nonempty proper subset
\begin{align}
	\mathfrak S \triangleq
	\{S\subset\{1,\ldots,J\}:1\le |S|\le J-1\}.
	\label{eq:all-held-out-orientations}
\end{align}
A subset $S$ and its complement $S^c$ are distinct because they exchange the held-out and remaining roles. For each $S\in\mathfrak S$, let $W_S=(W_j)_{j\in S}$ be the held-out vector, let $W_{-S}=(W_j)_{j\notin S}$ be the remaining vector, and define
\(V_S\triangleq(X,W_{-S}), \qquad Z_S\triangleq\phi_S(V_S),\)
using the same construction as for a single held-out block.

For each $S\in\mathfrak S$, define, whenever it is well-defined,
\(\theta_S \triangleq \mathbb E\!\left[ \mathbb E(Y\mid A=1,Z_S) - \mathbb E(Y\mid A=0,Z_S) \right].\)

Thm.~\ref{thm:exact-proxy-identification} applies to each proxy split $S$ after replacing $(V_j,W_j,Z_j)$ by $(V_S,W_S,Z_S)$ throughout. Thus, if Assumption~\ref{ass:consistency} and the grouped versions of Assumptions~\ref{ass:joint-latent-exchangeability} and~\ref{ass:outcome-relevant-completeness} hold, representation overlap holds, and $(X,W_S)\perp\!\!\!\perp A\mid Z_S$, then $\theta_S=\tau$.

The balance condition is observable at the population level. Define the population balance-screen set
\begin{align}
	\mathcal B
	\triangleq
	\bigl\{S\in\mathfrak S:
	0<P(A=1\mid Z_S)<1\ \text{a.s.},
	(X,W_S)\perp\!\!\!\perp A\mid Z_S
	\bigr\}.
	\label{eq:population-balance-screen}
\end{align}
Each $S\in\mathcal B$ defines one retained proxy split and one candidate effect $\theta_S$ through Eq.~\eqref{eq:block-candidate-adjustment-value}. We assume that each such $\theta_S$ is well defined in $\mathbb R$. The next result considers the case in which the true effect $\tau$ occurs more often than any other candidate value.

\begin{theorem}[Identification with an unknown valid block; restatement of Thm.~\ref{thm:unknown-valid-block-identification}]
\label{appdetail:thm:unknown-valid-block-identification}
	\normalfont
	Suppose $\mathcal B\neq\varnothing$ and
	the target has a unique proximal plurality (Definition~\ref{def:population-plurality}). Then
	\(\{\tau\} = \underset{c\in\mathbb R}{\operatorname{arg\,max}} \bigl|\{S\in\mathcal B:\theta_S=c\}\bigr|.\)
\end{theorem}
The theorem states that $\tau$ can be identified even when we do not know the exact held-out proxy $W_j$, if it is the unique most frequent candidate value. The following corollary is the strict-majority special case of Thm.~\ref{thm:unknown-valid-block-identification}. It gives both the unique most frequent candidate value and the corresponding median aggregation result.

\begin{corollary}[Target-valued majority and median aggregation]
	\label{cor:target-majority-median}
	\normalfont
	Suppose $\mathcal B\neq\varnothing$ and
	\begin{align}
		\bigl|\{S\in\mathcal B:\theta_S=\tau\}\bigr|>\frac{|\mathcal B|}{2}.
		\label{eq:target-valued-majority}
	\end{align}
	That is, more than half of the retained proxy splits yield the true causal effect. Then
	\(\{\tau\} = \underset{c\in\mathbb R}{\operatorname{arg\,max}} \bigl|\{S\in\mathcal B:\theta_S=c\}\bigr|,\)
	and
	\(\operatorname{median}\bigl((\theta_S)_{S\in\mathcal B}\bigr)=\tau.\)
\end{corollary}

Fig.~\ref{fig:majority-valid-orientations} illustrates the result. The labels $W_{1:3}$ and $W_{4:5}$ denote grouped proxy blocks, not single variables. Under Assumption~\ref{ass:consistency}, if all five singleton proxy splits pass the balance screen and proxy splits $1,2,3$ satisfy the grouped versions of Assumptions~\ref{ass:joint-latent-exchangeability} and~\ref{ass:outcome-relevant-completeness}, then three of the five singleton candidate values equal $\tau$, so their median equals $\tau$.

\FloatBarrier
	\begin{figure}[!t]
	\centering
	\includegraphics[width=0.75\textwidth]{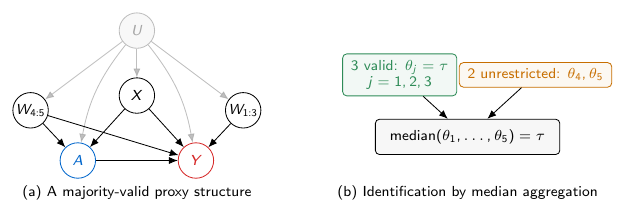}
	\caption{Median aggregation over five singleton proxy splits.}
	\label{fig:majority-valid-orientations}
\end{figure}

\subsection{Approximate Balance and Bias Bound}
\label{sec:approximate-proxy-bound}

Exact balance is an ideal endpoint.
We now consider representations that achieve only approximate, rather than exact, balance.
Fix a measurable encoder $\phi$ and write $T_j\triangleq(X,W_j)$; the notation $Z_j^\phi$ was defined at the beginning of this section.
Define
\(e_{j,\phi}(z)\triangleq P(A=1\mid Z_j^\phi=z), q_{j,\phi}(t,z)\triangleq P(A=1\mid T_j=t,Z_j^\phi=z).\)
Their difference is the treatment information that remains in the held-out variables after conditioning on the representation.

\begin{definition}[Residual treatment discrepancy; restatement of Definition~\ref{def:held-out-proxy-discrepancy}]
\label{appdetail:def:held-out-proxy-discrepancy}
\normalfont
	Define
	\(D_{\mathrm{res},j}^2(\phi) \triangleq \mathbb E\!\left[ \{q_{j,\phi}(T_j,Z_j^\phi)-e_{j,\phi}(Z_j^\phi)\}^2 \right]\) and \(D_{\mathrm{res},j}(\phi) \triangleq\{D_{\mathrm{res},j}^2(\phi)\}^{1/2}.\)
\end{definition}
The discrepancy measures the treatment information in $T_j$ that remains after conditioning on $Z_j^\phi$.
\begin{proposition}[Zero discrepancy; restatement of Prop.~\ref{prop:zero-discrepancy}]
	\label{appdetail:prop:zero-discrepancy}
	\normalfont
	\(D_{\mathrm{res},j}(\phi)=0 \quad\Longleftrightarrow\quad A\perp\!\!\!\perp T_j\mid Z_j^\phi.\)
\end{proposition}

To relate the discrepancy with the causal bias bound, we define
\(P_{\phi,z}\triangleq P(T_j\mid Z_j^\phi=z).\)
For a signed measure $\nu$ with a density with respect to $P_{\phi,z}$, write $\|\nu\|_{\phi,z}\triangleq\|d\nu/dP_{\phi,z}\|_{L_2(P_{\phi,z})}$.
For $P_{Z_j^\phi}$-almost every $z$ and $a\in\{0,1\}$, write
$Q_{a,j,z}^\phi\triangleq P(V_j,U\mid A=a,Z_j^\phi=z)$ and
$\mu_{a,j,z}^\phi(v,u)\triangleq\mathbb E\{Y(a)\mid V_j=v,U=u,Z_j^\phi=z\}$.
Let $\mathcal Q_{j,z}^\phi$ be the collection of probability distributions $Q$ on $(V_j,U)$ such that
$Q\ll P(V_j,U\mid Z_j^\phi=z)$ and
$\mathbb E_Q|\mu_{a,j,z}^\phi(V_j,U)|<\infty$ for both $a\in\{0,1\}$.
We assume the following:
\begin{assumption}[$L_2$ outcome-relevant stability; restatement of Assumption~\ref{ass:uniform-stability}]
	\label{appdetail:ass:uniform-stability}
\normalfont
	For every representation $\phi$ under consideration,
	there is a finite scalar $\Gamma_j(\phi)>0$
	such that, for $P_{Z_j^\phi}$-almost every $z$, every $a\in\{0,1\}$, and all $Q,Q'\in\mathcal Q_{j,z}^{\phi}$,
	\(\left|\mathbb E_Q\{\mu_{a,j,z}^{\phi}\} - \mathbb E_{Q'}\{\mu_{a,j,z}^{\phi}\}\right| \le \Gamma_j(\phi) \,\|\mathcal K_jQ-\mathcal K_jQ'\|_{\phi,z}.\)
\end{assumption}

If $\mathcal K_jQ=\mathcal K_jQ'$, the right-hand side is zero. Assumption~\ref{ass:uniform-stability} therefore implies Assumption~\ref{ass:outcome-relevant-completeness} for $Z_j^\phi$.
\begin{theorem}[Approximate-balance bias; restatement of Thm.~\ref{thm:approximate-proxy-bias-bound}]
\label{appdetail:thm:approximate-proxy-bias-bound}
\normalfont
	Suppose Assumptions~\ref{ass:consistency}, \ref{ass:joint-latent-exchangeability}, and \ref{ass:uniform-stability} hold for $Z_j^\phi$.
	Suppose there is a constant $\eta\in(0,1/2]$ such that $\eta\le e_{j,\phi}(z)\le1-\eta$ for $P_{Z_j^\phi}$-almost every $z$. Then
	\(\left| \mathbb E\!\left[ \mathbb E(Y\mid A=1,Z_j^\phi) - \mathbb E(Y\mid A=0,Z_j^\phi) \right] -\tau \right| \le \frac{\Gamma_j(\phi)}{\eta(1-\eta)} D_{\mathrm{res},j}(\phi).\)
\end{theorem}

The theorem implies the following sensitivity result.
\begin{proposition}[Population sensitivity region]
\label{prop:population-sensitivity-region}
\normalfont
Suppose Assumptions~\ref{ass:consistency}, \ref{ass:joint-latent-exchangeability}, and \ref{ass:uniform-stability} hold for $Z_j^\phi$, and suppose there is a constant $\eta\in(0,1/2]$ such that $\eta\le e_{j,\phi}(z)\le1-\eta$ for $P_{Z_j^\phi}$-almost every $z$. Define
\(\tau_{Z_j^\phi} \triangleq \mathbb E\!\left[ \mathbb E(Y\mid A=1,Z_j^\phi) - \mathbb E(Y\mid A=0,Z_j^\phi) \right].\)
Then
\(\tau\in \left[ \tau_{Z_j^\phi}- \frac{\Gamma_j(\phi)}{\eta(1-\eta)}D_{\mathrm{res},j}(\phi), \tau_{Z_j^\phi}+ \frac{\Gamma_j(\phi)}{\eta(1-\eta)}D_{\mathrm{res},j}(\phi) \right].\)
\end{proposition}

This is a population sensitivity region, not a finite-sample confidence interval. When $D_{\mathrm{res},j}(\phi)=0$, it collapses to the point-identified value in Thm.~\ref{thm:exact-proxy-identification}. The residual discrepancy is a functional of the observed distribution.
The constant $\Gamma_j(\phi)$ is structural and must be justified by scientific knowledge, additional theory, or sensitivity analysis.

\subsection{Blockwise Exclusion within a Single Image}
\label{sec:single-image-decomposition}

When the whole proxy is one image, the blocks are regions of the image, and a region may carry content that the treatment produced: a device placed because the unit was treated, or a change in the tissue that followed treatment. The following proposition separates the error of adjusting for a representation that includes such a region into a part that the representation's coverage of $U$ determines and a part that the treatment-dependent content adds. It explains when excluding a block helps and when it cannot.

Let $Z$ be any representation with $0<P(A=1\mid Z)<1$ almost surely and write $\tau_Z\triangleq\mathbb E[\mathbb E(Y\mid A=1,Z)-\mathbb E(Y\mid A=0,Z)]$ for the adjustment functional, as in Prop.~\ref{prop:population-sensitivity-region}. We call a block \emph{clean} when it carries no treatment-produced content, that is, when it is independent of $A$ given $U$. Throughout this subsection there are no covariates, so every representation is a function of $W$.

\begin{proposition}[Error of adjusting for a representation of a single image]
\label{prop:single-image-decomposition}
\normalfont
Suppose $Y=\tau A+\beta U+\varepsilon$ with $\mathbb E|U|<\infty$ and $\mathbb E[\varepsilon\mid A,U,W]=0$.
\begin{enumerate}
\item[(a)] For every representation $Z$ of $W$ with $0<P(A=1\mid Z)<1$ almost surely,
\(\tau_Z-\tau=\beta\,\mathbb E\bigl[\mathbb E(U\mid A=1,Z)-\mathbb E(U\mid A=0,Z)\bigr].\)
\item[(b)] (Gaussian working model.) Let $Z_0$ be a representation of the clean blocks, $\widehat U\triangleq\mathbb E(U\mid Z_0)$, and $R\triangleq U-\widehat U\sim N(0,s^2)$ independent of $Z_0$; when $U$ is standardised, $s^2=1-\rho^2$ with $\rho^2$ the squared correlation between $U$ and $\widehat U$. Suppose $0<P(A=1\mid Z_0)<1$ and $A\perp\!\!\!\perp R\mid Z_0$. Let a further block carry $C=U-\kappa A+\nu$ for a constant $\kappa\in\mathbb R$, with $\nu\sim N(0,\sigma_\nu^2)$ independent of $(A,U,Z_0)$ and $s^2+\sigma_\nu^2>0$. Then $\tau_{Z_0}=\tau$ and
\(\tau_{(Z_0,C)}-\tau=\beta\,\kappa\,\frac{s^2}{s^2+\sigma_\nu^2}.\)
\item[(c)] If instead the further block carries $C=c(A,\nu)$ with $\nu\perp\!\!\!\perp U\mid(A,Z_0)$, in particular with $\nu$ independent of $(A,U,Z_0)$, and $0<P(A=1\mid Z_0,C)<1$, then $\tau_{(Z_0,C)}=\tau_{Z_0}$.
\end{enumerate}
\end{proposition}

\begin{proof}
(a) By the outcome model, $\mathbb E(Y\mid A=a,Z)=\tau a+\beta\,\mathbb E(U\mid A=a,Z)$, because $\mathbb E(\varepsilon\mid A,Z)=\mathbb E[\mathbb E(\varepsilon\mid A,U,W)\mid A,Z]=0$. Subtracting the two arms and taking the expectation over $Z$ gives the identity.

(b) Write $C-\widehat U+\kappa A=R+\nu$. Given $(Z_0,A)$, the pair $(R,\nu)$ is Gaussian with independent coordinates, $R\sim N(0,s^2)$ and $\nu\sim N(0,\sigma_\nu^2)$, because $R$ is independent of $Z_0$, $A\perp\!\!\!\perp R\mid Z_0$, and $\nu$ is independent of $(A,U,Z_0)$ and hence of $(A,R,Z_0)$. Conditioning one Gaussian coordinate on the sum gives $\mathbb E(R\mid R+\nu=t,Z_0,A)=t\,s^2/(s^2+\sigma_\nu^2)$. Hence
\(\mathbb E(R\mid A,Z_0,C)=\frac{s^2}{s^2+\sigma_\nu^2}\,(C-\widehat U+\kappa A),\)
and, holding $(Z_0,C)$ fixed, the difference between $A=1$ and $A=0$ equals $\kappa s^2/(s^2+\sigma_\nu^2)$. Since $U=\widehat U+R$ and $\widehat U$ is a function of $Z_0$, part (a) gives $\tau_{(Z_0,C)}-\tau=\beta\kappa s^2/(s^2+\sigma_\nu^2)$. For $Z_0$ alone, $A\perp\!\!\!\perp R\mid Z_0$ and $\mathbb E(R\mid Z_0)=0$ give $\mathbb E(R\mid A,Z_0)=0$, so part (a) gives $\tau_{Z_0}=\tau$. Overlap holds for $(Z_0,C)$ because, given $(Z_0,A=a)$, $C\sim N(\widehat U-\kappa a,\,s^2+\sigma_\nu^2)$ has a positive density on $\mathbb R$ for both $a$ when $s^2+\sigma_\nu^2>0$, so Bayes' rule gives $0<P(A=1\mid Z_0,C)<1$ wherever $0<P(A=1\mid Z_0)<1$.

(c) Since $C=c(A,\nu)$ and $\nu\perp\!\!\!\perp U\mid(A,Z_0)$, we have $C\perp\!\!\!\perp U\mid(A,Z_0)$. Hence $\mathbb E(U\mid A,Z_0,C)=\mathbb E(U\mid A,Z_0)$, and part (a) gives the same value for $\tau_{(Z_0,C)}$ and $\tau_{Z_0}$.
\end{proof}

Three consequences guide the chest X-ray analysis of Section~\ref{sec:single-image}. First, when $\rho^2=1$ the representation determines $U$, $R\equiv0$, and no content of the image can bias the adjustment; excluding blocks then has nothing to remove. Second, when $\rho^2<1$, a block whose content depends on $A$ and $U$ jointly (part (b)) biases the adjustment by an amount that grows with the strength $\kappa$ of the treatment's imprint and shrinks as the representation covers more of $U$, whereas a block whose content depends on $A$ alone (part (c)) does not. Third, the term of part (b) is the only one that excluding a block can remove. With residual confounding, that is, when $A$ depends on $R$ given $Z_0$, part (a) gives $\tau_{Z_0}-\tau=\beta\,\mathbb E[\mathbb E(R\mid A=1,Z_0)-\mathbb E(R\mid A=0,Z_0)]$, which is in general nonzero; it depends only on $Z_0$, through the part of $U$ that the clean blocks leave in $R$ and the dependence of $A$ on that part, so excluding a further block cannot change it and only a change of the representation of the clean blocks can. The screen of Section~\ref{sec:estimation} finds the block of part (b) or (c) because holding it out leaves residual treatment information: $P(A=1\mid C,Z_0)\neq P(A=1\mid Z_0)$ with positive probability whenever $C$ and $A$ are dependent given $Z_0$, which holds in part (b) when $\kappa\neq0$ and in part (c) when $c(1,\nu)$ and $c(0,\nu)$ differ in distribution, so its held-out discrepancy is positive, whereas holding out a clean block changes the treatment probability only through that block's increment of information about $U$.

\section{Representation Learning and AIPW Estimation}
\label{app:learning-details}

The identification results concern one fixed representation. We now select a representation from data and estimate the adjustment functional that it induces. The representation-learning and AIPW analysis below uses one prespecified held-out block $j$ throughout.

\subsection{Representation Learning}
\label{sec:representation-learning}

Recall that one observed record is $O_i=(X_i,W_i,A_i,Y_i)$. Partition the observed sample into three mutually disjoint subsets,
\(\mathcal I_{\mathrm D}, \qquad \mathcal I_{\mathrm N}, \qquad \mathcal I_{\mathrm E},\)
with respective sizes $n_{\mathrm D}$, $n_{\mathrm N}$, and $n_{\mathrm E}$. The discrepancy-learning sample $\mathcal I_{\mathrm D}$ selects a representation. The nuisance-learning sample $\mathcal I_{\mathrm N}$ estimates the outcome regressions and propensity used by AIPW. The evaluation sample $\mathcal I_{\mathrm E}$ calculates the final estimate. Conditional on the preceding samples, the objects evaluated on the next sample are fixed.

 Let $\mathcal F_{j,n_{\mathrm D}}\subseteq\Phi_j$ be a prespecified, possibly infinite class of deterministic encoders. For $\phi\in\mathcal F_{j,n_{\mathrm D}}$, write $Z_j^\phi=\phi(V_j)$ and retain $T_j=(X,W_j)$ from Section~\ref{sec:approximate-proxy-bound}.

Assumption~\ref{ass:library-overlap} is imposed on $\mathcal F_{j,n_{\mathrm D}}$.
For any prespecified nonempty proper subset $S\subset\{1,\ldots,J\}$, the results below also apply with $(V_j,W_j,Z_j^\phi)$ replaced by $(V_S,W_S,\phi(V_S))$ and with the same assumptions imposed on this grouped split.

The Brier-risk definition below estimates residual discrepancy by comparing squared-loss prediction of $A$ from $Z_j^\phi$ alone with prediction from $(T_j,Z_j^\phi)$ \citep{brier1950verification}.

\begin{definition}[Empirical Brier discrepancy; general form of Definition~\ref{def:empirical-held-out-proxy-discrepancy}]
\label{appdetail:def:empirical-held-out-proxy-discrepancy}
\normalfont
 Let $\mathcal H_{0,n_{\mathrm D}}$ and $\mathcal H_{1,n_{\mathrm D}}$ be prespecified classes of measurable $[0,1]$-valued predictors with inputs $Z_j^\phi$ and $(T_j,Z_j^\phi)$, respectively, fixed independently of $\mathcal I_{\mathrm D}$. Assume $\mathcal H_{1,n_{\mathrm D}}$ contains each member of $\mathcal H_{0,n_{\mathrm D}}$ lifted to ignore $T_j$. 
Define \(R_{0,\phi}(h)\triangleq\mathbb E[\{A-h(Z_j^\phi)\}^2]\) and \(\widehat R_{0,\phi}(h)\triangleq n_{\mathrm D}^{-1}\sum_{i\in\mathcal I_{\mathrm D}}\{A_i-h(Z_{j,i}^\phi)\}^2\). Likewise, define \(R_{1,\phi}(h)\triangleq\mathbb E[\{A-h(T_j,Z_j^\phi)\}^2]\) and \(\widehat R_{1,\phi}(h)\triangleq n_{\mathrm D}^{-1}\sum_{i\in\mathcal I_{\mathrm D}}\{A_i-h(T_{j,i},Z_{j,i}^\phi)\}^2\),
where $Z_{j,i}^\phi=\phi(V_{j,i})$. For each $\phi$, let $\widehat h_{b,\phi}\in\mathcal H_{b,n_{\mathrm D}}$ be an approximate empirical minimizer satisfying, for some $\varepsilon_{b,\phi}\ge0$,
\begin{align}
	\widehat R_{b,\phi}(\widehat h_{b,\phi})
	\le
	\inf_{h\in\mathcal H_{b,n_{\mathrm D}}}\widehat R_{b,\phi}(h)
	+
	\varepsilon_{b,\phi},
	\qquad b\in\{0,1\}.
	\label{eq:approximate-critic-fits}
\end{align}
Define $\varepsilon_{\mathrm{crit},j}\triangleq\sup_{\phi\in\mathcal F_{j,n_{\mathrm D}}}(\varepsilon_{0,\phi}+\varepsilon_{1,\phi})$ and
\(\widetilde D_{j,n_{\mathrm D}}^2(\phi) \triangleq \max\left\{( \widehat R_{0,\phi}(\widehat h_{0,\phi}) - \widehat R_{1,\phi}(\widehat h_{1,\phi})), 0 \right\}\).
\end{definition}

Each encoder $\phi$ and critic $h$ give a squared-loss function of one observation: $\{A-h(\phi(V_j))\}^2$ or $\{A-h(T_j,\phi(V_j))\}^2$. Varying both $\phi$ and $h$ over their prespecified classes gives the loss classes $\mathcal L_{0,j,n_{\mathrm D}}$ and $\mathcal L_{1,j,n_{\mathrm D}}$, respectively. Let $\mathfrak R_{n_{\mathrm D}}(\mathcal L)$ denote the expected Rademacher complexity of a loss class $\mathcal L$. For $\delta_{\mathrm D}\in(0,1)$, define

\(r_{j,n_{\mathrm D}}(\delta_{\mathrm D}) \triangleq 2\{\mathfrak R_{n_{\mathrm D}}(\mathcal L_{0,j,n_{\mathrm D}})+\mathfrak R_{n_{\mathrm D}}(\mathcal L_{1,j,n_{\mathrm D}})\} + 2\sqrt{\frac{\log(4/\delta_{\mathrm D})}{2n_{\mathrm D}}}.\)

This radius controls sampling fluctuation uniformly over the encoder--critic choices. The complexity terms account for the flexibility of the two loss classes, and the last term sets the failure-probability level $\delta_{\mathrm D}$.

Throughout, the loss classes are taken pointwise measurable and separable, so that the suprema below are measurable, and the selected encoders and critics are measurable functions of the training observations and the algorithmic seed; every statement otherwise holds with outer probability.

\begin{proposition}[Uniform residual-discrepancy deviation]
\label{appdetail:prop:uniform-discrepancy-deviation}
\normalfont
For each $\phi\in\mathcal F_{j,n_{\mathrm D}}$, define
\begin{align}
\begin{aligned}
	a_{0,n_{\mathrm D}}(\phi)
	&\triangleq
	\inf_{h\in\mathcal H_{0,n_{\mathrm D}}}R_{0,\phi}(h)
	-
	\mathbb E[\{A-e_{j,\phi}(Z_j^\phi)\}^2],\\
	a_{1,n_{\mathrm D}}(\phi)
	&\triangleq
	\inf_{h\in\mathcal H_{1,n_{\mathrm D}}}R_{1,\phi}(h)
	-
	\mathbb E[\{A-q_{j,\phi}(T_j,Z_j^\phi)\}^2],\\
	a_{j,n_{\mathrm D}}
	&\triangleq
	\sup_{\phi\in\mathcal F_{j,n_{\mathrm D}}}
	\{a_{0,n_{\mathrm D}}(\phi)+a_{1,n_{\mathrm D}}(\phi)\}.
\end{aligned}
\label{eq:critic-approximation-error}
\end{align}

With probability at least $1-\delta_{\mathrm D}$,
\begin{align}
\sup_{\phi\in\mathcal F_{j,n_{\mathrm D}}} |\widetilde D_{j,n_{\mathrm D}}^2(\phi)-D_{\mathrm{res},j}^2(\phi)| &\le r_{j,n_{\mathrm D}}(\delta_{\mathrm D})+a_{j,n_{\mathrm D}}+\varepsilon_{\mathrm{crit},j}.
\label{eq:uniform-discrepancy-deviation}
\end{align}
\end{proposition}
At any fixed confidence level $1-\delta_{\mathrm D}$, this error bound is $O(n_{\mathrm D}^{-1/2})$ if both Rademacher complexities, $a_{j,n_{\mathrm D}}$, and $\varepsilon_{\mathrm{crit},j}$ are each $O(n_{\mathrm D}^{-1/2})$.

\begin{definition}[Approximate empirical minimizer]
\label{def:approximate-empirical-minimizer}
\normalfont
 For a prescribed tolerance $\varepsilon_{\mathrm{enc},j}\ge0$, an approximate empirical minimizer is an encoder $\widehat\phi_j\in\mathcal F_{j,n_{\mathrm D}}$ selected measurably from the training observations and algorithmic seed such that
\begin{align}
\widetilde D_{j,n_{\mathrm D}}^2(\widehat\phi_j)
\le
\inf_{\phi\in\mathcal F_{j,n_{\mathrm D}}}\widetilde D_{j,n_{\mathrm D}}^2(\phi)
+
\varepsilon_{\mathrm{enc},j}.
\label{eq:approximate-erm}
\end{align}
An exact empirical minimizer has $\varepsilon_{\mathrm{enc},j}=0$.

\end{definition}

\begin{proposition}[Residual-discrepancy oracle inequality]
\label{appdetail:thm:representation-oracle}
\normalfont
 Let $\delta_{\mathrm D}\in(0,1)$. Suppose the prespecified encoder and $[0,1]$-valued critic classes, approximate critic fits in Eq.~\eqref{eq:approximate-critic-fits} of Definition~\ref{appdetail:def:empirical-held-out-proxy-discrepancy} hold, and $\widehat\phi_j$ satisfies Eq.~\eqref{eq:approximate-erm}. Then, with probability at least $1-\delta_{\mathrm D}$,
\begin{align}
D_{\mathrm{res},j}^2(\widehat\phi_j) &\le \inf_{\phi\in\mathcal F_{j,n_{\mathrm D}}}D_{\mathrm{res},j}^2(\phi) +2\{r_{j,n_{\mathrm D}}(\delta_{\mathrm D})+a_{j,n_{\mathrm D}}+\varepsilon_{\mathrm{crit},j}\} +\varepsilon_{\mathrm{enc},j}.
\label{eq:representation-oracle-discrepancy}
\end{align}
\end{proposition}

Condition on the training observations $\{O_i:i\in\mathcal I_{\mathrm D}\}$ and any algorithmic seed used in training, which fixes $\widehat\phi_j$. Here and below, population expectations involving this learned encoder use a fresh observation $O=(X,W,A,Y)$ independent of the training observations and seed, with this conditioning understood.
The learned representation and its population adjustment functional are

\begin{align}
Z_j^{\widehat\phi_j}
&\triangleq\widehat\phi_j(V_j),\qquad\widehat Z_j\triangleq Z_j^{\widehat\phi_j},\nonumber\\
\tau_{\widehat Z_j}
&\triangleq
\mathbb E\!\left[
\mathbb E(Y\mid A=1,\widehat Z_j)
-
\mathbb E(Y\mid A=0,\widehat Z_j)
\right].
\label{eq:learned-adjustment-functional}
\end{align}

Thus $\tau_{\widehat Z_j}=\tau_{Z_j^\phi}\big|_{\phi=\widehat\phi_j}$, where $\tau_{Z_j^\phi}$ is the population adjustment functional defined in Thm.~\ref{thm:approximate-proxy-bias-bound}.
The adjustment functional is used whenever its two conditional means are absolutely integrable under the distribution of $\widehat Z_j$.

\begin{theorem}[Learned-representation error with an explicit radius]
\label{appdetail:cor:representation-causal}
\normalfont

Let $\delta_{\mathrm D}\in(0,1)$. Suppose the prespecified encoder and $[0,1]$-valued critic classes, approximate critic fits in Eq.~\eqref{eq:approximate-critic-fits} of Definition~\ref{appdetail:def:empirical-held-out-proxy-discrepancy} hold, and $\widehat\phi_j$ satisfies Eq.~\eqref{eq:approximate-erm}. Suppose Assumptions~\ref{ass:consistency}, \ref{ass:joint-latent-exchangeability}, and \ref{ass:uniform-stability} hold for every $Z_j^\phi$ with $\phi\in\mathcal F_{j,n_{\mathrm D}}$, and there is a constant $\eta\in(0,1/2]$ such that $\eta\le e_{j,\phi}(Z_j^\phi)\le1-\eta$ almost surely for every such $\phi$. Then, with probability at least $1-\delta_{\mathrm D}$,
\(|\tau_{\widehat Z_j}-\tau| \le \frac{\Gamma_j(\widehat\phi_j)}{\eta(1-\eta)} \left[ \inf_{\phi\in\mathcal F_{j,n_{\mathrm D}}}D_{\mathrm{res},j}^2(\phi) +2\{r_{j,n_{\mathrm D}}(\delta_{\mathrm D})+a_{j,n_{\mathrm D}}+\varepsilon_{\mathrm{crit},j}\} +\varepsilon_{\mathrm{enc},j} \right]^{1/2}.\)
\end{theorem}
At fixed confidence and fixed overlap constant $\eta$, the right-hand side is $O_p(n_{\mathrm D}^{-1/4})$ if $\inf_{\phi\in\mathcal F_{j,n_{\mathrm D}}}D_{\mathrm{res},j}^2(\phi)=0$, each of $r_{j,n_{\mathrm D}}(\delta_{\mathrm D})$, $a_{j,n_{\mathrm D}}$, $\varepsilon_{\mathrm{crit},j}$, and $\varepsilon_{\mathrm{enc},j}$ is $O(n_{\mathrm D}^{-1/2})$, and $\Gamma_j(\widehat\phi_j)=O_p(1)$. This is a conditional upper-bound rate; it is not a universal or minimax rate.

\subsection{AIPW Estimation}
\label{sec:aipw-estimation}
This section estimates the adjustment functional $\tau_{\widehat Z_j}$ of the learned representation with an augmented inverse probability weighted (AIPW) estimator \citep{robins1994estimation,tsiatis2006semiparametric} computed on the evaluation sample.
 Define the conditional outcome means and representation propensity
\begin{align}
\begin{aligned}
	m_{a,j}(z) &\triangleq \mathbb E(Y\mid A=a,\widehat Z_j=z),\\
	e_j(z) &\triangleq P(A=1\mid\widehat Z_j=z).
\end{aligned}
\label{eq:aipw-nuisances}
\end{align}
For the observed adjustment functional $\tau_{\widehat Z_j}$, the following is the AIPW score:
\(\mathtt{IF}_j(O) \triangleq m_{1,j}(\widehat Z_j)-m_{0,j}(\widehat Z_j) + \frac{A\{Y-m_{1,j}(\widehat Z_j)\}}{e_j(\widehat Z_j)} - \frac{(1-A)\{Y-m_{0,j}(\widehat Z_j)\}}{1-e_j(\widehat Z_j)} - \tau_{\widehat Z_j}.\)
When this score is square-integrable, it is the \textit{efficient influence function} in the nonparametric model for $(\widehat Z_j,A,Y)$. Bounded outcomes together with uniform representation overlap provide one sufficient condition for square integrability. Use the nuisance-learning sample, with $\widehat Z_{j,i}=\widehat\phi_j(V_{j,i})$, to estimate the nuisance functions. Denote the estimates by $\widehat m_{0,j}$, $\widehat m_{1,j}$, and $\widehat e_j$.

\begin{definition}[Honest AIPW estimator]
\label{appdetail:def:honest-aipw}
\normalfont
For $i\in\mathcal I_{\mathrm E}$, define $\widehat Z_{j,i}=\widehat\phi_j(V_{j,i})$ and
	\(\mathtt{UIF}_j(O_i) \triangleq{} \widehat m_{1,j}(\widehat Z_{j,i}) - \widehat m_{0,j}(\widehat Z_{j,i}) + \frac{A_i}{\widehat e_j(\widehat Z_{j,i})} \{Y_i-\widehat m_{1,j}(\widehat Z_{j,i})\} - \frac{1-A_i}{1-\widehat e_j(\widehat Z_{j,i})} \{Y_i-\widehat m_{0,j}(\widehat Z_{j,i})\}.\)
The score $\mathtt{UIF}_j$ is the uncentered plug-in version of $\mathtt{IF}_j$: it replaces the nuisance functions by their estimates and omits the centering term $\tau_{\widehat Z_j}$.
The honest AIPW estimator is
\(\widehat\tau_j^{\mathrm{AIPW}} \triangleq \frac{1}{n_{\mathrm E}} \sum_{i\in\mathcal I_{\mathrm E}}\mathtt{UIF}_j(O_i).\)
\end{definition}

The plug-in score has the doubly robust remainder structure of augmented inverse probability weighting \citep{chernozhukov2017double,kennedy2022semiparametric}: its conditional bias is a product of the propensity error and the outcome-regression errors, as the next proposition states exactly. For $a\in\{0,1\}$, write
\(r_{a,j}(z) \triangleq \frac{\widehat m_{a,j}(z)-m_{a,j}(z)} {a\widehat e_j(z)+(1-a)\{1-\widehat e_j(z)\}}.\)
\begin{proposition}[Exact signed nuisance remainder]
\label{appdetail:prop:aipw-remainder}
\normalfont
Conditional on $\mathcal I_{\mathrm D}$ and $\mathcal I_{\mathrm N}$,
\begin{align}
\mathbb E\{\mathtt{UIF}_j(O)\}-\tau_{\widehat Z_j} &= \mathbb E\Big[ \{\widehat e_j(\widehat Z_j)-e_j(\widehat Z_j)\} \{r_{1,j}(\widehat Z_j)+r_{0,j}(\widehat Z_j)\} \Big].
\label{eq:aipw-remainder}
\end{align}
\end{proposition}

The remainder is zero if the propensity estimator is correct. It is also zero if both conditional outcome estimators are correct.

Clip the propensity estimate so that $\eta\le\widehat e_j(z)\le1-\eta$. Conditional on the two learning samples, that is, on their observations and the seeds used to fit the encoder and the nuisance functions, define for a function $g$ of $\widehat Z_j$ the norm $\|g\|_{2,\widehat Z_j}\triangleq[\mathbb E\{g(\widehat Z_j)^2\mid\mathcal I_{\mathrm D},\mathcal I_{\mathrm N}\}]^{1/2}$ and the realized nuisance errors
\(q_{e,j}\triangleq\|\widehat e_j-e_j\|_{2,\widehat Z_j}, \qquad q_{a,j}\triangleq\|\widehat m_{a,j}-m_{a,j}\|_{2,\widehat Z_j}, \qquad a\in\{0,1\}.\)
Let $\sigma_j^2\triangleq\mathrm{Var}\{\mathtt{UIF}_j(O)\mid\mathcal I_{\mathrm D},\mathcal I_{\mathrm N}\}$ denote the conditional variance of the plug-in score; it is finite whenever $\mathbb E(Y^2)<\infty$ and the outcome estimates are square-integrable. For $\delta_{\mathrm E}\in(0,1)$, define
\begin{align}
\varepsilon_{\mathrm{AIPW},j}
\triangleq
\frac{\sigma_j}{\sqrt{n_{\mathrm E}\delta_{\mathrm E}}}
+
\frac{q_{e,j}(q_{0,j}+q_{1,j})}{\eta}.
\label{eq:aipw-error-radius}
\end{align}
In the notation of Thm.~\ref{thm:conditional-aipw-error}, $\|\cdot\|_{2,\widehat Z_j}=\|\cdot\|_{2,j}$ and $q_{e,j}(q_{0,j}+q_{1,j})/\eta=\mathtt{err}_j$, so $\varepsilon_{\mathrm{AIPW},j}=\sigma_j/\sqrt{n_{\mathrm E}\delta_{\mathrm E}}+\mathtt{err}_j$. The last term of the second bound of Thm.~\ref{appdetail:thm:conditional-aipw-error} below is $\mathtt{bias}_j$.
The next theorem is the final guarantee of this section: the causal error of the honest AIPW estimator is at most the estimation radius $\varepsilon_{\mathrm{AIPW},j}$ plus the stability-weighted residual discrepancy of the learned representation.

\begin{theorem}[Honest AIPW error; general form of Thm.~\ref{thm:conditional-aipw-error}]
\label{appdetail:thm:conditional-aipw-error}
\normalfont

Suppose $\sigma_j<\infty$. Conditional on the two learning samples, with probability at least $1-\delta_{\mathrm E}$ over $\mathcal I_{\mathrm E}$,
\begin{align}
|\widehat\tau_j^{\mathrm{AIPW}}-\tau_{\widehat Z_j}|
\le
\varepsilon_{\mathrm{AIPW},j}.
\label{eq:conditional-aipw-error}
\end{align}
If, in addition, Assumptions~\ref{ass:consistency}, \ref{ass:joint-latent-exchangeability}, and \ref{ass:uniform-stability} hold for $Z_j^{\widehat\phi_j}$ and $\eta\le e_j(\widehat Z_j)\le1-\eta$ almost surely, then on the same event
\(|\widehat\tau_j^{\mathrm{AIPW}}-\tau| \le \varepsilon_{\mathrm{AIPW},j} + \frac{\Gamma_j(\widehat\phi_j)}{\eta(1-\eta)} D_{\mathrm{res},j}(\widehat\phi_j).\)

\end{theorem}

\begin{corollary}[Consistency]
\label{cor:learned-aipw-consistency}
\normalfont

Consider a sequence of encoder and critic classes, sample splits, nuisance estimators, optimization tolerances, and selected encoders $\widehat\phi_{j,n}$ indexed by the total sample size $n$, with a common overlap constant $\eta\in(0,1/2]$, such that Thm.~\ref{thm:conditional-aipw-error} applies at every $n$. Suppose the following three conditions hold.
\begin{enumerate}[label=(C\arabic*),leftmargin=*]
\item \textit{Vanishing evaluation error.} $\delta_{\mathrm E,n}\to0$ and $\sigma_{j,n}^2/(n_{\mathrm E}\delta_{\mathrm E,n})\to0$. The evaluation sample grows faster than the confidence level shrinks.
\item \textit{Vanishing nuisance product.} $q_{e,j,n}(q_{0,j,n}+q_{1,j,n})\xrightarrow{p}0$. The propensity and outcome-regression errors need not vanish individually; their product vanishes, for example, when each is $o_p(n^{-1/4})$.
\item \textit{Vanishing learned imbalance.} $\Gamma_j(\widehat\phi_{j,n})D_{\mathrm{res},j}(\widehat\phi_{j,n})\xrightarrow{p}0$. By Prop.~\ref{appdetail:thm:representation-oracle}, this holds when $\inf_{\phi\in\mathcal F_{j,n_{\mathrm D}}}D_{\mathrm{res},j}^2(\phi)$, the two Rademacher complexities, $a_{j,n_{\mathrm D}}$, $\varepsilon_{\mathrm{crit},j,n}$, $\varepsilon_{\mathrm{enc},j,n}$, $\delta_{\mathrm D,n}$, and $\log(1/\delta_{\mathrm D,n})/n_{\mathrm D}$ all converge to zero and $\Gamma_j(\widehat\phi_{j,n})=O_p(1)$.
\end{enumerate}
Then $\widehat\tau_{j,n}^{\mathrm{AIPW}}\xrightarrow{p}\tau$.

\end{corollary}

Eq.~\eqref{eq:end-to-end-sensitivity} separates the final error into an AIPW estimation term and a hidden-confounding sensitivity term. The selected stability constant $\Gamma_j(\widehat\phi_j)$, critic approximation error $a_{j,n_{\mathrm D}}$, optimization tolerances, and nuisance errors are not generally observed. The result becomes an operational numerical certificate after these quantities are bounded by theory, computation, external information, or a declared sensitivity analysis.

\subsection{Searching Across Proxy Splits}
\label{sec:random-orientation-search}
Exhaustive evaluation of all proxy splits in $\mathfrak S$, defined in Eq.~\eqref{eq:all-held-out-orientations}, may be computationally expensive. The number of splits, $|\mathfrak S|=2^J-2$, grows exponentially in the number of blocks $J$, and each split requires its own representation and nuisance fits.

 Algorithm~\ref{appdetail:alg:probe-search} therefore samples $m$ proxy splits instead of evaluating all of them, and then works in four stages. The sampling stage draws the splits $S$ (line~\ref{ln:sample}). The screening stage keeps the splits that pass the empirical balance screen (lines~\ref{ln:screen-start}--\ref{ln:screen-end}). The estimation stage computes one honest AIPW estimate per retained split (lines~\ref{ln:estimate-start}--\ref{ln:estimate-end}). The aggregation stage links estimates within distance $2\rho$ and returns the median of the largest component (lines~\ref{ln:link}--\ref{ln:return}). The theorem below gives conditions under which those empirical components recover a fixed oracle partition of the population candidates.

 Three questions decide whether the output is accurate: whether enough screened splits are sampled, whether the designated target cell remains the unique plurality, and whether the estimated graph recovers the fixed oracle cells.

\begin{algorithm}[t]
\caption{PROBE split search, detailed form}
\label{appdetail:alg:probe-search}
	\KwIn{a block partition $W=(W_1,\ldots,W_J)$; observations split into $\mathcal I_{\mathrm D}$, $\mathcal I_{\mathrm N}$, $\mathcal I_{\mathrm E}$; a budget $m$; a threshold $t\ge0$; an overlap constant $\eta$; a linking radius $\rho>0$.}
	\KwOut{an estimate $\widehat\tau$ of $\tau$, or failure.}
	Draw $m$ proxy splits uniformly without replacement from $\mathfrak S$\label{ln:sample}\;
	\ForEach{sampled split $S$\label{ln:screen-start}}{
		learn $\widehat\phi_S$ on $\mathcal I_{\mathrm D}$ by minimizing $\widetilde D_{S,n_{\mathrm D}}^2$\;
		 fit an unclipped propensity $\widehat e_S^{\mathrm{raw}}$ on $\mathcal I_{\mathrm N}$; retain $S$ if $\widetilde D_{S,n_{\mathrm D}}^2(\widehat\phi_S)\le t$ and $\widehat e_S^{\mathrm{raw}}$ passes a prespecified empirical overlap check at level $\eta$.\;
	}
	Let $\mathcal R_m$ be the set of retained splits\label{ln:screen-end}\;
	\ForEach{$S\in\mathcal R_m$\label{ln:estimate-start}}{
		 fit outcome regressions on $\mathcal I_{\mathrm N}$, set $\widehat e_S=\operatorname{clip}(\widehat e_S^{\mathrm{raw}},\eta,1-\eta)$, and compute the honest AIPW estimate $\widehat\theta_S$ from Eq.~\eqref{eq:honest-aipw-score} on $\mathcal I_{\mathrm E}$.\label{ln:estimate-end}\;
	}
	Link $S$ and $S'$ whenever $|\widehat\theta_S-\widehat\theta_{S'}|\le2\rho$, and form the connected components\label{ln:link}\;
	\Return the median of $\{\widehat\theta_S\}$ over the unique largest component; declare failure if several components share the largest size\label{ln:return}\;
\end{algorithm}
\FloatBarrier

 For the analysis only, preassign a learning seed $\omega_S$ to every $S\in\mathfrak S$ before the split draw, let $\mathcal G\triangleq\sigma(\mathcal I_{\mathrm D},\mathcal I_{\mathrm N},\{\omega_S:S\in\mathfrak S\})$, and condition on $\mathcal G$ throughout this subsection. Counterfactually applying the fixed learning-and-screening rule to every $S\in\mathfrak S$ defines $\widehat{\mathcal B}$, although Algorithm~\ref{alg:probe-search} evaluates only sampled splits. Thus, for every $S\in\mathfrak S$, the learned encoder $\widehat\phi_S$, the nuisance estimates, the screen outcome, and the set $\widehat{\mathcal B}$ are fixed. For $S\in\widehat{\mathcal B}$, let $\widehat Z_S\triangleq\widehat\phi_S(V_S)$ and let $\theta_S\triangleq\tau_{\widehat Z_S}$ be its population adjustment value.

Let $\Pi=\{C_0,\ldots,C_L\}$ be a partition of $\widehat{\mathcal B}$ that is fixed given $\mathcal G$ and hence chosen before the random split draw and $\mathcal I_{\mathrm E}$. Because the population values $\theta_S$ are generally unknown, $\Pi$ is an oracle device for analysis rather than an object computed by Algorithm~\ref{alg:probe-search}.

Let $T\triangleq|\mathfrak S|=2^J-2$ and $N\triangleq|\widehat{\mathcal B}|\ge1$. Fix $m\in\{1,\ldots,T\}$. Line~\ref{ln:sample} of Algorithm~\ref{appdetail:alg:probe-search} draws $m$ proxy splits uniformly without replacement from $\mathfrak S$, independently of $\mathcal G$. Let $\mathcal R_m$ be the sampled splits in $\widehat{\mathcal B}$, let $R\triangleq|\mathcal R_m|$, and let $R_\ell\triangleq|\mathcal R_m\cap C_\ell|$. Then
\begin{align}
	R\sim\operatorname{Hypergeometric}(T,N,m),
\label{eq:retained-orientation-count}
\end{align}
since exactly $N$ of the $T$ candidates belong to $\widehat{\mathcal B}$ and the $m$ draws are made without replacement.

 For the oracle cells, define
\begin{align}\label{eq:population-orientation-count}
	n_\ell\triangleq|C_\ell|,\qquad \ell=0,\ldots,L.
\end{align}
Adopt the convention that a maximum over an empty collection equals zero and define
\begin{align}\label{eq:population-plurality-gap}
	\Delta_\Pi\triangleq\frac{n_0}{N}-\max_{1\le\ell\le L}\frac{n_\ell}{N}.
\end{align}
Thus $\Delta_\Pi>0$ means that the target cell $C_0$ is larger than every competing oracle cell. If $L=0$, the maximum in Eq.~\eqref{eq:population-plurality-gap} and the retained-cell maximum in Lemma~\ref{lem:random-subsampling-plurality} are both defined as zero.

Two quantities of the distribution in Eq.~\eqref{eq:retained-orientation-count} enter the guarantee:
\begin{align}
\begin{aligned}
P(R=0)&=\binom{T-N}{m}\Big/\binom{T}{m},\\
\mathbb E\{e^{-R\Delta_\Pi^2/2}\}&=\sum_{r=0}^{\min(m,N)}\binom Nr\binom{T-N}{m-r}e^{-r\Delta_\Pi^2/2}\Big/\binom Tm.
\end{aligned}
\label{eq:sampling-failure-terms}
\end{align}
Eq.~\eqref{eq:sampling-failure-terms} gives the empty-draw probability and the plurality-reversal term; both decrease as the budget $m$ grows.

\begin{lemma}[Cluster plurality under random subsampling]
\label{lem:random-subsampling-plurality}
\normalfont
Conditional on $\mathcal G$, $R\sim\operatorname{Hypergeometric}(T,N,m)$. If $\Delta_\Pi>0$, then with probability at least $1-P(R=0)-L\,\mathbb E\{e^{-R\Delta_\Pi^2/2}\}$, $R\ge1$ and $R_0>\max_{1\le\ell\le L}R_\ell$, with an empty maximum equal to zero.
\end{lemma}

Fix $\delta\in(0,1)$. We call $\Pi$ $(b,\varepsilon,\rho,\delta)$-recoverable if $b,\varepsilon\ge0$, $b+\varepsilon\le\rho$, $|\theta_S-\tau|\le b$ for every $S\in C_0$, within-cell distances are at most $2(\rho-\varepsilon)$, between-cell distances exceed $2(\rho+\varepsilon)$, and, conditional on $\mathcal G$, $|\widehat\theta_S-\theta_S|\le\varepsilon$ for every $S\in\widehat{\mathcal B}$ with probability at least $1-\delta$ over $\mathcal I_{\mathrm E}$.

The following theorem bounds the error of the output of Algorithm~\ref{alg:probe-search}. Its probability is over the evaluation sample and the random draw of splits, conditional on the learning samples.

\begin{theorem}[Finite-sample PROBE cluster recovery; general form of Thm.~\ref{thm:probe-search-error}]
\label{appdetail:thm:probe-search-error}
\normalfont
Conditional on $\mathcal G$, suppose $\Delta_\Pi>0$ and $\Pi$ is $(b,\varepsilon,\rho,\delta)$-recoverable. Then, over the independent split draw and $\mathcal I_{\mathrm E}$, the output $\widehat\tau$ of Algorithm~\ref{alg:probe-search} run with linking radius $\rho$ satisfies
\(P\bigl(|\widehat\tau-\tau|\le b+\varepsilon\bigr) \ \ge\ 1-\delta-P(R=0)-L\,\mathbb E\bigl\{e^{-R\Delta_\Pi^2/2}\bigr\},\)
where $R\sim\operatorname{Hypergeometric}(T,N,m)$ as in Eq.~\eqref{eq:retained-orientation-count}, the two sampling terms are given in Eq.~\eqref{eq:sampling-failure-terms}, and a declared failure counts as $|\widehat\tau-\tau|>b+\varepsilon$.

\end{theorem}

The subsampling lemma controls whether $C_0$ remains the unique largest retained cell. The within-cell and between-cell conditions then make the graph components in line~\ref{ln:link} of Algorithm~\ref{appdetail:alg:probe-search} equal the retained restrictions of the oracle cells. The target-band and evaluation-error conditions place every estimate in the selected target component within $b+\varepsilon$ of $\tau$. Eq.~\eqref{eq:conditional-aipw-error} at $\delta_{\mathrm E}=\delta/N$ and a union bound over $\widehat{\mathcal B}$ supply $\varepsilon=\max_{S\in\widehat{\mathcal B}}\varepsilon_{\mathrm{AIPW},S}$, and Eq.~\eqref{eq:uniform-approximate-proxy-bias-bound} supplies $b=\max_{S\in C_0}\Gamma_S(\widehat\phi_S)D_{\mathrm{res},S}(\widehat\phi_S)/\{\eta(1-\eta)\}$ when every $S\in C_0$ satisfies Assumptions~\ref{ass:consistency}, \ref{ass:joint-latent-exchangeability}, and~\ref{ass:uniform-stability}; this gives Thm.~\ref{thm:probe-search-error} of the main text, with $\Pi=\{C_0,\ldots,C_L\}$ and $\Delta=\Delta_\Pi$. The partition and its gap remain oracle quantities and are not certified by the empirical graph. When $b=0$ and $C_0=\{S:\theta_S=\tau\}$, the target-band premise uses the exact population equality notion of Section~\ref{sec:unknown-blocks}; the finite-sample theorem additionally involves empirical screening, random subsampling, and estimation error.

The evaluation failure $\delta$ is paid with the uniform radius $\varepsilon$, while the two sampling terms are paid with the budget $m$. The fraction $p\triangleq N/T$, the oracle margin $\Delta_\Pi$, and the number $L$ of competing cells are fixed after conditioning. The next corollary converts a sampling-failure tolerance $\alpha$ into a budget.

\begin{corollary}[Budget for a target failure level]
\label{cor:budget-rule}
\normalfont

Condition on $\mathcal I_{\mathrm D}$, $\mathcal I_{\mathrm N}$, and the learning seeds. Let $T\triangleq|\mathfrak S|$, let $\widehat{\mathcal B}$ be a fixed screened set of size $N\ge1$, and define $\theta_S\triangleq\tau_{\widehat\phi_S(V_S)}$. Before the split draw and $\mathcal I_{\mathrm E}$, fix a partition $\Pi=\{C_0,\ldots,C_L\}$ of $\widehat{\mathcal B}$ that is fixed given $\mathcal G$. Set $n_\ell\triangleq|C_\ell|$, $\Delta_\Pi\triangleq n_0/N-\max_{1\le\ell\le L}n_\ell/N>0$, and $p\triangleq N/T$, with an empty maximum equal to zero. Draw $m\in\{1,\ldots,T\}$ splits uniformly without replacement, let $R$ be the number drawn from $\widehat{\mathcal B}$, and fix $\alpha\in(0,1)$. Then
\begin{align}\label{eq:sampling-failure-closed-form}
	P(R=0)+L\,\mathbb E\bigl\{e^{-R\Delta_\Pi^2/2}\bigr\}
	\ \le\
	(L+1)\exp\bigl\{-m\,p\,(1-e^{-\Delta_\Pi^2/2})\bigr\}.
\end{align}
Suppose additionally that $b,\varepsilon\ge0$, $b+\varepsilon\le\rho$, $|\theta_S-\tau|\le b$ on $C_0$, within-cell distances are at most $2(\rho-\varepsilon)$, between-cell distances exceed $2(\rho+\varepsilon)$, and $|\widehat\theta_S-\theta_S|\le\varepsilon$ simultaneously with probability at least $1-\delta$ over $\mathcal I_{\mathrm E}$. With every largest-component tie counted as failure, Algorithm~\ref{alg:probe-search} then satisfies $P(|\widehat\tau-\tau|\le b+\varepsilon)\ge1-\delta-\alpha$ whenever
\(m\ \ge\ \frac{\log\{(L+1)/\alpha\}}{p\,(1-e^{-\Delta_\Pi^2/2})},\)
and the simpler budget $m\ge3\log\{(L+1)/\alpha\}/(p\Delta_\Pi^2)$ is sufficient. If $L=0$, the budget $m\ge\log(1/\alpha)/p$ suffices.

\end{corollary}

The budget scales as $\log\{(L+1)/\alpha\}/(p\Delta_\Pi^2)$. The ratio $R/m$ is unbiased for $p$, but $\Delta_\Pi$ is defined by the oracle partition and is not identified by the realized graph without additional assumptions.

\begin{corollary}[Median under majority]
\label{cor:median-under-majority}
\normalfont

Condition on the learning samples and seeds. Let $\widehat{\mathcal B}$ be fixed, define $\theta_S\triangleq\tau_{\widehat\phi_S(V_S)}$, and before the split draw and $\mathcal I_{\mathrm E}$ fix a partition of $\widehat{\mathcal B}$ that is fixed given $\mathcal G$ and contains a target cell $C_0$. Let $\mathcal R_m\subseteq\widehat{\mathcal B}$ be the retained splits with $R\triangleq|\mathcal R_m|$. Fix $b,\varepsilon\ge0$. Suppose $|\theta_S-\tau|\le b$ for every $S\in C_0$, and let $E$ be an event on which $|\widehat\theta_S-\theta_S|\le\varepsilon$ for every $S\in\mathcal R_m$. If $|\mathcal R_m\cap C_0|>R/2$, then on $E$ the median of $\{\widehat\theta_S:S\in\mathcal R_m\}$ lies within $b+\varepsilon$ of $\tau$. The linking step and cell-separation conditions are not needed.

\end{corollary}

When the target cell forms a strict retained majority, Cor.~\ref{cor:median-under-majority} dispenses with the linking step. With $b=0$ and $C_0=\{S:\theta_S=\tau\}$, its target-band premise reduces to the exact population target class used in Cor.~\ref{cor:target-majority-median}.

\section{Experiment Details and Evidence Boundaries}
\label{sec:appendix-experiments}

This appendix records the technical details behind Section~\ref{sec:main-experiments}. The tables in this appendix report the original runs of each experiment. Fig.~\ref{fig:experiments} and its detailed versions, Figs.~\ref{fig:appendix-synthetic} and~\ref{fig:appendix-realdata}, also include data sets added later with seeds fixed in advance, for $100$ data sets per synthetic setting, $100$ Twins replicates, and $88$ IHDP replicates in total, so their MAE values differ from the tables. The tables report PROBE with the blocks given; the figures report it with the blocks it finds from $W$, except in the image settings (Appendix~\ref{sec:appendix-block-discovery}).

\subsection{Shared structural family}

All four confirmations share one structure. Let $U$ be the hidden confounder, $I$ a treatment-only cause, $C$ an observed confounder, and let $\eta$, $\epsilon_0,\ldots,\epsilon_4$ and $\varepsilon_Y$ be independent standard normal noises. The informative coordinates, which no learner sees by name, are
\(X_1=U+2\eta,\qquad M_j=U+0.85\,\epsilon_j\ (j=0,\ldots,3),\qquad M_4=U-3.0\,I+0.85\,\epsilon_4.\)
Treatment and outcome follow
\(P(A=1\mid U,I,C,X_1)=0.1+0.8\,\Phi(1.5U+2.0I+0.2X_1+0.3C),\qquad Y(a)=a-2.5U+0.2X_1-0.5C+0.3\varepsilon_Y,\)
so the average treatment effect is $\tau=1$. The observed covariate $X$ carries $X_1$ and $C$. Blocks $W_1,W_2,W_3$ carry $M_1,M_2,M_3$, block $W_4$ carries the contaminated $M_4$, and block $W_5$ carries $M_0$ together with $I$. Each observed vector is a fixed random rotation of its informative coordinates stacked on pure noise coordinates with standard deviation $0.5$, so no observed coordinate is a named quantity. Placing $C$ inside $X$, which always enters the representation, prevents a representation from dropping an outcome-relevant coordinate in order to balance.

The three numeric settings differ only in how many noise coordinates hide the informative ones. SCM-1 gives $X$ two coordinates and the five blocks $(1,1,1,1,2)$ coordinates, for $8$ in total. SCM-2 gives $X$ eight and each block sixteen, for $88$. SCM-3 gives $X$ thirty-two and each block $256$, for $1{,}312$. Because the coefficients are shared, the population values and the census of balanceable splits are the same objects across the three.

In the population of this family, seven of the thirty candidate splits admit exact balance and target $\tau$, namely those built from the clean blocks $W_1,W_2,W_3$; the split that holds out $W_4$ also admits balance and targets a different value; the seven splits that mix $W_4$ with a clean block admit no balance; every split containing $W_5$ admits no balance at all. This is the census the screen is measured against when we report true and false positive rates.

\subsection{Image rendering}
\label{sec:appendix-image-rendering}

SCM-4 keeps the structure above and changes only what the learner receives in place of the six observed proxy coordinates. Each coordinate is recorded to $15$ equally spaced levels on $[-3\sigma,3\sigma]$, where $\sigma$ is its exact population standard deviation, and is then shown as one original $64\times64$ RGB Shapes3D image whose orientation factor equals that level. Every other generative factor of the image is drawn uniformly and independently, so the image carries the recorded level and nothing else about the unit. Blocks $W_1$ to $W_4$ contribute one image each and $W_5$ contributes two, giving six images per unit.

A frozen reader maps one image back to one number. It is a convolutional network trained with orientation labels on an image bank disjoint from every image used in the causal data, and it is never refitted during the experiment. Across the $14{,}400{,}000$ reads of all $100$ data sets it produced no level error, so SCM-4 measures the pipeline under an image proxy and not the discovery of a representation from raw pixels. The pixel CNN comparator in Section~\ref{sec:main-experiments} is a convolutional network fitted on all six images without holding out a block and without the balance screen.

Two of the twenty SCM-4 data sets reached the wall-clock cap for that comparator while three tasks shared one GPU, and the sealed summary therefore recorded FAIL on the completeness gate and on the comparison that depends on it. The two comparator runs were repeated alone under the sealed configuration and the sealed seeds, and nothing else was rerun or changed. The sealed summary and the amended summary are both retained; the amended one passes all $13$ gates.

\subsection{Search family, screen, and aggregation}

The search ranges over the thirty nonempty proper subsets of the five blocks, or over the $2^J-2$ subsets of the $J$ blocks found from $W$ (Appendix~\ref{sec:appendix-block-discovery}). For a split $S$ the representation is $Z_S=(X,Bv_{-S})$, where $v_{-S}$ stacks the preprocessed kept blocks and $B$ has rank $k=2$. Preprocessing standardises $X$ and reduces any block with more than two coordinates to its principal directions above the Marchenko--Pastur edge, at most four, fitted on the representation fold alone. The family is therefore linear in the observed coordinates and role-blind: it receives no label identifying a block, but it is not an unrestricted neural encoder.

Each data set is split into four folds of nearly equal size, $6{,}000$ rows each here, rotated so that every fold serves once as representation, screen, nuisance, and evaluation sample. The critics are bounded probit models with link $0.1+0.8\Phi$ and ridge $10^{-4}$; the augmented critic nests the base one, and the lifted base is a candidate whenever the augmented fit cannot improve on it. The screen statistic is the held-out Brier risk difference on the screen fold, refitted there rather than reused from the search and cross-fitted by row parity. A split is retained when $\max(\mathrm{gap},0)+z\,\mathrm{SE}$ falls below the threshold and every base-critic prediction on that fold lies in $[0.05,0.95]$, with $z=\Phi^{-1}(1-0.05/30)\approx 2.935$ for the thirty simultaneous decisions. The thresholds are $1.5$, $1.8$, $1.6$ and $1.7$ times $10^{-3}$ for SCM-1 to SCM-4, each fixed on development data before the confirmation seeds were generated. Retained splits are grouped by single linkage at twice a common radius, and the reported estimate is the median of the largest agreeing group, with no estimate returned when two groups tie for largest. Relative to Algorithm~\ref{alg:probe-search}, this implementation searches all $30$ splits ($m=T$), adds a separate screen fold to the three samples, retains a split only if it passes the screen in all four rotations, averages its estimate over the rotations, and sets the linking radius $\rho$ from the data by a Gaussian multiplier bootstrap. Appendix~\ref{sec:appendix-implementation} states every one of these choices in full.

\subsection{Honest learning and evaluation}

The four folds of a data set carry four separate jobs, and every rotation assigns each fold to each job once. The representation and its search critics use the representation fold. The screen refits fresh critics on the screen fold, so no statistic that decides retention was fitted on the rows that produced the representation. Propensity and outcome nuisance functions use the nuisance fold, and the AIPW score uses the evaluation fold. Preprocessing parameters are estimated only from the fold that is training at that step, and the four rotations are averaged.

For a fixed representation, the base critic predicts $A$ from $Z$, and the augmented critic predicts $A$ from $(X,W_S,Z)$. The empirical screen uses the clipped difference between their paired Brier risks. The runs additionally record critic reliability diagnostics, while the overlap check is part of the retention rule. These diagnostics are empirical guards. They are not upper confidence bounds on the population discrepancy.

The reported paired comparison uses one value per independently generated data set. For method $M$ and baseline $B$, define
\(d_k=|\widehat\tau_{M,k}-1|-|\widehat\tau_{B,k}-1|.\)
The prespecified one-sided upper bound is
\(\overline d+t_{0.95,K-1}\frac{\operatorname{sd}(d_1,\ldots,d_K)}{\sqrt K}.\)
The data set, not a proxy split, row, or training restart, is the replication unit. A no-return data set is included in the return-rate denominator and is not filled with zero error.

{%
\subsection{Modeling choices and implementation}
\label{sec:appendix-implementation}

\paragraph{Why this representation family.} Two requirements shaped the choice. The screen must certify balance from a finite sample, which asks for an encoder class whose uniform deviation term in Eq.~\eqref{eq:representation-oracle-discrepancy} is small, and the search must stay role-blind, so the class may not use any label that identifies a block. A rank-two linear map on preprocessed kept blocks meets both, and every fit it produces is deterministic given its seed, which the sealing protocol requires. The price is that this family cannot represent a nonlinear measurement channel. SCM-4 pays that price explicitly: a frozen network reads each image into a number before PROBE sees it, so the linear family acts on recovered numbers rather than on pixels.

\paragraph{Preprocessing.} All preprocessing is fitted on the representation fold alone. The covariates $X$ are standardised. A block with at most two coordinates passes through unchanged; with the blocks found from $W$, a block with at most four coordinates does (Appendix~\ref{sec:appendix-block-discovery}). A block with more coordinates is reduced by a singular value decomposition: with eigenvalues $s^2/n$ of the centred block, the noise level is the median eigenvalue $\sigma^2$, the Marchenko--Pastur edge is $\sigma^2(1+\sqrt{d/n})^2$ for $d$ coordinates, and the retained directions are those whose eigenvalue exceeds $1.1$ times that edge, at least one and at most four. The retained scores are standardised again.

\paragraph{Representation.} For a split $S$, let $v_{-S}$ stack the preprocessed kept blocks, of width $q$, and let $B\in\mathbb R^{k\times q}$ with $k=\min(2,q)$, so that $Z_S=(X,v_{-S}B^{\top})$. The rows of $B$ are normalised inside the objective, which makes the parameterisation scale invariant rather than constrained. The objective is the ridge-penalised Brier gap $G(B)=\widehat R_{0}-\widehat R_{1}$ of the two nested critics on the representation fold, and its gradient follows from the envelope theorem, projected onto the unit-row tangent space. The optimiser is L-BFGS-B with at most $60$ outer iterations, tolerances $10^{-14}$ and $10^{-12}$, and two starts: the top-$k$ right singular vectors of the centred kept features, and one Gaussian random matrix. The start with the lower $G$ is kept, and the critic fits are warm-started across outer iterations. No penalty is placed on $B$.

\paragraph{Critics.} Both critics are bounded probit models, $q=0.1+0.8\,\Phi(\eta)$, fitted by L-BFGS-B with at most $500$ iterations on the ridge-penalised Brier risk with penalty $10^{-4}$ on every coefficient except the intercept. The bounded link is what makes Assumption~\ref{ass:library-overlap} hold by construction for the fitted propensity and keeps the critic class $[0,1]$-valued, as Definition~\ref{def:empirical-held-out-proxy-discrepancy} requires. The base critic reads $(1,X,Z_S)$ and the augmented critic reads those inputs together with $T_S$. The base critic lifted into the augmented space, by padding its coefficients with zeros, is always a candidate for the augmented fit, so the empirical gap is nonnegative by construction.

\paragraph{Screen.} On the screen fold the critics are refitted from scratch and cross-fitted by row parity: even rows fit and odd rows score, then the roles are reversed. The statistic is the pooled out-of-fold paired Brier difference, with $\mathrm{gap}$ its mean and $\mathrm{SE}$ its standard deviation divided by $\sqrt{n_{\mathrm S}}$. A split is retained when $\max(\mathrm{gap},0)+z\,\mathrm{SE}$ falls below the threshold and every base-critic prediction on that fold lies in $[0.05,0.95]$, and it is retained overall only if it passes in all four rotations. The overlap condition is part of the rule, not a diagnostic.

\paragraph{Aggregation.} The common radius $\rho$ comes from a Gaussian multiplier bootstrap rather than from resampling. The evaluation-fold scores of the thirty splits are stacked into one matrix, $\Sigma$ is their centred cross-product divided by $n^2$, and $\rho$ is the $95$th percentile of $\max_j|\xi_j|$ over $5{,}000$ draws $\xi\sim N(0,\Sigma)$. Two splits are linked when their estimates differ by at most $2\rho$, groups are the connected components under single linkage, and the reported estimate is the median of the largest group. When two groups tie for largest, the run returns no estimate, and that data set enters the return-rate denominator.

\paragraph{Nuisance functions and the AIPW score.} The propensity is the same bounded probit at ridge $10^{-4}$, fitted on the nuisance fold and clipped to $[0.05,0.95]$. Each arm's outcome regression is a ridge least squares fit on the nuisance fold with an unpenalised intercept and penalty $10^{-4}n$. The score is evaluated on the evaluation fold, and a split's estimate is the plain mean of its four rotations.

\paragraph{Folds by setting.} Every setting splits its rows into four folds of nearly equal size and rotates them, so the $6{,}000$-row folds are specific to the synthetic settings.

\begin{center}
\footnotesize
\begin{tabular}{lrrll}
\toprule
Setting & $n$ & Data sets & Fold sizes & Screen rule \\
\midrule
SCM-1 to SCM-4 & 24,000 & 20 sealed, 100 total & $6{,}000$ each & fixed threshold \\
Twins & 11,984 & 10 sealed, 100 total & $2{,}996$ each & noise floor \\
IHDP & 747 & 10 sealed, 88 total & $187,187,187,186$ & noise floor \\
ACIC & 4,802 & 9, exploratory & $1{,}201$ or $1{,}200$ & noise floor \\
RHC & 5,735 & 1, effect unknown & $1{,}434$ or $1{,}433$ & noise floor \\
Kernel benchmark & 3,000 & 20 sealed, 100 total & $750$ each & pooled gap \\
\bottomrule
\end{tabular}
\end{center}

The synthetic settings require a pass in all four rotations at the fixed threshold. The kernel benchmark disables the fixed threshold and retains a split when its pooled gap is at most $z\,\mathrm{SE}$ with $z=\Phi^{-1}(1-0.05/2)=1.96$, since only two splits are available there. The noise floor pools the four rotations, taking $\mathrm{gap}=\max(\mathrm{mean},0)$ and $\mathrm{SE}=\sqrt{\sum\mathrm{SE}^2}/4$, and retains a split when $\mathrm{gap}+z\,\mathrm{SE}$ is at most the median of $z\,\mathrm{SE}$ over the thirty splits of the same data set. The realised floors are $0.002661$ for Twins, $0.021346$ for IHDP, $0.004694$ for ACIC, and $0.006335$ for RHC. Overlap is not part of this rule.

\paragraph{The image reader and the pixel comparator.} The image bank is split once, under a fixed seed, into $200{,}000$ reader-training images, $40{,}000$ selection images, and $240{,}000$ images reserved for the causal data, and the three parts are disjoint. The reader is four convolutions, $3\to32\to64\to128\to128$ channels, each $3\times3$ with stride $2$ and padding $1$ and a ReLU, followed by a flattening and two linear layers of widths $256$ and $1$. It is trained as a regression on the scaled level $(\ell-7)/7$ under squared loss for six epochs with batch size $256$ and Adam at learning rate $10^{-3}$, the epoch with the lowest validation error is kept, and the weights are then frozen, hashed, and reloaded in evaluation mode for every run. The pixel comparator starts from the reader: it loads the reader's convolutional tower and its $256$-dimensional feature layer, adds a linear map to sixteen outputs whose first row is copied from the reader's output layer, and shares that tower across the six images. The sixteen outputs per image and the standardised covariates enter two hidden layers of widths $128$ and $64$ with ReLU activations and three heads for the propensity and the two outcome regressions. Training minimises the Brier loss for treatment plus the squared loss for the standardised outcome, with Adam at $10^{-4}$ for the tower and $10^{-3}$ elsewhere, five epochs, batch size $128$, a validation share of $0.2$, and patience two. The comparator uses PROBE's folds and rotations and the same clipping, and it carries a wall-clock cap of $2{,}400$ seconds for its four rotations.

\paragraph{Which comparators share the folds.} Adjustment for $X$, adjustment for $(X,W)$, the unchecked summary, the oracle, and the pixel comparator all use PROBE's folds and rotations, with the bounded probit propensity and arm-wise ridge outcome models whose ridge levels are chosen by three-fold cross-validation on the nuisance fold. The oracle adjusts for the latent coordinates rather than for the rotated observed covariates, so it is stronger than an estimator that sees $X$ and $U$ in observed form. CEVAE, proximal 2SLS, single-proxy control, and KSPC are fitted on the whole data set without PROBE's folds and without cross-fitting, as in their source implementations.

\paragraph{Seeds and sealing.} A sealed run writes a protocol file once, recording the sample size, the full design, and a hash of every source file it will execute; a later run aborts if any hash or any design field differs. Shards are written atomically and never over an existing file, and summaries are also write-once. Fold seeds, representation start seeds, bootstrap seeds, and comparator seeds are fixed offsets of the task seed, so a rerun of a task reproduces its fits exactly on the same hardware.
}

\subsection{Finding the blocks from the proxy}
\label{sec:appendix-block-discovery}

In Fig.~\ref{fig:experiments} and in Figs.~\ref{fig:appendix-synthetic} and~\ref{fig:appendix-realdata}, PROBE is not told which coordinates of $W$ form each block, except in the image settings. It finds the blocks from $W$ alone, on all rows and before it sees $A$ or $Y$, like the unsupervised preprocessing, and Algorithm~\ref{alg:probe-search} then runs unchanged on the blocks it finds. The tables of this appendix report the original runs, in which the blocks are given.

\paragraph{Procedure.} Let $R$ be the correlation matrix of the $m$ coordinates of $W$ over the $n$ rows. (i) Keep the $K$ eigenvalues $\lambda_1,\dots,\lambda_K$ of $R$ above the Marchenko--Pastur edge $(1+\sqrt{m/n})^2$, with eigenvectors $v_1,\dots,v_K$, and give coordinate $i$ the loading row $\ell_i=(\sqrt{\lambda_1}v_{1i},\dots,\sqrt{\lambda_K}v_{Ki})$. Coordinates that measure the same value have loading rows that point in one direction; coordinates that mix two values have loading rows in one plane. (ii) Cluster the coordinates into $J=5$ groups by spectral clustering of the affinity $|\cos(\ell_i,\ell_j)|^4$: normalise the affinity, take its $J$ leading eigenvectors, and run $k$-means with $20$ restarts. (iii) Refine by $k$-subspaces. Each group gets the line spanned by its loading rows or, when $K>2$ and the second singular value of the rows is at least $0.3$ times the first, the plane; every coordinate moves to the group whose subspace makes the smallest angle with its loading row; repeat until no coordinate moves. (iv) Merge two groups when they hold a pair of coordinates whose partial correlation, given all other coordinates, exceeds $0.5$. Such a pair is nearly redundant given the rest, so a split that holds out one and keeps the other is a weak check that can pass while biased. With $J$ found blocks the search ranges over the $2^J-2$ nonempty proper subsets, the screen's level is $\alpha/(2^J-2)$, and a block with at most four coordinates passes through preprocessing unreduced, since the reduction keeps at most four directions; for blocks of one, two, $16$, or $256$ coordinates this is the preprocessing of the original runs.

\paragraph{What it finds.} With six coordinates, as in SCM-1, Twins, IHDP, and ACIC, the signal space has $K=2$ dimensions, and the coordinate of $W_4$ and the two coordinates of $W_5$ are coupled: the largest partial correlation between two groups is $0.65$ to $0.68$, and every other pair of groups stays at or below $0.28$. The found blocks are $\{W_1\}$, $\{W_2\}$, $\{W_3\}$, and $\{W_4,W_5\}$ in every data set, so the valid held-out sets are those of the given blocks, and in all $197$ benchmark replicates the screen retains only held-out sets inside $\{W_1,W_2,W_3\}$. In SCM-2 the given blocks are recovered exactly in $50$ of the $100$ data sets and otherwise up to at most two of the $80$ coordinates; in SCM-3, $104$ of the $1{,}280$ coordinates per data set sit in another block on average. The misplaced coordinates are the weakest ones: they carry $0.02\%$ (SCM-2) and $0.4\%$ (SCM-3) of the proxy's signal on average, and no block made mostly of valid coordinates receives more than $0.01\%$ of the treatment-only content. On RHC no pair of groups is coupled, and the found blocks are $\{$PaO$_2$/FiO$_2$, sodium$\}$, $\{$PaCO$_2$, pH, white cell count$\}$, $\{$hematocrit, bilirubin$\}$, $\{$albumin$\}$, and $\{$creatinine, potassium$\}$. The screen retains $15$ splits, as it does with the clinical grouping, and every retained split keeps the block that holds PaO$_2$/FiO$_2$, as every retained split does with the clinical grouping.

\begin{table}[t]
\centering
\caption{PROBE with the blocks given and with the blocks found from $W$: MAE over the data sets of Fig.~\ref{fig:experiments}, and for RHC the estimate with its resampled $95\%$ interval.}
\label{tab:block-discovery}
\begin{tabular}{lcc}
\toprule
Setting & Blocks given & Blocks found from $W$\\
\midrule
SCM-1 ($100$ data sets) & $0.0187$ & $0.0184$\\
SCM-2 ($100$) & $0.0184$ & $0.0195$\\
SCM-3 ($100$) & $0.0178$ & $0.0213$\\
Twins ($100$) & $0.0073$ & $0.0065$\\
IHDP ($88$) & $0.409$ & $0.390$\\
ACIC ($9$) & $0.203$ & $0.187$\\
RHC & $-1.29$ $(-1.88,-0.70)$ & $-1.35$ $(-1.93,-0.75)$\\
\bottomrule
\end{tabular}
\end{table}

\paragraph{Confirmation on new data.} Step (iv) was added after the first run on Twins, in which the splits that held out one coordinate of $W_5$ passed the screen with a mean bias of $-0.020$, and its threshold $0.5$ was read from the gap in the coupling values above. With the procedure then fixed, we ran it on data sets not used before: Twins replicates $111$ to $210$ and $20$ new data sets per synthetic setting. On the new Twins replicates the MAE is $0.0057$ for PROBE with the found blocks, $0.0069$ with the given blocks, $0.0082$ for KSPC given the correct block, $0.0137$ for proximal 2SLS with the correct roles, and $0.0058$ for the oracle; the paired $t$-statistics of the found blocks against the given blocks and against KSPC are $-3.1$ and $-3.7$. On the new synthetic data sets the MAE with the found against the given blocks is $0.0165$ against $0.0177$ in SCM-1, $0.0239$ against $0.0197$ in SCM-2, and $0.0226$ against $0.0205$ in SCM-3.

\paragraph{Limitation.} In SCM-2 and SCM-3 the misplaced weak coordinates cost accuracy: on the $100$ data sets of Fig.~\ref{fig:experiments} the MAE rises by $0.0011$ in SCM-2 and by $0.0035$ in SCM-3 (Table~\ref{tab:block-discovery}). Judging both block splits by the noise-scaled screen of the real-world analyses instead of the sealed tolerance changes no MAE by more than $0.0005$. A likelihood-ratio form of step (iii), which charges a plane for its extra dimension, placed fewer valid coordinates in the plane of $W_5$ in SCM-2 but moved more coordinates of $W_4$ and $W_5$ into valid blocks in SCM-3, and it did not lower the MAE, so we did not adopt it. Once the signal space is truncated to $K$ dimensions, the lines of the valid blocks and of $W_4$ lie close to the plane of $W_5$, and second-order statistics alone cannot place a coordinate whose loading row is mostly noise.

\subsection{Additional results for the synthetic analysis}
\label{sec:appendix-synthetic-results}

 This subsection details the original sealed confirmation runs, summarized in Table~\ref{tab:verified-results}. Fig.~\ref{fig:synthetic-detail} shows every method on its own row, and Fig.~\ref{fig:appendix-kspc} shows the kernel single-proxy benchmark.

Methods that need proxy roles receive the roles of the data-generating process in their correct-role runs. Proximal 2SLS uses $W_4$, which the treatment-only cause $I$ affects, as the treatment-side proxy, a role $W_4$ can fill although it is not a valid held-out block, and $W_2$ as the outcome-side proxy, and single-proxy control and KSPC use the clean block $W_1$. The wrong-role runs swap the two 2SLS proxies and give single-proxy control and KSPC the contaminated block $W_4$. KSPC is our implementation of the estimator of \citet{xu2025kernel} with observed covariates, using its SKPV estimator only, on a random subsample of $3{,}000$ rows when a data set is larger.

\paragraph{Comparators.}Adjustment for $X$, adjustment for $(X,W)$, and the oracle, which also adjusts for $U$, are AIPW estimators on the same folds and rotations as PROBE. Proximal 2SLS is linear. It regresses the outcome-side proxy on $A$, the treatment-side proxy, and $X$, and then regresses $Y$ on $A$, the fitted proxy, and $X$; a block with more than four coordinates enters through its first four principal components. Single-proxy control fits the linear outcome bridge of \citet{park2024single} in $X$ and the first principal component of its block on the untreated units, which gives the effect on the treated. This effect equals $\tau$ in SCM-1 to SCM-4 because the effect is constant. As an ATE estimator, single-proxy control also fits the bridge on the treated units, which gives the effect on the untreated, and averages the two effects with weights equal to the shares of treated and untreated units. CEVAE is reimplemented in PyTorch with the architecture and default settings of the Pyro reference implementation: a $20$-dimensional latent, three hidden layers of $200$ units, $30$ epochs, batch size $100$, and learning rate $10^{-3}$, fitted on the standardised $(X,W)$. In SCM-4 the comparators that use the proxy receive the six numbers of the image reader, except the pixel CNN, which starts from the reader's weights and fine-tunes them on the images.

The kernel single-proxy benchmark uses the generator of \citet{xu2025kernel} with a hidden $U\sim\mathrm{Unif}(-1,1)$, a binary treatment, $P(A=1\mid U)=0.1+0.8\{1+\mathrm{erf}(U)\}/2$, and the deterministic outcome $Y=\sin(\pi U/2)+A-0.3$, so $\tau=1$. PROBE receives the two proxies of that generator as two blocks, $W_1=e^U+\epsilon$ with $\epsilon\sim N(0,0.1^2)$ and $W_2$, an MNIST image of the digit $\lfloor 5U+5\rfloor$; KSPC receives one of them at a time. Here KSPC is the authors' code without covariates, and Fig.~\ref{fig:appendix-kspc} shows both of its estimators, SKPV and SPMMR. With two splits, PROBE retains a split when its mean gap over the four rotations is at most $z$ times its standard error, with $z=\Phi^{-1}(1-0.05/2)$. The benchmark has $20$ replicates of size $n=3{,}000$.

\begin{figure}[t]
\centering
\includegraphics[width=\linewidth]{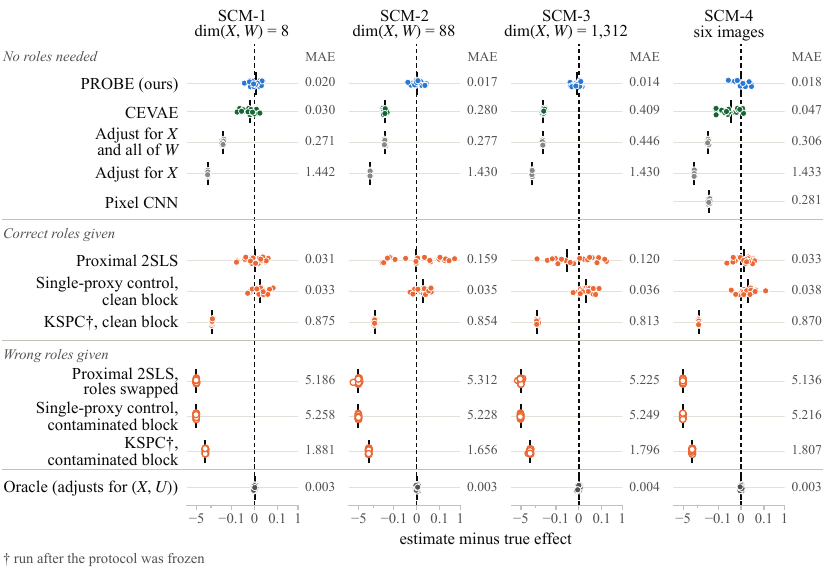}
\caption[Detailed errors in the synthetic settings.]{Detailed version of the left of Fig.~\ref{fig:experiments} for the original $20$ data sets per setting, with every method on its own row. KSPC ($\dagger$) was run after the protocol was frozen.}
\label{fig:synthetic-detail}
\end{figure}

\begin{figure}[t]
\centering
\includegraphics[width=0.55\linewidth]{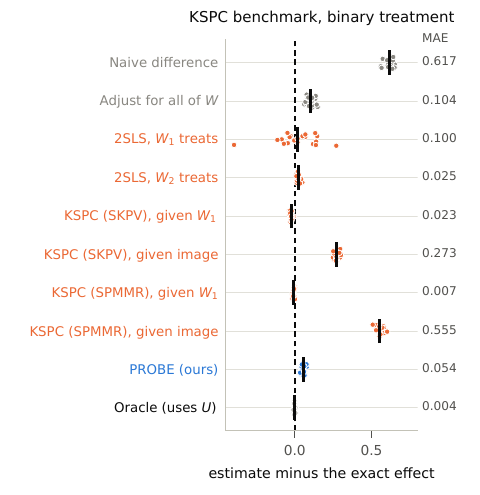}
\caption[Kernel single-proxy benchmark.]{Kernel single-proxy benchmark of \citet{xu2025kernel} with a binary treatment; the effect is exactly $1$. Estimates minus the effect over $20$ replicates, with the MAE on the right.}
\label{fig:appendix-kspc}
\end{figure}

Fig.~\ref{fig:synthetic-detail} gives the detailed comparison. PROBE answers every one of the $80$ original data sets, with mean absolute error between $0.014$ and $0.020$ and 90th percentile error between $0.033$ and $0.046$. Its prespecified one-sided paired bound on the excess over the oracle ranges from $0.015$ to $0.023$. The two ways of using the proxy without an audit both fail: adjustment for $X$ and the entire raw proxy has mean estimate between $0.55$ and $0.73$ because it absorbs the treatment-only coordinate, and in SCM-4 a convolutional network that reads every pixel and never tests balance has mean estimate $0.72$. The paired bound against that pixel network is $-0.25$.

The prespecified-block contrast isolates identification from search, as Thm.~\ref{thm:exact-proxy-identification} requires. Holding out the clean block $W_1$ and balancing it gives $0.998$, $0.987$, $0.980$ and $1.000$ in the four settings. Holding out the contaminated block $W_4$ gives $0.660$, $0.656$, $0.616$ and $0.667$, and the balance screen retains that split in $79$ of the $80$ original data sets and $396$ of all $400$, because $W_4$ really can be balanced. Holding out $W_5$ gives about $0.29$ and is rejected in all $80$ original data sets and all $400$, since no representation balances a block that carries the treatment-only cause. Balance is observable, and on its own it is not enough.

Aggregation repairs the one failure that balance leaves. Thm.~\ref{thm:unknown-valid-block-identification} and Cor.~\ref{cor:target-majority-median} predict that the valid splits outvote $W_4$ when they form a majority, and they do: averaging over the retained splits is dragged to $0.963$, $0.956$, $0.945$ and $0.959$, the plain median of the retained splits already recovers most of that at $1.002$, $0.998$, $0.991$ and $0.997$, and the median of the largest agreeing component that PROBE returns reaches $1.005$, $1.003$, $0.993$ and $1.001$. The screen has true positive rate between $0.986$ and $1.000$ and false positive rate between $0.041$ and $0.046$, every returned component is pure, and every data set returns a unique largest component.

The two bounds of Thm.~\ref{cor:representation-causal} and Thm.~\ref{thm:conditional-aipw-error} are visible separately. For a single split in SCM-1 to SCM-3, the population effect of the learned representation sits a median of $0.042$ to $0.046$ from $\tau$, with 90th percentile $0.103$ to $0.110$, and the realised error of that one split has median $0.051$ to $0.059$. Aggregating the retained splits cuts this to the $0.014$ to $0.020$ reported above, so the combination step, not a single well-behaved representation, carries the accuracy. The same separation shows up in the intervals. A nominal $95\%$ interval around one split, built from the estimate's own standard error, covers $\tau_{\widehat Z}$ at $0.950$, $0.957$ and $0.936$ in SCM-1 to SCM-3, which is the nominal rate, and covers $\tau$ at $0.836$, $0.857$, $0.879$ and $0.807$. The shortfall is the representation gap, which no interval built from sampling variation can contain; Fig.~\ref{fig:appendix-coverage} shows it per data set.

Two boundaries limit what these results show. The image setting hands PROBE numbers produced by a reader that was trained with orientation labels on an auxiliary image bank, disjoint from the causal data; that reader made no level error in $14{,}400{,}000$ reads of all $100$ data sets, so SCM-4 tests the pipeline under an image proxy rather than the discovery of a representation from raw pixels. The second boundary is the one the $W_4$ contrast exposes: the screen certifies balance and cannot certify the outcome-relevant condition that Assumption~\ref{ass:outcome-relevant-completeness} adds.

\begin{table}[t]
\caption{Sealed confirmations. Each row is $20$ independently generated data sets at $n=24{,}000$.}
\label{tab:verified-results}
\centering
\footnotesize
\begin{tabular}{llrrrl}
\toprule
Setting & \shortstack[l]{\\What the learner sees} & Mean error & p90 & Oracle gap & Status \\
\midrule
SCM-1 & $8$ coordinates & 0.020 & 0.033 & 0.022 & 12/12 gates \\
SCM-2 & $88$ coordinates & 0.017 & 0.037 & 0.019 & 12/12 gates \\
SCM-3 & $1{,}312$ coordinates & 0.014 & 0.035 & 0.015 & 12/12 gates \\
SCM-4 & six $64\times64$ images & 0.018 & 0.046 & 0.023 & 13/13 gates, amended \\
\bottomrule
\end{tabular}
\end{table}

The SCM-4 row carries a correction that we report rather than hide. Two of its $20$ data sets reached the wall-clock cap while three tasks shared one GPU, so the pixel-network comparator recorded no value there and the sealed summary returned FAIL on the completeness gate and on the comparison that depends on it. Those two comparator runs were repeated alone under the sealed configuration and the sealed seeds, every other number was left untouched, and the amended summary passes all $13$ gates.

\begin{figure}[t]
\centering
\includegraphics[width=\linewidth]{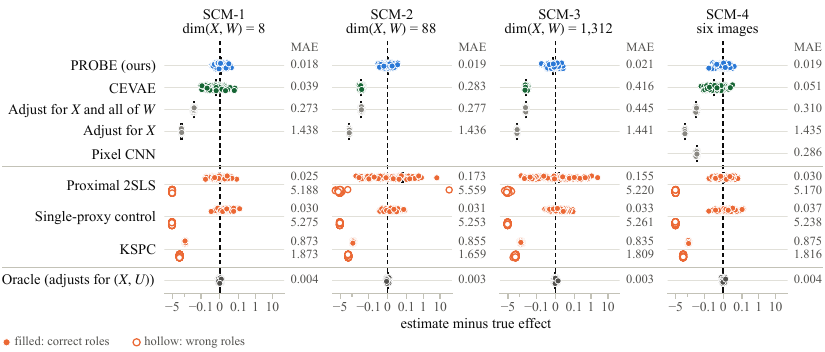}
\caption[Detailed synthetic results.]{Detailed version of the left of Fig.~\ref{fig:experiments} over all $100$ data sets per setting, with the MAE of each method on the right; for methods that need proxy roles, the upper number is for the correct roles. PROBE finds the blocks from $W$ (Appendix~\ref{sec:appendix-block-discovery}) except in SCM-4, where they are given.}
\label{fig:appendix-synthetic}
\end{figure}

\begin{figure}[t]
\centering
\includegraphics[width=\linewidth]{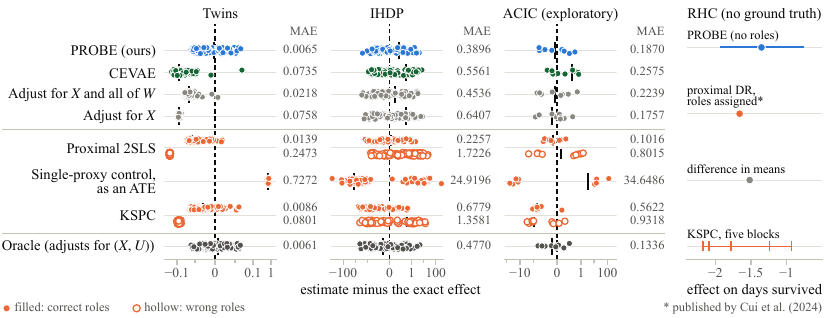}
\caption[Detailed semi-synthetic and RHC results.]{Detailed version of the right of Fig.~\ref{fig:experiments} over $100$ Twins, $88$ IHDP, and $9$ ACIC data sets, with the MAE of each method, and the RHC estimates, where no true effect is known; PROBE's estimate carries its resampled $95\%$ interval, and * marks estimates published by \citet{cui2024semiparametric}. PROBE finds the blocks from $W$, including the ten RHC measurements (Appendix~\ref{sec:appendix-block-discovery}).}
\label{fig:appendix-realdata}
\end{figure}

\subsection{Chest X-ray analysis}
\label{sec:appendix-single-image}

\paragraph{Data and hidden confounder.} The images are the $64\times64$ release of PneumoniaMNIST in MedMNIST+ \citep{yang2023medmnist}, built from the pediatric chest X-ray collection of \citet{kermany2018identifying} (CC BY 4.0): $5{,}856$ images, $1{,}583$ normal and $4{,}273$ with pneumonia, with the official training, validation, and test parts pooled. The hidden confounder $U$ is the standardised logit of a small convolutional classifier of the pneumonia label, fitted on one half of the images and applied to the other half and vice versa (six epochs; cross-fitted AUC $0.983$). The binary label itself is not used as $U$, because the image carries only part of it: a linear probe on whole-image features explains $66\%$ of its variance, so no adjustment that uses the image could recover the effect. The logit is an image-determined severity, and the representation below recovers $91\%$ of its variance. Treatment and outcome follow $P(A=1\mid U)=0.1+0.8\,\Phi(1.5U)$ and $Y=A+U+\varepsilon$ with $\varepsilon\sim N(0,1)$, so $\tau=1$, except in the mediator scenario below. Each of the $50$ replicates of a scenario redraws $A$, $Y$, the noise, and the overlays; the images are fixed. The $50$ replicates were generated with new seeds after an exploratory phase of ten replicates per scenario, on which the rule below was fixed.

\paragraph{Blocks and scenarios.} In pixel coordinates (row, column; zero-based, inclusive), $W_1$ is rows $8$ to $30$ and columns $4$ to $27$, $W_2$ rows $33$ to $55$ and columns $4$ to $27$, $W_3$ rows $8$ to $30$ and columns $36$ to $59$, $W_4$ rows $33$ to $55$ and columns $36$ to $59$, and $W_5$ rows $0$ to $63$ and columns $29$ to $34$. Blocks are separated by at least two pixels, and every overlay is drawn only on the pixels of its block. A \emph{device} is a vertical line of intensity $+0.35$ at rows $4$ to $44$ and columns $31$ to $32$ of $W_5$, drawn with probability $0.9$ when $A=1$ and $0.1$ when $A=0$; the footprint in a lung block draws the same line at rows $36$ to $52$ and columns $14$ to $15$ of $W_2$. A \emph{response mark} is a Gaussian spot of standard deviation two pixels with amplitude $0.25$ times the standardised value of $U-\kappa A+0.5\nu$, $\nu\sim N(0,1)$, clipped to $[-2.5,2.5]$, with $\kappa=1.5$ unless stated; it is centred at row $20$, column $31.5$ in $W_5$, at row $44$, column $15.5$ in $W_2$, and at rows $19$ and $44$ and columns $15.5$ and $47.5$ in the four quadrants. The scenarios are: clean; device in $W_5$; device line in $W_2$; response mark in $W_5$; mediator, in which $M=A+0.5\nu$ darkens the position of the mark in $W_5$ by $0.25$ times its standardised value and $Y=0.5A+0.5M+U+\varepsilon$, so that the total effect is still $1$ and adjustment for the mark targets the direct effect $0.5$; the response mark in $W_2$ and $W_5$, drawn from one value of $U-1.5A+0.5\nu$; the response mark in all five blocks; a device collider, the line in $W_5$ drawn with probability $\operatorname{expit}(-2+2A+1.5U)$ ($11\%$ of untreated and $67\%$ of treated units); the device in $W_5$ together with the response mark in $W_2$; a near-deterministic device, the line in $W_5$ drawn with probability $0.98$ when $A=1$ and $0.02$ when $A=0$; and the imprint sweep, the response mark in $W_5$ with $\kappa\in\{0,0.5,1,2,3\}$.

\paragraph{Block-additive representation.} The features of block $k$ come from a network that sees the image with only block $k$ visible, the other pixels set to zero and a visibility channel appended. The network has the convolutional trunk of the synthetic image experiments and a head that predicts $A$ by the logistic loss and the standardised $Y$ by the squared loss from the block alone; it is trained for six epochs on one half of the units and applied to the other half and vice versa, so that no unit's features come from a network that saw that unit's treatment or outcome. Its pooled trunk output ($128$ numbers) is reduced to $16$ principal components fitted on all rows without labels, and the five blocks are concatenated to $80$ features; hiding a block removes exactly its $16$ features. The cross-validated $R^2$ of $U$ on the $80$ features is $0.91$ in every scenario. An unsupervised masked autoencoder \citep{he2022masked} in place of the supervised trunk gave $R^2=0.875$ and a higher floor in the exploratory runs and is not used here. Whole-image features are not block-additive: hiding a block changes the embedding of the remaining pixels, which confounds the information cost of hiding a block with a change of the remaining representation.

\paragraph{Audit and estimation.} For block $j$, the retained representation $Z_{-j}$ stacks the features of the other blocks, and the held-out block $T_j$ is the pixel vector of block $j$ reduced by the per-block principal component rule of the sealed code. The screen of Appendix~\ref{sec:appendix-experiments} is applied unchanged on each of the four rotations, and the gap and its standard error are pooled over rotations. The rule flags the block with the largest gap when that gap exceeds the median gap of the audited blocks by more than two standard errors and exceeds an absolute floor of $0.005$; the flagged block is hidden, the remaining blocks are audited again, and the step repeats at most once more, so that at most two blocks are hidden. The absolute thresholds of the synthetic settings do not transfer, because the noise level of the gap depends on the $80$-dimensional representation, for the same reason that led to the noise-scaled rule of Appendix~\ref{sec:appendix-benchmarks}. The one-pass rule of the exploratory phase, which flags every block whose gap exceeds the median by more than two standard errors without a floor or a second audit, is reported alongside. The estimate is the sealed AIPW on the remaining blocks; adjustment for the whole image, the pixel CNN of the main text, uses all $80$ features in the same AIPW, and the oracle adjusts for $U$. On clean images the five gaps lie between $-0.0004$ and $0.0004$ with standard errors near $0.001$; a block that carries a mark has a gap of $0.036$ to $0.21$.

\paragraph{Results.} Table~\ref{tab:single-image-appendix} gives all scenarios. Under the sequential rule the hidden blocks are exactly the blocks that carry the mark in all $50$ replicates of the response, mediator, two-block, device-and-response, and near-deterministic-device scenarios, in $49$ of the lung-footprint replicates, and nothing is hidden in any clean replicate. The one-pass rule differs in two ways: it flags a clean block in addition to the marked one in up to $13$ replicates per scenario, because the mark left in $Z_{-j}$ raises the gaps of the clean blocks from about $0.001$ to about $0.004$, and it misses the second block of the two-block response in $45$ replicates, because given $W_5$ the mark in $W_2$ repeats information the representation already holds, so its gap is small until $W_5$ is hidden. A device alone, drawn by $A$ and noise, does not change the estimate whether it is hidden or kept, as Prop.~\ref{prop:single-image-decomposition}(c) states; the near-deterministic device raises the MAE of the pixel CNN to $0.16$, against $0.14$ for the device in $W_5$, without a visible bias, and hiding it removes the increase. Two scenarios show where exclusion stops helping. When the response mark is drawn in every block, the five gaps are similar ($0.0005$ to $0.008$), the median is itself a marked block's gap, and hiding the block with the largest gap only loses information ($1.44$ against $1.39$). The device collider is a block whose two causes act in the same direction, so conditioning on it induces a negative association between $A$ and $U$ that offsets part of the positive residual confounding; the screen flags it because it carries $A$, and hiding it costs $0.04$. Exclusion therefore helps when content that $A$ and $U$ jointly cause, or that $A$ passes on to $Y$, is confined to few blocks and moves the estimate away from $\tau$, and the main text reports the scenarios of each kind.

\begin{table}[t]
\caption{Chest X-ray analysis, all scenarios (hand-drawn blocks, supervised encoder, $50$ replicates, true effect $1$): mean estimate, with the MAE in parentheses. The last column counts the replicates in which the sequential rule hides exactly the blocks that carry the mark; in the clean scenario, the replicates in which it hides nothing. The difference in means is between $2.10$ and $2.12$ in every scenario.}
\label{tab:single-image-appendix}
\centering
\footnotesize
\begin{tabular}{lccccc}
\toprule
Scenario & Pixel CNN & PROBE, one pass & PROBE, sequential & Oracle & Exact \\
\midrule
Clean & 1.132 (0.132) & 1.135 (0.135) & 1.132 (0.132) & 0.999 & 50 \\
Device in $W_5$ & 1.137 (0.139) & 1.129 (0.129) & 1.128 (0.128) & 0.998 & 50 \\
Device line in $W_2$ & 1.158 (0.158) & 1.163 (0.163) & 1.161 (0.161) & 1.006 & 49 \\
Response in $W_5$ & 1.268 (0.268) & 1.150 (0.150) & 1.138 (0.138) & 1.003 & 50 \\
Mediator in $W_5$ & 0.740 (0.260) & 1.152 (0.152) & 1.149 (0.149) & 1.016 & 50 \\
Response in $W_2$ and $W_5$ & 1.327 (0.327) & 1.295 (0.295) & 1.171 (0.171) & 1.001 & 50 \\
Device in $W_5$ and response in $W_2$ & 1.291 (0.291) & 1.164 (0.164) & 1.164 (0.164) & 0.993 & 50 \\
Response in every block & 1.390 (0.390) & 1.430 (0.430) & 1.436 (0.436) & 0.992 & -- \\
Device collider in $W_5$ & 1.088 (0.090) & 1.129 (0.129) & 1.130 (0.130) & 0.995 & 48 \\
Near-deterministic device in $W_5$ & 1.148 (0.155) & 1.128 (0.128) & 1.126 (0.126) & 0.994 & 50 \\
\bottomrule
\end{tabular}
\end{table}

\paragraph{Imprint sweep.} Fig.~\ref{fig:single-image} (right) repeats the response scenario with $\kappa\in\{0,0.5,1,1.5,2,3\}$. The MAE of the pixel CNN is $0.10$, $0.19$, $0.24$, $0.27$, $0.33$, and $0.57$, and that of PROBE $0.10$, $0.14$, $0.13$, $0.14$, $0.12$, and $0.14$. At $\kappa=0$ the mark is a function of $U$ and noise, the gap of $W_5$ is $0.0008$, nothing is hidden, and the extra information about $U$ lowers both errors below the clean case ($R^2=0.93$). At $\kappa=0.5$ the gap is $0.009$ and $W_5$ is hidden in $44$ replicates; from $\kappa=1$ on it is hidden in all $50$.

\paragraph{Automatic block layouts.} The five blocks above were drawn by hand from anatomy. Table~\ref{tab:single-image-layouts} repeats the main scenarios with a $4\times4$ grid of $16$-pixel cells, eight principal components per cell, the same encoder, the same rule with at most three cells hidden, and the same replicates. The grid covers the whole image but carries no more of $U$ ($R^2=0.91$) and gives a similar error on clean images ($0.12$ against $0.13$). A mark straddles two cells, and the grid hides both only when each has a gap of its own: under the response it does so in all replicates and the error matches the hand-drawn blocks, whereas under the mediator both cells are hidden in $20$ of $50$ replicates and the MAE stays at $0.22$. The hand-drawn layout was therefore kept, by a rule fixed before these replicates were run. In the exploratory phase, with the masked autoencoder, grids of $3\times3$ and $4\times4$ cells had given the same or a smaller error than the hand-drawn blocks under a treatment response.

\begin{table}[t]
\caption{Chest X-ray analysis with a $4\times4$ grid in place of the hand-drawn blocks (supervised encoder, $50$ replicates, true effect $1$): mean estimate of adjustment for all cells and of PROBE, with the MAE in parentheses.}
\label{tab:single-image-layouts}
\centering
\footnotesize
\begin{tabular}{lcccc}
\toprule
 & \multicolumn{2}{c}{Hand-drawn blocks} & \multicolumn{2}{c}{$4\times4$ grid} \\
\cmidrule(lr){2-3}\cmidrule(lr){4-5}
Scenario & Pixel CNN & PROBE & All cells & PROBE \\
\midrule
Clean & 1.132 (0.132) & 1.132 (0.132) & 1.119 (0.120) & 1.119 (0.120) \\
Response in $W_5$ & 1.268 (0.268) & 1.138 (0.138) & 1.325 (0.325) & 1.127 (0.128) \\
Mediator in $W_5$ & 0.740 (0.260) & 1.149 (0.149) & 0.714 (0.286) & 0.871 (0.225) \\
Response in $W_2$ and $W_5$ & 1.327 (0.327) & 1.171 (0.171) & 1.359 (0.359) & 1.289 (0.289) \\
Response in every block & 1.390 (0.390) & 1.436 (0.436) & 1.407 (0.407) & 1.421 (0.421) \\
Device collider in $W_5$ & 1.088 (0.090) & 1.130 (0.130) & 1.074 (0.081) & 1.074 (0.081) \\
\bottomrule
\end{tabular}
\end{table}

\paragraph{Resolution.} Doubling the resolution does not lower the floor. With the $128\times128$ release of the same images, the masks and overlays scaled by two, the same encoder, and the sequential rule, we ran the clean, device-in-$W_5$, device-line-in-$W_2$, and response scenarios with $10$ data sets each on new seeds. The $R^2$ of $U$ is $0.907$ ($0.900$ to $0.911$), as at $64$ pixels. On clean images the pixel CNN and PROBE both give $1.104$ (MAE $0.107$), and under the treatment response the pixel CNN gives $1.250$ ($0.250$) and PROBE $1.140$ ($0.140$), against $1.132$ ($0.132$), $1.268$ ($0.268$), and $1.138$ ($0.138$) at $64$ pixels in Table~\ref{tab:single-image-appendix}; each difference is about one standard error of a mean over ten data sets ($0.01$ to $0.02$) or less. PROBE hides nothing in any clean data set and exactly the marked block in all $30$ marked data sets. The hidden $U$ is computed from the $64$-pixel images, so the higher resolution adds no information about it; the floor is set by the encoder and by the pixels that the blocks cover, not by the number of pixels.

\paragraph{What this experiment does and does not show.} The overlays are synthetic and the data set is one. In the scenarios where exclusion helps, the treatment-dependent content sits in one or two regions that carry little of $U$; a treatment response in the lungs themselves, where most of $U$ is, falls under the every-block scenario, in which exclusion cannot help. Exploratory runs that preceded this design are kept in the repository and bound the claim: on a lung phantom without treatment-dependent content, and on the clean X-rays, adjusting for the whole image was as accurate as PROBE or more accurate, because holding out a block costs information and there was nothing to exclude; learning the representation from pixels with the balance objective alone, by a rank-two linear map or by fine-tuning a convolutional network, was less accurate than the block-additive encoders; and the gap that remains between every image-based method and the oracle is the part of $U$ that the image representation does not carry, which Prop.~\ref{prop:single-image-decomposition} shows no choice of blocks can remove.

\subsection{Semi-synthetic benchmarks}
\label{sec:appendix-benchmarks}

Twins \citep{louizos2017causal}, IHDP \citep{hill2011bayesian}, and ACIC 2016 \citep{dorie2019automated} keep their covariates and both potential outcomes, so the effect is known exactly. In each, one covariate is hidden as $U$ (gestational age in Twins, the sixth of the $25$ IHDP covariates, and $x_{44}$ in ACIC), and the five proxy blocks and the treatment are generated from it by one fixed recipe. The blocks follow the proxy mechanism of SCM-1 with a generated standard normal treatment-only cause $I$, and the treatment follows $P(A=1\mid U,I,X)=0.1+0.8\,\Phi(1.5U+2.0I+w^\top\tilde X)$, where $U$ is standardised, $\tilde X$ is the standardised $X$, and $w\sim N(0,0.1^2I)$ is drawn once. Twins has $n=11{,}984$ and $42$ covariates, IHDP $n=747$ and $24$, and ACIC $n=4{,}802$ and $78$. The original runs use Twins replicates $11$ to $20$, IHDP replicates $12$ to $21$, and ACIC realizations $2$ to $10$. The comparators take the proxy roles of Appendix~\ref{sec:appendix-synthetic-results}; single-proxy control is run once, as an ATE estimator with the clean block $W_1$.

The sealed thresholds of SCM-1 to SCM-4 do not transfer to these data sets, because the noise level of the screen statistic changes with the numbers of rows and covariates. The screen therefore uses a noise-scaled threshold: instead of requiring a pass in all four rotations, it pools them, taking their mean gap and its standard error, and it retains a split when $\max(\mathrm{gap},0)+z\,\mathrm{SE}$ is at most the median of $z\,\mathrm{SE}$ over the thirty splits of the same data set. This rule was adopted after the pre-registered RHC analysis and the RHC-planted, ACIC, and earlier IHDP analyses with the sealed threshold returned no estimate. It was frozen before Twins replicates $11$ to $20$ and IHDP replicates $12$ to $21$ were generated, so these two studies are prospective tests; it was chosen after seeing ACIC realizations $2$ to $10$, so the ACIC results are exploratory.

\paragraph{RHC study.}\label{sec:appendix-rhc} The RHC study \citep{connors1996effectiveness} has $5{,}735$ critically ill patients, days survived within $30$ as the outcome, and ten physiological measurements grouped into five proxy blocks, as analyzed by \citet{cui2024semiparametric}; no true effect is known. The blocks are, in order, the PaO$_2$/FiO$_2$ ratio and PaCO$_2$, pH and hematocrit, white cell count and albumin, bilirubin and creatinine, and sodium and potassium, and the comparators are the difference in means, the proximal estimate that \citet{cui2024semiparametric} obtain by assigning roles to $W_1$ and $W_2$, and KSPC given each block in turn. PROBE is not given this grouping and finds its own blocks from the ten measurements (Appendix~\ref{sec:appendix-block-discovery}). The interval around PROBE's RHC estimate in Fig.~\ref{fig:appendix-realdata} gives the sampling spread of the median of the largest agreeing group over $20{,}000$ draws, holding the retained splits and the linking radius at their observed values, so it omits the uncertainty of the screen. PROBE's interval contains the proximal estimate that \citet{cui2024semiparametric} obtain by assigning the roles, while KSPC's estimate changes with the block it is given (Fig.~\ref{fig:appendix-realdata}).

Table~\ref{tab:benchmarks} gives the MAE on the original replicates. On IHDP, ten replicates do not resolve the gap between PROBE and CEVAE: the error ratio is $1.67$, with a $90\%$ paired-bootstrap interval from $0.78$ to $2.78$.

\begin{table}[t]
\caption{Mean absolute error on the original replicates in three semi-synthetic studies with known effects; ACIC is exploratory. Proximal 2SLS, single-proxy control, and KSPC require proxy roles; KSPC ($\dagger$) was run after the protocol was frozen. Bold marks the smallest error among the methods other than the oracle.}
\label{tab:benchmarks}
\centering
\footnotesize
\begin{tabular}{lrrr}
\toprule
Method & Twins & IHDP & ACIC \\
\midrule
PROBE (ours) & \textbf{0.0077} & 0.3054 & 0.2025 \\
CEVAE & 0.0894 & \textbf{0.1828} & 0.2575 \\
Adjust for $X$ and all of $W$ & 0.0207 & 0.3927 & 0.2239 \\
Adjust for $X$ & 0.0749 & 0.7221 & 0.1757 \\
Proximal 2SLS, roles given & 0.0156 & 0.2328 & \textbf{0.1016} \\
Proximal 2SLS, roles swapped & 0.2453 & 1.5856 & 0.8015 \\
Single-proxy control, as an ATE & 0.7318 & 5.9838 & 34.6486 \\
KSPC$^\dagger$, clean block & 0.0099 & 0.6214 & 0.5622 \\
KSPC$^\dagger$, contaminated block & 0.0815 & 1.2499 & 0.9318 \\
\midrule
Oracle (adjusts for $(X,U)$) & 0.0074 & 0.4868 & 0.1336 \\
\bottomrule
\end{tabular}
\end{table}

\subsection{Protocol discipline}

One generated data set is one replicate; splits, rows, rotations, and restarts are not. Seeds overlap across the synthetic settings: SCM-1 replicates $10$ to $19$ share seeds with SCM-2 replicates $0$ to $9$, and SCM-2 replicates $10$ to $19$ with SCM-3 replicates $0$ to $9$, so comparisons across those settings are not independent. Paired methods share data, sample roles, and nuisance budgets. For SCM-1 to SCM-4 the screen threshold is fixed before the confirmation data are generated, and Appendix~\ref{sec:appendix-benchmarks} gives the rule and its history for the other data sets. The threshold remains a screen rather than proof of causal validity. No-return outcomes stay in the return-rate denominator, failed sealed runs stay in the record, and this appendix documents exclusions, replacements, and the representation family the search ranges over.

\subsection{Evidence boundaries}

 The empirical gates are performance criteria, not proofs of completeness, stability, or the population plurality condition.

\subsection{Additional figures}

Fig.~\ref{fig:appendix-aggregation} shows the three aggregation rules on the sealed confirmations, Fig.~\ref{fig:appendix-coverage} shows the interval coverage of one split, and Fig.~\ref{fig:appendix-learned} separates the representation and estimation parts of the error.

The interval in Fig.~\ref{fig:appendix-coverage} is the AIPW estimate of one split plus or minus $1.96$ times its own standard error, and that standard error is computed from the score cross-product the sealed run already stored, so the figure refits nothing. It covers what the learned representation targets at the nominal rate and covers $\tau$ less often, by the amount of the representation gap. An interval built from sampling variation cannot contain that gap, which is why Thm.~\ref{cor:representation-causal} bounds it separately from Thm.~\ref{thm:conditional-aipw-error}.

\begin{figure}[t]
\centering
\includegraphics[width=\linewidth]{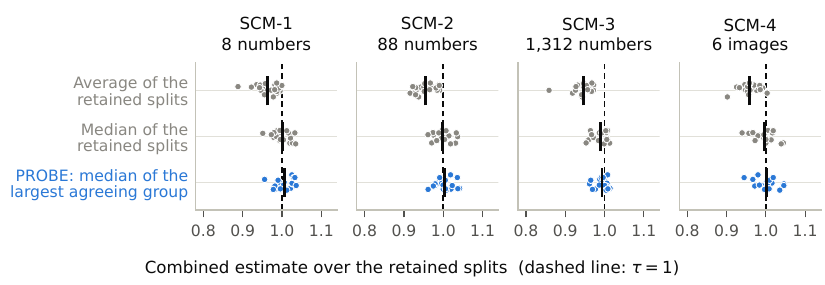}
\caption[Aggregation rules.]{Three ways to combine the retained splits. Averaging is pulled away from $\tau$ by the balanced but invalid split; the median of the largest agreeing group is not.}
\label{fig:appendix-aggregation}
\end{figure}

\begin{figure}[t]
\centering
\includegraphics[width=\linewidth]{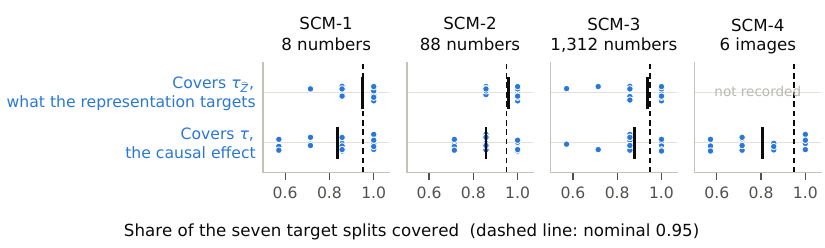}
\caption[Interval coverage.]{Coverage of a nominal $95\%$ interval around one split, one dot per data set. The interval covers the population effect of the learned representation at the nominal rate and covers $\tau$ less often, by the representation gap.}
\label{fig:appendix-coverage}
\end{figure}

\begin{figure}[t]
\centering
\includegraphics[width=\linewidth]{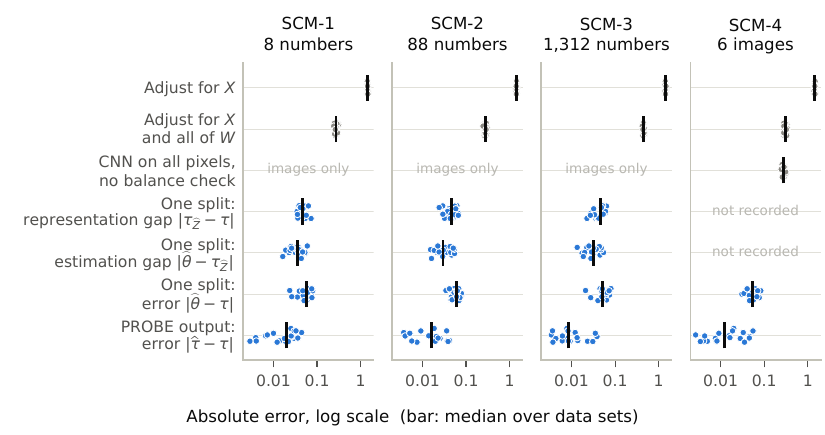}
\caption[Absolute errors of one split and of aggregation.]{Absolute errors on a log scale. For one split the representation gap $|\tau_{\widehat Z}-\tau|$ is the population effect of the learned representation, which is exact for the linear representations of SCM-1 to SCM-3 and is not recorded for SCM-4. Aggregation reduces the error of a single split by roughly a factor of three.}
\label{fig:appendix-learned}
\end{figure}

\FloatBarrier

\section{Proofs}
\label{app:proofs}

\subsection{Proofs for Section~\ref{sec:identification}}

\paragraph{Sufficient conditions for outcome-relevant completeness.}
\begin{proof}[Proof of Proposition~\ref{appdetail:prop:orc-sufficient}]
Fix $z$ and $Q,Q'\in\mathcal Q_{j,z}$ with $\mathcal K_jQ=\mathcal K_jQ'$. Write $\pi_Q$ for the distribution of $(X,U)$ under $Q$, $\pi_{Q,X}$ for its $X$-marginal, and $\pi_Q(\cdot\mid x)$ for its conditional distribution of $U$ given $X=x$. By condition (i), $K_j(\cdot\mid v,u)=K_j(\cdot\mid x,u)$, so the distribution $\mathcal K_jQ$ of $(X,W_j)$ has $X$-marginal $\pi_{Q,X}$ and conditional distribution $\int K_j(\cdot\mid x,u)\,\pi_Q(du\mid x)$ given $X=x$. The equality $\mathcal K_jQ=\mathcal K_jQ'$ therefore gives $\pi_{Q,X}=\pi_{Q',X}$ and, for $\pi_{Q,X}$-almost every $x$, $\int K_j(\cdot\mid x,u)\,\pi_Q(du\mid x)=\int K_j(\cdot\mid x,u)\,\pi_{Q'}(du\mid x)$. Condition (ii) yields $\pi_Q(\cdot\mid x)=\pi_{Q'}(\cdot\mid x)$ for almost every $x$, hence $\pi_Q=\pi_{Q'}$. By condition (i), $\mu_{a,j,z}(v,u)=\bar\mu_{a,z}(x,u)$ for a measurable function $\bar\mu_{a,z}$, so $\mathbb E_Q\{\mu_{a,j,z}(V_j,U)\}=\int\bar\mu_{a,z}\,d\pi_Q=\int\bar\mu_{a,z}\,d\pi_{Q'}=\mathbb E_{Q'}\{\mu_{a,j,z}(V_j,U)\}$, which is Assumption~\ref{ass:outcome-relevant-completeness}.

For case (a), the distribution of $W_j$ given $X=x$ is the convolution of the pushforward of $\pi$ under $g(x,\cdot)$ with the distribution of $\epsilon_j$. Its characteristic function is the product of the characteristic function of the pushforward and that of $\epsilon_j$, and the latter vanishes nowhere by assumption. Two distributions $\pi$ and $\pi'$ with equal convolutions therefore have pushforwards with equal characteristic functions, hence equal pushforwards by the uniqueness theorem for characteristic functions, and injectivity of $g(x,\cdot)$ gives $\pi=\pi'$. Thus (ii) holds.

For case (b), the probability generating function of $W_j$ given $X=x$ under a distribution $\pi$ of $U$ is $\mathbb E(\prod_\ell s_\ell^{W_{j,\ell}})=\int\exp\{-\sum_\ell\lambda_\ell(x,u)(1-s_\ell)\}\,\pi(du)$ for $s\in[0,1]^{d}$, which is the Laplace transform of the pushforward of $\pi$ under $\lambda(x,\cdot)$ evaluated on $[0,1]^{d}$. A finite measure on $[0,\infty)^{d}$ is determined by its Laplace transform on a set with nonempty interior, so two distributions with equal mixtures have equal pushforwards, and injectivity of $\lambda(x,\cdot)$ gives $\pi=\pi'$. In the scalar case this is the identifiability of Poisson mixtures of \citet{teicher1961identifiability}. Thus (ii) holds.

\end{proof}

\paragraph{Identification from a held-out audit.}
\begin{proof}[Proof of Theorem~\ref{thm:exact-proxy-identification}]
Because $Z_j=\phi_j(V_j)$ is a measurable function of $V_j$, conditioning on $(U,V_j,Z_j)$ is equivalent to conditioning on $(U,V_j)$. The two clauses of Assumption~\ref{ass:joint-latent-exchangeability} therefore remain valid after additionally conditioning on $Z_j$. Thus, for each $a\in\{0,1\}$,
\begin{align}
\mathbb E\{Y(a)\mid A,V_j,U,Z_j\}
&=
\mathbb E\{Y(a)\mid V_j,U,Z_j\},
\label{eq:conditioning-safe-outcome-app}\\
 P(W_j\mid A,V_j,U,Z_j)
&=
K_j(\cdot\mid V_j,U).
\nonumber
\end{align}
Integrating the second identity over $Q_{b,j,z}$ and using the definition of $\mathcal K_j$ gives
\( P(X,W_j\mid A=b,Z_j=z) = \mathcal K_jQ_{b,j,z}.\)
Exact balance in Eq.~\eqref{eq:exact-proxy-balance} therefore implies
\(\mathcal K_jQ_{1,j,z}=\mathcal K_jQ_{0,j,z}\)
for $P_{Z_j}$-almost every $z$. Representation overlap gives $Q_{b,j,z}\ll P(V_j,U\mid Z_j=z)$ for $b\in\{0,1\}$. Assumption~\ref{ass:consistency}, Eq.~\eqref{eq:conditioning-safe-outcome-app}, and Jensen's inequality give the required integrability, so $Q_{b,j,z}\in\mathcal Q_{j,z}$ almost everywhere. Assumption~\ref{ass:outcome-relevant-completeness} therefore gives, for each $a\in\{0,1\}$,
\begin{align}
\mathbb E_{Q_{1,j,z}}\{\mu_{a,j,z}(V_j,U)\}
=
\mathbb E_{Q_{0,j,z}}\{\mu_{a,j,z}(V_j,U)\}.
\label{eq:orc-equal-latent-means}
\end{align}
By Eq.~\eqref{eq:conditioning-safe-outcome-app} and iterated expectation, for each $b\in\{0,1\}$,
\(\mathbb E\{Y(a)\mid A=b,Z_j=z\} = \mathbb E_{Q_{b,j,z}}\{\mu_{a,j,z}(V_j,U)\}.\)
Eq.~\eqref{eq:orc-equal-latent-means} shows that the conditional mean of $Y(a)$ given $(A,Z_j)$ does not depend on $A$. Consistency then gives
\(\mathbb E(Y\mid A=a,Z_j=z) = \mathbb E\{Y(a)\mid A=a,Z_j=z\} = \mathbb E\{Y(a)\mid Z_j=z\}.\)
Integrating over $P_{Z_j}$ yields Eq.~\eqref{eq:exact-identification}.\qedhere
\end{proof}

\paragraph{Non-nesting of canonical proximal and proximal-balancing sufficient conditions.}
\begin{proof}[Proof of Proposition~\ref{prop:non-nesting-proximal-probe}]
We give one finite-state distribution for each direction. In both constructions, $X$ is constant, $Y=Y(A)$, and every treatment draw is generated by an exogenous randomizer independent of all proxy noises and outcome noises.

First, we construct a distribution satisfying the canonical proximal conditions but not the proximal-balancing conditions for the designated proxy split. Let
\(U\sim\operatorname{Bernoulli}(1/2), P(A=1\mid U=u)=\frac14+\frac12u,\)
and let
\(Z^{\mathrm{tr}}=U\mathbin\oplus E_Z, \qquad Z^{\mathrm{out}}=W_j=U\mathbin\oplus E_W,\)
where $E_Z$ and $E_W$ are mutually independent $\operatorname{Bernoulli}(1/4)$ variables, independent of the treatment randomizer. Set $Y(a)=a+U$ and $V_j=Z^{\mathrm{tr}}$. Consistency, integrability, and strict latent overlap hold. Independence of the exogenous noises gives the three canonical proxy-validity conditions.

The function
\(h(z^{\mathrm{out}},a)=a+2z^{\mathrm{out}}-\frac12\)
is an outcome bridge because $\mathbb E(Z^{\mathrm{out}}\mid U=u)=1/4+u/2$, so $\mathbb E\{h(Z^{\mathrm{out}},a)\mid U=u\}=a+u$. Iterated expectation gives the required observed bridge equation. The binary channel from $U$ to $Z^{\mathrm{tr}}$ has matrix
\(\begin{pmatrix} 3/4&1/4\\ 1/4&3/4 \end{pmatrix}\)
and determinant $1/2$. Conditional on either treatment arm, both latent states have positive probability, and the corresponding unnormalized conditional-expectation matrix has determinant $3/32$. It therefore has full column rank, which proves the canonical treatment-proxy completeness condition.

Now consider any measurable representation $Z_j=\phi_j(V_j)$. The Markov relation $U\to V_j\to Z_j$ holds. Conditional on $(U,Z_j)$, treatment and $W_j$ are independent, and
\(\mathbb E(A\mid U,Z_j) = \mathbb E(W_j\mid U,Z_j) = \frac14+\frac12U.\)
The conditional covariance decomposition gives
\(\operatorname{Cov}(A,W_j\mid Z_j) = \frac14\operatorname{Var}(U\mid Z_j).\)
Because $P(U=1\mid V_j=0)=1/4$ and $P(U=1\mid V_j=1)=3/4$,
\(P(U=1\mid Z_j) = \mathbb E\{P(U=1\mid V_j)\mid Z_j\} \in[1/4,3/4]\)
almost surely. Hence $\operatorname{Var}(U\mid Z_j)\ge3/16$ and
\(\operatorname{Cov}(A,W_j\mid Z_j)\ge\frac{3}{64}>0.\)
Exact balance therefore fails for every such representation. This proves the first direction.

Second, we construct a distribution satisfying the proximal-balancing conditions for the designated proxy split but not the canonical proximal conditions. Let $U=(U_1,U_2)$, where $U_1$ and $U_2$ are independent $\operatorname{Bernoulli}(1/2)$ variables. Let $Z^{\mathrm{tr}}$ and $Z^{\mathrm{out}}$ be separate observed coordinates that agree almost surely with $U_1$, and set
\(V_j=Z^{\mathrm{tr}}=U_1, \qquad W_j=Z^{\mathrm{out}}=U_1, \qquad Z_j=U_1.\)
Define
\(P(A=1\mid U_1,U_2) = \frac3{20}+\frac15U_1+\frac25U_2,\)
and let $\xi$ be an independent mean-zero Rademacher variable, also independent of the treatment randomizer. Set
\(Y(a)=a+U_1+(1+U_2)\xi.\)
The four latent treatment probabilities are $3/20$, $11/20$, $7/20$, and $15/20$, so strict latent overlap holds. Latent exchangeability follows from the independent outcome noise, and held-out treatment independence holds because $W_j$ is deterministic given $U$. Moreover,
\(P(A=1\mid Z_j=0)=\frac7{20}, \qquad P(A=1\mid Z_j=1)=\frac{11}{20},\)
so representation overlap holds. Because $W_j=Z_j$ and $X$ is constant, exact balance also holds.

It remains to verify the all-$Q$ outcome-relevant completeness condition. Within a stratum $Z_j=z$, every admissible distribution $Q\in\mathcal Q_{j,z}$ is supported on $V_j=U_1=z$, while $U_2$ may vary. On this support,
\(\mu_{a,j,z}(v,u)=a+z.\)
Thus, for any admissible $Q,Q'$ and either potential treatment level,
\(\mathcal K_jQ=\mathcal K_jQ' \quad\Longrightarrow\quad \mathbb E_Q\{\mu_{a,j,z}(V_j,U)\} =a+z =\mathbb E_{Q'}\{\mu_{a,j,z}(V_j,U)\}.\)
This proves all proximal-balancing conditions for the designated proxy split.

The outcome bridge $h(z^{\mathrm{out}},a)=a+z^{\mathrm{out}}$ exists because
\(\mathbb E(Y\mid Z^{\mathrm{tr}}=z,A=a) =a+z =\mathbb E\{h(Z^{\mathrm{out}},a)\mid Z^{\mathrm{tr}}=z,A=a\}.\)
Canonical treatment-proxy completeness nevertheless fails. In the treated arm,
\(P(U_2=1\mid U_1=0,A=1)=\frac{11}{14}, \qquad P(U_2=1\mid U_1=1,A=1)=\frac{15}{22}.\)
Define the square-integrable, nonzero function
\(g_1(U) = U_2- \begin{cases} 11/14,&U_1=0,\\ 15/22,&U_1=1. \end{cases}\)
Then
\(\mathbb E\{g_1(U)\mid Z^{\mathrm{tr}},A=1\} = \mathbb E\{g_1(U)\mid U_1,A=1\} =0\)
almost surely, although $g_1(U)$ is not almost surely zero. Hence the canonical completeness condition fails, proving the second direction.

For the final claim, the second construction satisfies Assumption~\ref{ass:outcome-relevant-completeness} but not canonical treatment-proxy completeness. In the first construction, replace $E_W$ by an independent $\operatorname{Bernoulli}(1/2)$ variable. Canonical treatment-proxy completeness is unchanged because it involves only $(U,Z^{\mathrm{tr}},A)$. Now $W_j$ is independent of $U$, so $\mathcal K_jQ$ is the same distribution for every $Q$. For every representation $\phi_j$ and almost every $z$, the bound $P(U=1\mid Z_j)\in[1/4,3/4]$ shown above gives both values of $U$ positive probability given $Z_j=z$, so $\mathbb E_Q\mu_{a,j,z}=a+Q(U=1)$ takes more than one value over $\mathcal Q_{j,z}$ and Assumption~\ref{ass:outcome-relevant-completeness} fails.
\end{proof}

\paragraph{Identification with an unknown valid block.}
\begin{proof}[Proof of Theorem~\ref{thm:unknown-valid-block-identification}]
Under condition~\eqref{eq:unique-target-plurality}, the number of proxy splits in $\mathcal B$ whose candidate value equals $\tau$ is strictly larger than the count for every non-target value. Therefore, $\tau$ is the unique value maximizing the population count:
\(\{\tau\} = \underset{c\in\mathbb R}{\operatorname{arg\,max}} \bigl|\{S\in\mathcal B:\theta_S=c\}\bigr|.\)
If no non-target value occurs, the convention in Definition~\ref{def:population-plurality} compares the nonempty target cluster with zero and gives the same conclusion.
\end{proof}

\paragraph{Target-valued majority.}
\begin{proof}[Proof of Corollary~\ref{cor:target-majority-median}]
Define
\(N_\tau \triangleq \bigl|\{S\in\mathcal B:\theta_S=\tau\}\bigr|\) and \(N_c \triangleq \bigl|\{S\in\mathcal B:\theta_S=c\}\bigr|, \qquad c\ne\tau.\)
Condition~\eqref{eq:target-valued-majority} gives $N_\tau>|\mathcal B|/2$. For every $c\ne\tau$,
\(N_c \le |\mathcal B|-N_\tau < N_\tau.\)
Therefore, $\tau$ is the unique candidate value of greatest multiplicity.

It remains to establish the median statement. Let $n=|\mathcal B|$. If $n=2k+1$ is odd, then $N_\tau\ge k+1$. At most $k$ candidate values are strictly below $\tau$, and at most $k$ are strictly above $\tau$, so the $(k+1)$st order statistic equals $\tau$. If $n=2k$ is even, then $N_\tau\ge k+1$ and $n-N_\tau\le k-1$. Hence both the $k$th and the $(k+1)$st order statistics equal $\tau$. Under the conventional definition of the even-sample median as the average of these two middle order statistics, the median equals $\tau$. Thus,
\(\operatorname{median}\bigl((\theta_S)_{S\in\mathcal B}\bigr)=\tau.\)
The stronger condition that more than half of the proxy splits in $\mathcal B$ satisfy the grouped assumptions is sufficient because each such proxy split has candidate value $\tau$, but it is not necessary.
\end{proof}

\paragraph{Zero discrepancy.}
\begin{proof}[Proof of Proposition~\ref{prop:zero-discrepancy}]
By Definition~\ref{def:held-out-proxy-discrepancy}, $D_{\mathrm{res},j}(\phi)=0$ if and only if
\(q_{j,\phi}(T_j,Z_j^\phi)=e_{j,\phi}(Z_j^\phi) \quad\text{almost surely}.\)
Because $A$ is binary, this equality is equivalent to $T_j\perp\!\!\!\perp A\mid Z_j^\phi$.
\end{proof}

\paragraph{Supporting bias decomposition.}
Fix a representation $\phi$ and abbreviate $e(z)=e_{j,\phi}(z)$. For each $a\in\{0,1\}$ and $P_{Z_j^\phi}$-almost every $z$, define
\(\delta_{a,j}^{\phi}(z) \triangleq \mathbb E_{Q_{1,j,z}^{\phi}}\{\mu_{a,j,z}^{\phi}(V_j,U)\} - \mathbb E_{Q_{0,j,z}^{\phi}}\{\mu_{a,j,z}^{\phi}(V_j,U)\}.\)
Under Assumptions~\ref{ass:consistency} and \ref{ass:joint-latent-exchangeability}, and under $0<e(Z_j^\phi)<1$ almost surely, the conditioning argument used in Eq.~\eqref{eq:conditioning-safe-outcome-app} gives
\(\mathbb E(Y\mid A=a,Z_j^\phi=z) = \mathbb E_{Q_{a,j,z}^{\phi}}\{\mu_{a,j,z}^{\phi}(V_j,U)\}.\)
Conditional expectation over $A$ gives
\(\mathbb E\{Y(a)\mid Z_j^\phi=z\} ={} e(z)\mathbb E_{Q_{1,j,z}^{\phi}}\{\mu_{a,j,z}^{\phi}\}+\{1-e(z)\}\mathbb E_{Q_{0,j,z}^{\phi}}\{\mu_{a,j,z}^{\phi}\}.\)
Subtracting the two conditional expressions and collecting terms gives
\begin{align}
&\mathbb E(Y\mid A=1,Z_j^\phi=z)
-\mathbb E(Y\mid A=0,Z_j^\phi=z)
\nonumber\\
&\quad-
\left[
\mathbb E\{Y(1)\mid Z_j^\phi=z\}
-
\mathbb E\{Y(0)\mid Z_j^\phi=z\}
\right]
\nonumber\\
&\qquad=
\{1-e(z)\}\delta_{1,j}^{\phi}(z)
+e(z)\delta_{0,j}^{\phi}(z).
\label{eq:pointwise-bias-decomposition}
\end{align}

\paragraph{Approximate-balance bias.}
\begin{proof}[Proof of Theorem~\ref{thm:approximate-proxy-bias-bound}]
Fix $z$ in the almost-sure set on which overlap and the structural conditions hold. Bayes' rule gives
\(\frac{d P(T_j\mid A=1,Z_j^\phi=z)}{dP_{\phi,z}}(t) =\frac{q_{j,\phi}(t,z)}{e(z)}\) and \(\frac{d P(T_j\mid A=0,Z_j^\phi=z)}{dP_{\phi,z}}(t) =\frac{1-q_{j,\phi}(t,z)}{1-e(z)}.\)
Subtracting gives the density identity
\(\frac{d\{ P(T_j\mid A=1,Z_j^\phi=z) - P(T_j\mid A=0,Z_j^\phi=z)\}} {dP_{\phi,z}}(t) = \frac{q_{j,\phi}(t,z)-e(z)} {e(z)\{1-e(z)\}}.\)
Held-out treatment independence makes the conditional distribution of $W_j$ given $(U,V_j)$ treatment-invariant, so $\mathcal K_jQ_{b,j,z}^{\phi}= P(T_j\mid A=b,Z_j^\phi=z)$ for $b\in\{0,1\}$. As in the proof of Thm.~\ref{thm:exact-proxy-identification}, overlap and Assumption~\ref{ass:consistency} give $Q^{\phi}_{b,j,z}\in\mathcal Q^{\phi}_{j,z}$ for $b\in\{0,1\}$. Assumption~\ref{ass:uniform-stability} therefore implies
\begin{align}
|\delta_{a,j}^{\phi}(z)|
&\le
\frac{\Gamma_j(\phi)}{e(z)\{1-e(z)\}}
\left[
\mathbb E\!\left(
\left.
\{q_{j,\phi}(T_j,z)-e(z)\}^2
\right|Z_j^\phi=z
\right)
\right]^{1/2}.
\label{eq:stratum-bias-stability-bound}
\end{align}
Apply the triangle inequality to Eq.~\eqref{eq:pointwise-bias-decomposition}, use Eq.~\eqref{eq:stratum-bias-stability-bound} for both potential outcomes, and integrate over $Z_j^\phi$. This gives
\(\left| \mathbb E\!\left[ \mathbb E(Y\mid A=1,Z_j^\phi) - \mathbb E(Y\mid A=0,Z_j^\phi) \right] -\tau \right| \le \mathbb E\!\left[ \frac{\Gamma_j(\phi)} {e_{j,\phi}(Z_j^\phi)\{1-e_{j,\phi}(Z_j^\phi)\}} \left\{ \mathbb E\!\left[ \left. \{q_{j,\phi}(T_j,Z_j^\phi)-e_{j,\phi}(Z_j^\phi)\}^2 \right|Z_j^\phi \right] \right\}^{1/2} \right] \le \frac{\Gamma_j(\phi)}{\eta(1-\eta)}D_{\mathrm{res},j}(\phi).\)
Because $\Gamma_j(\phi)$ does not depend on $z$, the overlap bound $\eta\le e_{j,\phi}(z)\le1-\eta$ gives the last inequality, together with $\mathbb E\{s(Z_j^\phi)\}\le[\mathbb E\{s(Z_j^\phi)^2\}]^{1/2}=D_{\mathrm{res},j}(\phi)$ by Cauchy--Schwarz, where $s(z)$ denotes the conditional root-mean-square factor in the display, and proves Eq.~\eqref{eq:uniform-approximate-proxy-bias-bound}.
\end{proof}

\paragraph{Population sensitivity region.}
\begin{proof}[Proof of Proposition~\ref{prop:population-sensitivity-region}]
Thm.~\ref{thm:approximate-proxy-bias-bound} gives $|\tau_{Z_j^\phi}-\tau|\le\Gamma_j(\phi)D_{\mathrm{res},j}(\phi)/\{\eta(1-\eta)\}$, and the displayed interval is the set of values of $\tau$ compatible with this inequality.
\end{proof}

\subsection{Proofs for Section~\ref{sec:estimation}}

\paragraph{Uniform discrepancy deviation.}
\begin{proof}[Proof of Proposition~\ref{appdetail:prop:uniform-discrepancy-deviation}]
For clarity, the joint encoder--critic loss classes used in the proposition are
\(\mathcal L_{0,j,n_{\mathrm D}} \triangleq \left\{ o\mapsto\{a-h(\phi(v_j))\}^2: \phi\in\mathcal F_{j,n_{\mathrm D}},\ h\in\mathcal H_{0,n_{\mathrm D}} \right\}\); \(\mathcal L_{1,j,n_{\mathrm D}} \triangleq \left\{ o\mapsto\{a-h(t_j,\phi(v_j))\}^2: \phi\in\mathcal F_{j,n_{\mathrm D}},\ h\in\mathcal H_{1,n_{\mathrm D}} \right\},\)
where $o=(v_j,t_j,a,y)$. For a loss class $\mathcal L$, its expected Rademacher complexity is
\(\mathfrak R_{n_{\mathrm D}}(\mathcal L) \triangleq \mathbb E\!\left[ \sup_{\ell\in\mathcal L} \frac1{n_{\mathrm D}}\sum_{i=1}^{n_{\mathrm D}}\sigma_i\ell(O_i) \right],\)
where the $\sigma_i$ are independent Rademacher signs independent of the observations. The loss classes are pointwise measurable and separable, so these suprema are measurable.
The same argument may instead be stated with outer probability and outer expectation.

Every loss in either class lies in $[0,1]$. Standard symmetrization and bounded-difference concentration therefore imply, simultaneously for $b\in\{0,1\}$, that
\begin{align}
\sup_{\phi\in\mathcal F_{j,n_{\mathrm D}},\,h\in\mathcal H_{b,n_{\mathrm D}}}
|\widehat R_{b,\phi}(h)-R_{b,\phi}(h)|
\le
2\mathfrak R_{n_{\mathrm D}}(\mathcal L_{b,j,n_{\mathrm D}})
+\sqrt{\frac{\log(4/\delta_{\mathrm D})}{2n_{\mathrm D}}}
\label{eq:brier-risk-uniform-deviation-app}
\end{align}
with probability at least $1-\delta_{\mathrm D}$. Let $R_{b,\phi}^{\star}=\inf_{h\in\mathcal H_{b,n_{\mathrm D}}}R_{b,\phi}(h)$ and define $\widehat R_{b,\phi}^{\star}$ analogously. Eq.~\eqref{eq:brier-risk-uniform-deviation-app} controls $|\widehat R_{b,\phi}^{\star}-R_{b,\phi}^{\star}|$ by the same right-hand side. Eq.~\eqref{eq:approximate-critic-fits} also gives
\(0 \le \widehat R_{b,\phi}(\widehat h_{b,\phi}) -\widehat R_{b,\phi}^{\star} \le \varepsilon_{b,\phi}.\)

The conditional-mean property of squared loss yields the Brier identity
\(\mathbb E[\{A-e_{j,\phi}(Z_j^\phi)\}^2] -\mathbb E[\{A-q_{j,\phi}(T_j,Z_j^\phi)\}^2] = \mathbb E[\{q_{j,\phi}(T_j,Z_j^\phi)-e_{j,\phi}(Z_j^\phi)\}^2] =D_{\mathrm{res},j}^2(\phi).\)
It follows from Eq.~\eqref{eq:critic-approximation-error} that
\begin{align}
\left|
R_{0,\phi}^{\star}-R_{1,\phi}^{\star}
-D_{\mathrm{res},j}^2(\phi)
\right|
\le
a_{0,n_{\mathrm D}}(\phi)+a_{1,n_{\mathrm D}}(\phi)
\le a_{j,n_{\mathrm D}}.
\label{eq:population-critic-gap-app}
\end{align}
Combine Eqs.~\eqref{eq:brier-risk-uniform-deviation-app}--\eqref{eq:population-critic-gap-app}. The positive-part map is one-Lipschitz and $D_{\mathrm{res},j}^2(\phi)\ge0$, so taking the supremum over $\phi$ gives Eq.~\eqref{eq:uniform-discrepancy-deviation}.
\end{proof}

\paragraph{Representation-learning oracle inequality.}
\begin{proof}[Proof of Proposition~\ref{appdetail:thm:representation-oracle}]
On the event in Eq.~\eqref{eq:uniform-discrepancy-deviation}, set
\(\bar\xi_{j,n_{\mathrm D}} \triangleq r_{j,n_{\mathrm D}}(\delta_{\mathrm D})+a_{j,n_{\mathrm D}}+\varepsilon_{\mathrm{crit},j}.\)
The deviation inequality at $\widehat\phi_j$, Eq.~\eqref{eq:approximate-erm}, and the deviation inequality at a near-minimizer of $D_{\mathrm{res},j}^2$ over $\mathcal F_{j,n_{\mathrm D}}$ give
\(D_{\mathrm{res},j}^2(\widehat\phi_j) \le \widetilde D_{j,n_{\mathrm D}}^2(\widehat\phi_j)+\bar\xi_{j,n_{\mathrm D}} \le \inf_{\phi\in\mathcal F_{j,n_{\mathrm D}}} \widetilde D_{j,n_{\mathrm D}}^2(\phi) +\varepsilon_{\mathrm{enc},j}+\bar\xi_{j,n_{\mathrm D}} \le \inf_{\phi\in\mathcal F_{j,n_{\mathrm D}}} D_{\mathrm{res},j}^2(\phi) +2\bar\xi_{j,n_{\mathrm D}}+\varepsilon_{\mathrm{enc},j}\),
which proves Eq.~\eqref{eq:representation-oracle-discrepancy}.
\end{proof}

\paragraph{Learned-representation error.}
\begin{proof}[Proof of Theorem~\ref{cor:representation-causal}]
Conditional on the training observations $\{O_i:i\in\mathcal I_{\mathrm D}\}$ and any algorithmic seed used in training, the selected encoder is fixed; the population expectations below use a fresh observation independent of these training inputs. Eq.~\eqref{eq:uniform-approximate-proxy-bias-bound} gives
\(|\tau_{\widehat Z_j}-\tau| \le \frac{\Gamma_j(\widehat\phi_j)}{\eta(1-\eta)} D_{\mathrm{res},j}(\widehat\phi_j).\)
On the event in Eq.~\eqref{eq:uniform-discrepancy-radius}, which has probability at least $1-\delta_{\mathrm D}$, every $\phi\in\mathcal F_{j,n_{\mathrm D}}$ satisfies \(D_{\mathrm{res},j}^2(\widehat\phi_j)\le\widetilde D_{j,n_{\mathrm D}}^2(\widehat\phi_j)+\xi_{j,n_{\mathrm D}}(\delta_{\mathrm D})\le\widetilde D_{j,n_{\mathrm D}}^2(\phi)+\xi_{j,n_{\mathrm D}}(\delta_{\mathrm D})\le D_{\mathrm{res},j}^2(\phi)+2\xi_{j,n_{\mathrm D}}(\delta_{\mathrm D})\). The first and last inequalities hold on this event, and the middle one holds because $\widehat\phi_j$ minimizes $\widetilde D_{j,n_{\mathrm D}}^2$ over $\mathcal F_{j,n_{\mathrm D}}$. Taking the infimum over $\phi$ and then the square root gives $D_{\mathrm{res},j}(\widehat\phi_j)\le[\inf_{\phi\in\mathcal F_{j,n_{\mathrm D}}}D_{\mathrm{res},j}^2(\phi)+2\xi_{j,n_{\mathrm D}}(\delta_{\mathrm D})]^{1/2}$. Substituting this bound into the first display proves Eq.~\eqref{eq:representation-oracle-causal}. Substituting the square root of Eq.~\eqref{eq:representation-oracle-discrepancy} instead proves Thm.~\ref{appdetail:cor:representation-causal}.
\end{proof}

\paragraph{Exact signed nuisance remainder.}
\begin{proof}[Proof of Proposition~\ref{appdetail:prop:aipw-remainder}]
Fix $\widehat Z_j=z$. Taking the conditional expectation of the score and using Eq.~\eqref{eq:aipw-nuisances} gives
\(\mathbb E\{\mathtt{UIF}_j(O)\mid\widehat Z_j=z\} ={} \widehat m_{1,j}(z)-\widehat m_{0,j}(z)+ \frac{e_j(z)\{m_{1,j}(z)-\widehat m_{1,j}(z)\}} {\widehat e_j(z)}- \frac{\{1-e_j(z)\}\{m_{0,j}(z)-\widehat m_{0,j}(z)\}} {1-\widehat e_j(z)}.\)
Subtract $m_{1,j}(z)-m_{0,j}(z)$ and collect the two nuisance errors. Averaging the resulting identity over $\widehat Z_j$ proves Eq.~\eqref{eq:aipw-remainder}.
\end{proof}

\paragraph{Honest AIPW error.}
\begin{proof}[Proof of Theorem~\ref{thm:conditional-aipw-error}]
 Conditional on the two learning samples, the evaluation observations are independent and the score $\mathtt{UIF}_j(O)$ has variance $\sigma_j^2$, so Chebyshev's inequality bounds the difference between the evaluation average and its conditional expectation by $\sigma_j/\sqrt{n_{\mathrm E}\delta_{\mathrm E}}$, the first term in Eq.~\eqref{eq:aipw-error-radius}, with probability at least $1-\delta_{\mathrm E}$. Prop.~\ref{appdetail:prop:aipw-remainder}, Cauchy--Schwarz, and the propensity clipping bound control the absolute conditional remainder by
\(\frac{q_{e,j}(q_{0,j}+q_{1,j})}{\eta}.\)
The triangle inequality proves Eq.~\eqref{eq:conditional-aipw-error}.
For Eq.~\eqref{eq:end-to-end-sensitivity}, the triangle inequality gives $|\widehat\tau_j^{\mathrm{AIPW}}-\tau|\le|\widehat\tau_j^{\mathrm{AIPW}}-\tau_{\widehat Z_j}|+|\tau_{\widehat Z_j}-\tau|$. Eq.~\eqref{eq:conditional-aipw-error} bounds the first term on the same event, and Eq.~\eqref{eq:uniform-approximate-proxy-bias-bound} bounds the second term by $\Gamma_j(\widehat\phi_j)D_{\mathrm{res},j}(\widehat\phi_j)/\{\eta(1-\eta)\}$.
\end{proof}

\paragraph{Consistency.}
\begin{proof}[Proof of Corollary~\ref{cor:learned-aipw-consistency}]
 Condition (C1) makes the evaluation term $\sigma_{j,n}/\sqrt{n_{\mathrm E}\delta_{\mathrm E,n}}$ of $\varepsilon_{\mathrm{AIPW},j}$ vanish, condition (C2) makes its nuisance-product term vanish, and condition (C3) makes the stability-weighted residual-discrepancy term in Eq.~\eqref{eq:end-to-end-sensitivity} vanish. The failure probability $\delta_{\mathrm E,n}$ of Thm.~\ref{thm:conditional-aipw-error} also converges to zero, which proves convergence in probability.
\end{proof}

\paragraph{Cluster plurality under random subsampling.}
\begin{proof}[ Proof of Lemma~\ref{lem:random-subsampling-plurality}]
Conditional on $\mathcal G$ and $R=r\ge1$, the retained proxy splits form a uniformly sampled $r$-element subset of $\widehat{\mathcal B}$ because the $m$ draws are uniform without replacement and $\widehat{\mathcal B}$ is fixed. For each competing cell $C_\ell$, define $H_\ell(S)\triangleq\mathbf 1\{S\in C_0\}-\mathbf 1\{S\in C_\ell\}$ for $S\in\widehat{\mathcal B}$. This variable lies in $[-1,1]$ and has population mean $(n_0-n_\ell)/N\ge\Delta_\Pi$. Let $\overline H_{\ell,r}$ be its retained-sample average. The event $R_\ell\ge R_0$ implies $\overline H_{\ell,r}\le0$. Hoeffding's inequality for sampling without replacement \citep{hoeffding1963probability} gives $P(\overline H_{\ell,r}\le0\mid\mathcal G,R=r)\le e^{-r\Delta_\Pi^2/2}$. A union bound over the $L$ competing cells yields $P(\max_{1\le\ell\le L}R_\ell\ge R_0\mid\mathcal G,R=r)\le L e^{-r\Delta_\Pi^2/2}$. Averaging over $R$ and counting $R=0$ as failure gives $P(\text{failure}\mid\mathcal G)\le P(R=0\mid\mathcal G)+L\,\mathbb E\{e^{-R\Delta_\Pi^2/2}\mid\mathcal G\}$. If $L=0$, every retained split belongs to $C_0$, so the conclusion holds whenever $R\ge1$.
\end{proof}

\paragraph{Finite-sample PROBE cluster recovery.}
\begin{proof}[Proof of Theorem~\ref{thm:probe-search-error}]
Work conditionally on $\mathcal G$. Let $E_1$ be the uniform-evaluation-error event and let $E_2$ be the event that $R\ge1$ and $R_0>\max_{1\le\ell\le L}R_\ell$. Eq.~\eqref{eq:conditional-aipw-error} applied to each $S\in\widehat{\mathcal B}$ at $\delta_{\mathrm E}=\delta/N$ gives $P(|\widehat\theta_S-\theta_S|>\varepsilon\mid\mathcal G)\le\delta/N$, because $\varepsilon$ is at least the radius of Eq.~\eqref{eq:aipw-error-radius} for $S$ at that level, and a union bound over the $N$ splits in $\widehat{\mathcal B}$ gives $P(E_1^{\mathrm c}\mid\mathcal G)\le\delta$, while Lemma~\ref{lem:random-subsampling-plurality} gives $P(E_2^{\mathrm c}\mid\mathcal G)\le P(R=0\mid\mathcal G)+L\,\mathbb E\{e^{-R\Delta_\Pi^2/2}\mid\mathcal G\}$. For $S,S'\in C_0$, Eq.~\eqref{eq:uniform-approximate-proxy-bias-bound} for $S$ with encoder $\widehat\phi_S$ and Assumption~\ref{ass:library-overlap} give $|\theta_S-\tau|\le\mathtt{bias}_S\le b$, and likewise $|\theta_{S'}-\tau|\le b$. Hence $|\theta_S-\theta_{S'}|\le2b\le2(\rho-\varepsilon)$ by condition (i), so the within-cluster bound of condition (ii) also holds on $C_0$. On $E_1$, two retained splits in the same cell satisfy $|\widehat\theta_S-\widehat\theta_{S'}|\le2(\rho-\varepsilon)+2\varepsilon=2\rho$, so they are linked. Two retained splits in different cells satisfy $|\widehat\theta_S-\widehat\theta_{S'}|>2(\rho+\varepsilon)-2\varepsilon=2\rho$, so they are not linked. Hence the connected components are exactly the nonempty sets $\mathcal R_m\cap C_\ell$. On $E_2$, $\mathcal R_m\cap C_0$ is the unique largest component, and Algorithm~\ref{alg:probe-search} returns the median of its estimates. For every $S\in C_0$, $|\widehat\theta_S-\tau|\le|\widehat\theta_S-\theta_S|+|\theta_S-\tau|\le\varepsilon+b$, where $|\theta_S-\tau|\le b$ by Eq.~\eqref{eq:uniform-approximate-proxy-bias-bound}, because every $S\in C_0$ satisfies Assumptions~\ref{ass:consistency}, \ref{ass:joint-latent-exchangeability}, and~\ref{ass:uniform-stability}. The median therefore lies in $[\tau-(b+\varepsilon),\tau+(b+\varepsilon)]$. A union bound over $E_1^{\mathrm c}$ and $E_2^{\mathrm c}$ proves Eq.~\eqref{eq:probe-search-error-probability}.
\end{proof}

\begin{proof}[Proof of Corollary~\ref{cor:budget-rule}]
Write $c\triangleq\Delta_\Pi^2/2$. First, $P(R=0)=\prod_{i=0}^{m-1}(T-N-i)/(T-i)\le(1-p)^m\le e^{-mp}$. Second, $R$ is the sum of the indicators $\mathbf 1\{S\in\widehat{\mathcal B}\}$ over the $m$ splits drawn without replacement. Because $x\mapsto e^{-cx}$ is continuous and convex, Thm.~4 of \citet{hoeffding1963probability} bounds $\mathbb E\{e^{-cR}\}$ by the corresponding expectation under sampling with replacement. Thus $\mathbb E\{e^{-cR}\}\le\{1-p(1-e^{-c})\}^m\le\exp\{-mp(1-e^{-c})\}$. Since $1-e^{-c}\le1$, adding the bound for $P(R=0)$ gives Eq.~\eqref{eq:sampling-failure-closed-form}. Its right side is at most $\alpha$ when $m\ge\log\{(L+1)/\alpha\}/\{p(1-e^{-c})\}$, and Thm.~\ref{appdetail:thm:probe-search-error} gives $P(|\widehat\tau-\tau|\le b+\varepsilon)\ge1-\delta-\alpha$. Finally, $1-e^{-c}\ge c/(1+c)\ge\Delta_\Pi^2/3$ because $0<\Delta_\Pi\le1$. If $L=0$, only the empty-draw bound is needed, and $m\ge\log(1/\alpha)/p$ suffices.
\end{proof}

\paragraph{Median under majority.}
\begin{proof}[Proof of Corollary~\ref{cor:median-under-majority}]
On the target-band and uniform-evaluation-error event, every retained split in $C_0$ satisfies $|\widehat\theta_S-\tau|\le b+\varepsilon$. If $|\mathcal R_m\cap C_0|>R/2$, more than half of all retained estimates lie in $[\tau-(b+\varepsilon),\tau+(b+\varepsilon)]$. Both middle order statistics, and hence their midpoint when needed, lie in this interval. Therefore the median lies within $b+\varepsilon$ of $\tau$.
\end{proof}

\subsection{Finite-Strata MMD Specialization}
\label{app:finite-strata-mmd-specialization}

This subsection records the earlier finite-strata analysis as an independent specialization. It is not used by the continuous-representation results in the main text.

Fix a finite library $\mathcal F_j^{\mathrm{fs}}\subseteq\Phi_j$ with $N_j^{\mathrm{fs}}\triangleq|\mathcal F_j^{\mathrm{fs}}|$. Every $\phi\in\mathcal F_j^{\mathrm{fs}}$ is deterministic and takes values in a common finite set $\mathcal S_j$ of size $M_j$. Write $Z_j^\phi=\phi(V_j)$, let $\mathcal R_j(\phi)$ be its positive-mass strata, and set $p_{j,z}(\phi)=P(Z_j^\phi=z)$. Assume, uniformly over the library and $z\in\mathcal R_j(\phi)$,
\(P(Z_j^\phi=z)\ge p_{\min}>0, \eta\le P(A=1\mid Z_j^\phi=z)\le1-\eta\)
for some $\eta\in(0,1/2]$. Let $\kappa_j$ be a characteristic kernel on $(X,W_j)$ with reproducing kernel Hilbert space $\mathcal H_j$ and
\(\sup_v\sqrt{\kappa_j(v,v)}\le\rho_j<\infty.\)
For $z\in\mathcal R_j(\phi)$, define
\(\Delta_{j,z}^{\mathrm{MMD}}(\phi) \triangleq \operatorname{MMD}_{\kappa_j} \left\{ P(X,W_j\mid A=1,Z_j^\phi=z), P(X,W_j\mid A=0,Z_j^\phi=z) \right\}\); \(D_j^{\mathrm{MMD}}(\phi) \triangleq \sum_{z\in\mathcal R_j(\phi)} p_{j,z}(\phi)\Delta_{j,z}^{\mathrm{MMD}}(\phi).\)

For $i\in\mathcal I_{\mathrm D}$, let $Z_{j,i}^\phi=\phi(V_{j,i})$ and define
\(N_{a,z}^\phi \triangleq \sum_{i\in\mathcal I_{\mathrm D}} \mathbf 1\{A_i=a,Z_{j,i}^\phi=z\}\); \(N_z^\phi\triangleq N_{0,z}^\phi+N_{1,z}^\phi\); \(\widehat p_{j,z}(\phi)\triangleq\frac{N_z^\phi}{n_{\mathrm D}}\); \(\widehat m_{a,j,z}(\phi) \triangleq \frac{1}{N_{a,z}^\phi} \sum_{i\in\mathcal I_{\mathrm D}} \mathbf 1\{A_i=a,Z_{j,i}^\phi=z\}\); \(\qquad{} \times\kappa_j\{(X_i,W_{j,i}),\cdot\}.\)
When both treatment-arm cells are nonempty, set
\(\widehat\Delta_{j,z}^{\mathrm{MMD}}(\phi) \triangleq \|\widehat m_{1,j,z}(\phi)-\widehat m_{0,j,z}(\phi)\|_{\mathcal H_j}\); \(\widehat D_j^{\mathrm{MMD}}(\phi) \triangleq \sum_{z:N_z^\phi>0} \widehat p_{j,z}(\phi)\widehat\Delta_{j,z}^{\mathrm{MMD}}(\phi).\)
If an occupied stratum contains only one treatment arm, set $\widehat D_j^{\mathrm{MMD}}(\phi)=+\infty$.

The MMD specialization uses its own structural stability condition. For every $\phi\in\mathcal F_j^{\mathrm{fs}}$, every positive-mass stratum $z$, every $a\in\{0,1\}$, and all $Q,Q'\in\mathcal Q_{j,z}$, define
\(\Delta^\mu_{a,j,z}(Q,Q') \triangleq \mathbb E_Q\{\mu_{a,j,z}\}-\mathbb E_{Q'}\{\mu_{a,j,z}\}.\)
Assume
\begin{align}
\left|\Delta^\mu_{a,j,z}(Q,Q')\right|
\le
\Gamma_j^{\mathrm{MMD}}(\phi)
\operatorname{MMD}_{\kappa_j}(\mathcal K_jQ,\mathcal K_jQ'),
\qquad
\sup_{\phi\in\mathcal F_j^{\mathrm{fs}}}
\Gamma_j^{\mathrm{MMD}}(\phi)
\le\overline\Gamma_j^{\mathrm{MMD}}<\infty.
\label{eq:finite-mmd-stability}
\end{align}

\begin{proposition}[Finite-strata MMD learning bound]
\label{prop:finite-strata-mmd-learning-bound}
Let $\delta_{\mathrm D}\in(0,1)$ and suppose
\begin{align}
n_{\mathrm D}
\ge
\frac{8}{p_{\min}\eta}
\log\left(\frac{8M_jN_j^{\mathrm{fs}}}{\delta_{\mathrm D}}\right).
\label{eq:finite-mmd-sample-threshold}
\end{align}
Define
\(e_{j,n_{\mathrm D}}^{\mathrm{MMD}}(\delta_{\mathrm D}) \triangleq \rho_j \left[ 2+\sqrt{2\log\left(\frac{4M_jN_j^{\mathrm{fs}}}{\delta_{\mathrm D}}\right)} \right] \sqrt{\frac{2}{n_{\mathrm D}p_{\min}\eta}}\); \(r_{j,n_{\mathrm D}}^{\mathrm{MMD}}(\delta_{\mathrm D}) \triangleq 2e_{j,n_{\mathrm D}}^{\mathrm{MMD}}(\delta_{\mathrm D}) +2\rho_j \sqrt{ \frac{ 2\left[M_j\log2+\log\{4N_j^{\mathrm{fs}}/\delta_{\mathrm D}\}\right] }{n_{\mathrm D}} }.\)
Then, with probability at least $1-\delta_{\mathrm D}$,
\begin{align}
\sup_{\phi\in\mathcal F_j^{\mathrm{fs}}}
\left|
\widehat D_j^{\mathrm{MMD}}(\phi)-D_j^{\mathrm{MMD}}(\phi)
\right|
\le
r_{j,n_{\mathrm D}}^{\mathrm{MMD}}(\delta_{\mathrm D}).
\label{eq:finite-mmd-uniform-deviation}
\end{align}
If $\widehat\phi_j^{\mathrm{fs}}$ obeys
\(\widehat D_j^{\mathrm{MMD}}(\widehat\phi_j^{\mathrm{fs}}) \le \inf_{\phi\in\mathcal F_j^{\mathrm{fs}}} \widehat D_j^{\mathrm{MMD}}(\phi) +\varepsilon_{\mathrm{opt},j}^{\mathrm{fs}},\)
then, on the same event,
\begin{align}
D_j^{\mathrm{MMD}}(\widehat\phi_j^{\mathrm{fs}})
&\le
\inf_{\phi\in\mathcal F_j^{\mathrm{fs}}}
D_j^{\mathrm{MMD}}(\phi)
+2r_{j,n_{\mathrm D}}^{\mathrm{MMD}}(\delta_{\mathrm D})
+\varepsilon_{\mathrm{opt},j}^{\mathrm{fs}},
\label{eq:finite-mmd-oracle}\\
|\tau_{\widehat Z_j^{\mathrm{fs}}}-\tau|
&\le
\overline\Gamma_j^{\mathrm{MMD}}
\left[
\inf_{\phi\in\mathcal F_j^{\mathrm{fs}}}
D_j^{\mathrm{MMD}}(\phi)
+2r_{j,n_{\mathrm D}}^{\mathrm{MMD}}(\delta_{\mathrm D})
+\varepsilon_{\mathrm{opt},j}^{\mathrm{fs}}
\right],
\label{eq:finite-mmd-causal-oracle}
\end{align}
where $\widehat Z_j^{\mathrm{fs}}=\widehat\phi_j^{\mathrm{fs}}(V_j)$, provided the primitive structural conditions and overlap used in the exact bias decomposition hold for every candidate.
\end{proposition}

\begin{proof}
Eq.~\eqref{eq:finite-mmd-sample-threshold} and a multiplicative Chernoff bound imply that every candidate treatment-by-stratum cell contains at least $n_{\mathrm D}p_{\min}\eta/2$ observations except on an event of probability at most $\delta_{\mathrm D}/4$. Conditional on the cell assignments, bounded Hilbert-space mean concentration and a union bound over at most $2M_jN_j^{\mathrm{fs}}$ cells give
\(\left| \widehat\Delta_{j,z}^{\mathrm{MMD}}(\phi) -\Delta_{j,z}^{\mathrm{MMD}}(\phi) \right| \le 2e_{j,n_{\mathrm D}}^{\mathrm{MMD}}(\delta_{\mathrm D})\)
simultaneously, with failure probability at most $\delta_{\mathrm D}/2$. A multinomial $L_1$ concentration bound gives
\(\sum_z|\widehat p_{j,z}(\phi)-p_{j,z}(\phi)| \le \sqrt{ \frac{ 2\left[M_j\log2+\log\{4N_j^{\mathrm{fs}}/\delta_{\mathrm D}\}\right] }{n_{\mathrm D}} }\)
simultaneously over candidates, with failure probability at most $\delta_{\mathrm D}/4$. On the intersection of these events, $0\le\Delta_{j,z}^{\mathrm{MMD}}(\phi)\le2\rho_j$ and the triangle inequality give Eq.~\eqref{eq:finite-mmd-uniform-deviation}.

The empirical-risk comparison yields
\(D_j^{\mathrm{MMD}}(\widehat\phi_j^{\mathrm{fs}}) \le \widehat D_j^{\mathrm{MMD}}(\widehat\phi_j^{\mathrm{fs}}) +r_{j,n_{\mathrm D}}^{\mathrm{MMD}}(\delta_{\mathrm D}) \le \inf_{\phi\in\mathcal F_j^{\mathrm{fs}}} D_j^{\mathrm{MMD}}(\phi) +2r_{j,n_{\mathrm D}}^{\mathrm{MMD}}(\delta_{\mathrm D}) +\varepsilon_{\mathrm{opt},j}^{\mathrm{fs}}\),
which proves Eq.~\eqref{eq:finite-mmd-oracle}. The exact bias decomposition and the MMD stability condition in Eq.~\eqref{eq:finite-mmd-stability} give
\(|\tau_{Z_j^\phi}-\tau| \le \Gamma_j^{\mathrm{MMD}}(\phi)D_j^{\mathrm{MMD}}(\phi).\)
Apply this inequality to $\widehat\phi_j^{\mathrm{fs}}$ and substitute Eq.~\eqref{eq:finite-mmd-oracle} to obtain Eq.~\eqref{eq:finite-mmd-causal-oracle}.
\end{proof}

\endgroup

\end{document}